\RequirePackage{fix-cm}
\documentclass[twocolumn]{svjour3arxiv}
\smartqed

  \usepackage{caption}
  \usepackage{subcaption}

  \usepackage{graphicx}
  \graphicspath{ {./}{./fig/} }

  \usepackage{booktabs,tabularx}

  \usepackage{tikz,pgfplots}
  \usetikzlibrary{plotmarks}
  \pgfplotsset{compat=newest}

  \usetikzlibrary{external}

  \usepackage{hyperref}
  \hypersetup{
    bookmarks=true,
    pdfstartview={FitH},
    colorlinks=true,
    linkcolor=black,
    citecolor=black,
    filecolor=black,
    urlcolor=black
  }

  \usepackage{color}

  \newcommand{\eg}{\textit{e.g.}}
  \newcommand{\ie}{\textit{i.e.}}
  \newcommand{\cf}{\textit{cf.}}

  \usepackage{amsmath,amsfonts,amssymb,amsthm,mathrsfs}

  \newcommand{\T}{^\mathsf{T}}
  \newcommand{\tr}{\mathrm{Tr}}
  
  \newcommand{\R}{\mathbb{R}}

  \providecommand{\norm}[1]{\lVert#1\rVert}
  
  \providecommand{\innerp}[1]{\langle#1\rangle}
  \providecommand{\op}[1]{\mathcal{#1}}
  
  \providecommand{\Fourier}{\mathscr{F}}
  \newcommand{\N}{\mathcal{N}}
  \newcommand{\GP}{\mathcal{GP}}
  \newcommand{\dd}{\,\mathrm{d}}
  \renewcommand{\O}{\mathcal{O}}
  \providecommand{\imag}{\mathrm{i}}

  \usepackage{bm}
  
  \newcommand{\mbf}[1]{\mathbf{#1}}
  \newcommand{\vect}[1]{\mbf{#1}}
  \newcommand{\vectb}[1]{\mathbf{#1}}

  \newtheorem{mytheorem}{Theorem}
  \newtheorem{mylemma}[mytheorem]{Lemma}
  \newtheorem{mycorollary}[mytheorem]{Corollary}
  \newtheorem{myremark}[mytheorem]{Remark}

  \usepackage{natbib}

\begin{document}

\title{Hilbert Space Methods for Reduced-Rank \\ Gaussian Process Regression}

\author{Arno Solin         \and
        Simo S\"arkk\"a
}

\institute{A.\ Solin \at
              Department of Computer Science  \\
              Aalto University \\
              P.O.\ Box 15400,  
              FI-00076 Aalto, Finland \\
              Tel.: +358-40-5776226\\
              \email{arno.solin@aalto.fi} 
           \and
           S.\ S\"arkk\"a \at
              Department of Electrical Engineering and Automation \\
              Aalto University
}

\date{~}

\maketitle

\begin{abstract}%
This paper proposes a novel scheme for reduced-rank Gaussian process regression. The method is based on an approximate series expansion of the covariance function in terms of an eigenfunction expansion of the Laplace operator in a compact subset of $\R^d$. On this  approximate eigenbasis the eigenvalues of the covariance function can be expressed as simple functions of the spectral density of the Gaussian process, which allows the GP inference to be solved under a computational cost scaling as $\O(nm^2)$ (initial) and $\O(m^3)$ (hyperparameter learning) with $m$ basis functions and $n$ data points. Furthermore, the basis functions are independent of the parameters of the covariance function, which allows for very fast hyperparameter learning. The approach also allows for rigorous error analysis with Hilbert space theory, and we show that the approximation becomes exact when the size of the compact subset and the number of eigenfunctions go to infinity. We also show that the convergence rate of the truncation error is independent of the input dimensionality provided that the differentiability order of the covariance function is increases appropriately, and for the squared exponential covariance function it is always bounded by ${\sim}1/m$ regardless of the input dimensionality. The expansion generalizes to Hilbert spaces with an inner product which is defined as an integral over a specified input density. The method is compared to previously proposed methods theoretically and through empirical tests with simulated and real data. \sloppy
\keywords{  Gaussian process regression \and
  Laplace operator \and
  eigenfunction expansion \and
  pseudo-differential operator \and
  reduced-rank approximation}
\end{abstract}

\section{Introduction}
Gaussian processes \citep[GPs,][]{Rasmussen+Williams:2006} are powerful tools for non-parametric Bayesian inference and learning. In GP regression the model functions $f(\vect{x})$ are assumed to be realizations from a Gaussian random process prior with a given covariance function $k(\vect{x},\vect{x}')$, and learning amounts to solving the posterior process given a set of noisy measurements $y_1,y_2,\dots,y_n$ at some given test inputs. This model is often written in the form \sloppy
\begin{equation}
\begin{split}
    f &\sim \GP(0,k(\vect{x},\vect{x}')), \\
  y_i &= f(\vect{x}_i) + \varepsilon_i,
\end{split}
\end{equation}
where $\varepsilon_i \sim \N(0,\sigma_\mathrm{n}^2)$, for $i=1,2,\ldots,n$.
One of the main limitations of GPs in machine learning is the computational and memory requirements that scale as $\O(n^3)$ and $\O(n^2)$ in a direct implementation. This limits the applicability of GPs when the number of training samples $n$ grows large. The computational requirements arise because in solving the GP regression problem we need to invert the $n \times n$ Gram matrix $\vect{K} + \sigma_\mathrm{n}^2\vect{I}$, where $\vect{K}_{ij} = k(\vect{x}_i,\vect{x}_j)$, which is an $\O(n^3)$ operation in general.

To overcome this problem, over the years, several schemes have been proposed. They typically reduce the storage requirements to $\O(nm)$ and complexity to $\O(nm^2)$, where $m < n$. Some early methods have been reviewed in \citet{Rasmussen+Williams:2006}, and \citet{Quinonero-Candela+Rasmussen:2005} provide a unifying view on several methods. From a spectral point of view, several of these methods (\eg, SOR, DTC, VAR, FIC) can be interpreted as modifications to the so-called \emph{Nystr{\"o}m method} \citep[see][]{Baker:1977,Williams+Seeger:2001}, a scheme for approximating the eigenspectrum. 

For stationary covariance functions the spectral density of the covariance function can be employed: in this context the spectral approach has mainly been considered in regular grids, as this allows for the use of FFT-based methods for fast solutions \mbox{\citep[see][]{Paciorek:2007, Fritz+Neuweiler+Nowak:2009}}, and more recently in terms of converting GPs to state space models \citep{Sarkka+Hartikainen:2012, Sarkka+Solin+Hartikainen:2013}. Recently, \citet{Lazaro-Gredilla+Quinonero-Candela+Rasmussen+Figueiras-Vidal:2010} proposed a sparse spectrum method where a randomly chosen set of spectral points span a trigonometric basis for the problem.

The methods proposed in this article fall into the class of methods called reduced-rank approximations \citep[see, \eg,][Ch.~8]{Rasmussen+Williams:2006} which are based on approximating the Gram matrix $\vect{K}$ with a matrix $\tilde{\vect{K}}$ with a smaller rank $m < n$. This allows for the use of matrix inversion lemma (Woodbury formula) to speed up the computations. It is well-known that the optimal reduced-rank approximation of the Gram (covariance) matrix $\vect{K}$ with respect to the Frobenius norm is $\tilde{\vect{K}} = \vectb{\Phi}\vectb{\Lambda}\vectb{\Phi}\T$, where $\vectb{\Lambda}$ is a diagonal matrix of the leading $m$ eigenvalues of $\vect{K}$ and $\vectb{\Phi}$ is the matrix of the corresponding orthonormal eigenvectors \citep[][Ch.~8]{Golub+VanLoan:1996, Rasmussen+Williams:2006}. Yet, as computing the eigendecomposition is an $\O(n^3)$ operation, this provides no remedy as such.

In this work we propose a novel method for obtaining approximate eigendecompositions of covariance functions in terms of an eigenfunction expansion of the Laplace operator in a compact subset of $\R^d$. The method is based on interpreting the covariance function as the kernel of a pseudo-differential operator \citep{Shubin:1987} and approximating it using Hilbert space methods \citep{Courant+Hilbert:2008,Showalter:2010}. This results in a reduced-rank approximation for the covariance function, where the basis functions are independent of the covariance functions and its parameters. We also show that the approximation converges to the exact solution in well-defined conditions, analyze its convergence rate, and provide theoretical and experimental comparisons to existing state-of-the-art methods. This path has not been explored in GP regression context before, although the approach is related to the Fourier feature methods \citep{Hensman+Durrande+Solin:2018} and stochastic partial differential equation based methods recently introduced to spatial statistics and GP regression \citep{Lindgren+Rue+Lindstrom:2011, Sarkka+Hartikainen:2012, Sarkka+Solin+Hartikainen:2013} as well as to classical works in the spectral representations of stochastic processes \citep{Loeve:1963,Van-Trees:1968,Adler:1981,Cramer+Leadbetter:1967} and spline interpolation \citep{Wahba:1978,Kimeldorf+Wahba:1970,Wahba:1990}. Recently, the scalable eigendecomposition approach has also been tackled by various structure exploiting methods \citep[building on the work by][]{Wilson+Nickisch:2015} and extended to methods exploiting GPU computations.

This paper is structured as follows: In Section~\ref{sec:approximating-the-covariance-function} we derive the approximative series expansion of the covariance functions. Section~\ref{sec:application-of-the-method-to-gp-regression} is dedicated to applying the approximation scheme to GP regression and providing details of the computational benefits. We provide a detailed analysis of the convergence of the method in Section~\ref{sec:convergence-analysis}. Section~\ref{sec:relationship-to-other-methods} and \ref{sec:experiments} provide comparisons to existing methods, the former from a more theoretical point of view, whereas the latter contains examples and comparative evaluation on several datasets. Finally the properties of the method are summarized and discussed in Section~\ref{sec:conclusion-and-discussion}.

\section{Approximating the Covariance Function}
\label{sec:approximating-the-covariance-function}
In this section, we start by stating the assumptions and properties of the class of covariance functions that we are considering, and show how a homogenous covariance function can be considered as a pseudo-differential operator constructed as a series of Laplace operators. Then we show how the pseudo-differential operators can be approximated with Hilbert space methods on compact subsets of $\R^d$ or via inner products with integrable weight functions, and discuss connections to Sturm--Liouville theory.

\subsection{Spectral Densities of Homogeneous and Isotropic Gaussian Processes}
In this work it is assumed that the covariance function is homogeneous (stationary), which means that the covariance function $k(\vect{x},\vect{x}')$ is actually a function of $\vect{r} = \vect{x} - \vect{x}'$ only. This means that the covariance structure of the model function $f(\vect{x})$ is the same regardless of the absolute position in the input space \citep[\cf][Ch.~4]{Rasmussen+Williams:2006}.
In this case the covariance function can be equivalently represented in terms of the spectral density. This results from the \emph{Bochner's theorem} \citep[see, \eg,][]{Akhiezer+Glazman:1993, DaPrato:1992} which states that a bounded continuous positive definite function $k(\vect{r})$ can be represented as
\begin{equation}
  k(\vect{r}) = {1  \over (2\pi)^{d}} \int 
       \exp\left(\imag \, \vectb{\omega} \T \vect{r}\right) \, \mu(\mathrm{d} \vectb{\omega}),
\end{equation}
where $\mu$ is a positive measure.

If the measure $\mu(\vectb{\omega})$ has a density, it is called the \emph{spectral density} $S(\vectb{\omega})$ corresponding to the covariance function $k(\vect{r})$. This gives rise to the Fourier duality of covariance and spectral density, which is known as the \emph{Wiener--Khintchin theorem} \citep[][Ch.~4]{Rasmussen+Williams:2006}, giving the identities
\begin{equation} \label{eq:duality}
\begin{split}
  k(\vect{r}) &= {1 \over (2\pi)^d} \int S(\vectb{\omega}) \,
       e^{ \imag \, \vectb{\omega} \T \vect{r}} \dd \vectb{\omega},
  \\
  S(\vectb{\omega}) &= \int k(\vect{r}) \,
       e^{-\imag \, \vectb{\omega} \T \vect{r}} \dd \vect{r}.
\end{split}
\end{equation}
From these identities it is easy to see that if the covariance function is \emph{isotropic}, that is, it only depends on the Euclidean norm $\norm{\vect{r}}$ such that $k(\vect{r}) \triangleq k(\norm{\vect{r}})$, then the spectral density will also only depend on the norm of $\vectb{\omega}$ such that we can write $S(\vectb{\omega}) \triangleq S(\norm{\vectb{\omega}})$. In the following we assume that the considered covariance functions are indeed isotropic, but the approach can be generalized to more general homogenous covariance functions.

\subsection{The Covariance Operator As a Pseudo-Differential Operator}
Associated to each covariance function $k(\vect{x},\vect{x}')$ we can
also define a covariance operator $\op{K}$ as follows:
\begin{equation} \label{eq:kop}
  \op{K} \, \phi
  = \int k(\cdot,\vect{x}') \, \phi(\vect{x}') \dd\vect{x}'.
\end{equation}
Note that because the covariance function is homogeneous, this can also be written as a convolution. As we show in the next section, this interpretation allows us to approximate the covariance operator using Hilbert space methods which are typically used for approximating differential and pseudo-differential operators in the context of partial differential equations \citep{Showalter:2010}. When the covariance function is homogenous, the corresponding operator will be translation invariant thus allowing for Fourier-representation as a transfer function. This transfer function is just the spectral density of the Gaussian process.

Consider an isotropic covariance function $k(\vect{x},\vect{x}') \triangleq k(\norm{\vect{r}})$ (recall that $\norm{\cdot}$ denotes the Euclidean norm). The spectral density of the Gaussian process and thus the transfer function corresponding to the covariance operator will now have the form $S(\norm{\vectb{\omega}})$. We can 
formally write it as a function of $\norm{\vectb{\omega}}^2$ such that 
\begin{equation}
  S(\norm{\vectb{\omega}}) = \psi(\norm{\vectb{\omega}}^2).
\end{equation}
Assume that the spectral density $S(\cdot)$ and hence $\psi(\cdot)$ have the following polynomial expansion:
\begin{equation}
  \psi(\norm{\vectb{\omega}}^2) = 
    a_0 + 
    a_1 \norm{\vectb{\omega}}^2 + 
    a_2 (\norm{\vectb{\omega}}^2)^2 + 
    a_3 (\norm{\vectb{\omega}}^2)^3 + \cdots.
\end{equation}
This can be ensured, for example, by requiring that $\psi(\cdot)$ is an analytic function. Thus we also have
\begin{equation} \label{eq:polynomial-expansion}
  S(\norm{\vectb{\omega}}) =
    a_0 + 
    a_1 \norm{\vectb{\omega}}^2 + 
    a_2 (\norm{\vectb{\omega}}^2)^2 + 
    a_3 (\norm{\vectb{\omega}}^2)^3 + \cdots.
\end{equation}
Recall that the transfer function corresponding to the Laplace operator
$\nabla^2$ is $-\norm{\vectb{\omega}}^2$ in the sense that for a regular enough function $f$ we have
\begin{equation} 
  \Fourier[\nabla^2 f](\vectb{\omega}) = -\norm{\vectb{\omega}}^2
  \Fourier[f](\vectb{\omega}),
\end{equation} 
where $\Fourier[\cdot]$ denotes the Fourier transform of its argument. If we take the inverse Fourier transform of \eqref{eq:polynomial-expansion}, we get the following representation for the covariance operator $\op{K}$, which defines a pseudo-differential operator \citep{Shubin:1987} as a formal series of Laplace operators:
\begin{equation} \label{eq:psi_series}
  \op{K} = 
    a_0 + 
    a_1 (-\nabla^2) + 
    a_2 (-\nabla^2)^2 + 
    a_3 (-\nabla^2)^3 + \cdots.
\end{equation}
In the next section we will use this representation to form a series expansion
approximation for the covariance function.

\subsection{Hilbert-Space Approximation of the Covariance Operator}
We will now form a Hilbert-space approximation for the pseudo-differential operator defined by \eqref{eq:psi_series}.  Let $\Omega \subset \R^d$ be a compact set, and consider the eigenvalue problem for the Laplace operators with Dirichlet boundary conditions (we could use other boundary conditions as well):
\begin{equation}
\begin{cases}
  -\nabla^2 \phi_j(\vect{x}) = \lambda_j \, \phi_j(\vect{x}), \quad%
  & \vect{x} \in \Omega, \\
  \phantom{-\nabla^2 } %
            \phi_j(\vect{x}) = 0, \quad %
  & \vect{x} \in \partial \Omega.
\end{cases}
\label{eq:eigenf_eqs}
\end{equation}
Let us now assume that we have selected $\partial \Omega$ to be sufficiently smooth, for example, a hypercube or hypersphere, so that the eigenfunctions and eigenvalues exist. Because $-\nabla^2$ is a positive definite Hermitian operator, the set of eigenfunctions $\phi_j(\cdot)$ is orthonormal with respect to the inner product
\begin{equation}
  \innerp{f,g} = \int_\Omega f(\vect{x})\,g(\vect{x})\dd \vect{x}
\end{equation}
that is,
\begin{equation}
  \int_\Omega \phi_i(\vect{x})\,\phi_j(\vect{x})\dd \vect{x} = \delta_{ij},
\end{equation}
and all the eigenvalues $\lambda_j$ are real and positive. The negative Laplace operator can then be assigned the formal kernel
\begin{equation} \label{eq:kernel}
  l(\vect{x},\vect{x}') = \sum_j \lambda_j \, \phi_j(\vect{x}) \, \phi_j(\vect{x}')
\end{equation}
in the sense that
\begin{equation} 
  -\nabla^2 f(\vect{x}) = \int l(\vect{x},\vect{x}') \, f(\vect{x}') \dd \vect{x}',
\end{equation}
for sufficiently (weakly) differentiable functions $f$ in the domain $\Omega$ assuming Dirichlet boundary conditions.

If we consider the formal powers of this representation, due to orthonormality of the basis, we can write the arbitrary operator power $s = 1,2,\ldots$ of the kernel as 
\begin{equation} \label{eq:lpower}
  l^s(\vect{x},\vect{x}') = \sum_j \lambda_j^s\,\phi_j(\vect{x})\,\phi_j(\vect{x}').
\end{equation}
This is again to be interpreted to mean that
\begin{equation} 
  (-\nabla^2)^s f(\vect{x}) = \int l^s(\vect{x},\vect{x}') \, f(\vect{x}') \dd \vect{x}',
\end{equation}
for regular enough functions $f$ and in the current domain with the assumed boundary conditions.

This implies that on the domain $\Omega$, assuming the boundary conditions, we also have
\begin{multline}
  \left[ a_0 + 
    a_1 (-\nabla^2) + 
    a_2 (-\nabla^2)^2 + 
    \cdots \right] f(\vect{x}) \\
  = \int \left[a_0 + a_1 \, l^1(\vect{x},\vect{x}')
          + a_2 \, l^2(\vect{x},\vect{x}')
          + \cdots 
          \right] f(\vect{x}') \dd \vect{x}'.
\end{multline}
The left hand side is just $\op{K} \, f$ via \eqref{eq:psi_series}, on the domain with the boundary conditions, and thus by comparing to \eqref{eq:kop} and using \eqref{eq:lpower} we can conclude that
\begin{equation} \label{eq:kapp0}
\begin{split} 
  k(\vect{x},\vect{x}')
  &\approx a_0 +
  	  a_1 \, l^1(\vect{x},\vect{x}') +
          a_2 \, l^2(\vect{x},\vect{x}') +
          \cdots \\
  &= \sum_j \left[a_0 +
  	 a_1 \, \lambda_j^1 +
          a_2 \, \lambda_j^2 +
          \cdots \right] 
         \phi_j(\vect{x}) \, \phi_j(\vect{x}'),
\end{split} 
\end{equation} 
which is only an approximation to the covariance function due to restriction of the domain to $\Omega$ and the boundary conditions. By letting $\norm{\vectb{\omega}}^2 = \lambda_j$ in \eqref{eq:polynomial-expansion} we now obtain
\begin{equation} 
  S(\sqrt{\lambda_j}) =
    a_0 + 
    a_1 \lambda_j^1 + 
    a_2 \lambda_j^2 + 
    \cdots
\end{equation}
and substituting this into \eqref{eq:kapp0} then leads to the approximation
\begin{equation} \label{eq:approximation}
\begin{split} 
  \boxed{\quad \vphantom{\Bigg|} k(\vect{x},\vect{x}')
    \approx \sum_j S(\sqrt{\lambda_j}) \, \phi_j(\vect{x})\,\phi_j(\vect{x}'), \quad}
\end{split} 
\end{equation} 
where $S(\cdot)$ is the spectral density of the covariance function, $\lambda_j$ is the $j$th eigenvalue and $\phi_j(\cdot)$ the eigenfunction of the Laplace operator in a given domain. These expressions tend to be simple closed-form expressions.

The right hand side of \eqref{eq:approximation} is very easy to evaluate, because it corresponds to evaluating the spectral density at the square roots of the eigenvalues and multiplying them with the eigenfunctions of the Laplace operator. Because the eigenvalues of the Laplace operator are monotonically increasing with $j$ and for bounded covariance functions the spectral density goes to zero fast with higher frequencies, we can expect to obtain a good approximation of the right hand side by retaining only a finite number of terms in the series. However, even with an infinite number of terms this is only an approximation, because we assumed a compact domain with boundary conditions. The approximation can be, though, expected to be good at the input values which are not near the boundary of $\Omega$, where the Laplacian was taken to be zero.

\begin{figure*}[!t]

  \newlength{\figureheight}
  \newlength{\figurewidth}
  \setlength{\figureheight}{.20\textheight}
  \setlength{\figurewidth}{0.99\textwidth}

  \tikzsetnextfilename{tikz-covfun-approx}

  \centering \hspace{-3.0mm} 
  \sffamily \footnotesize
  \input{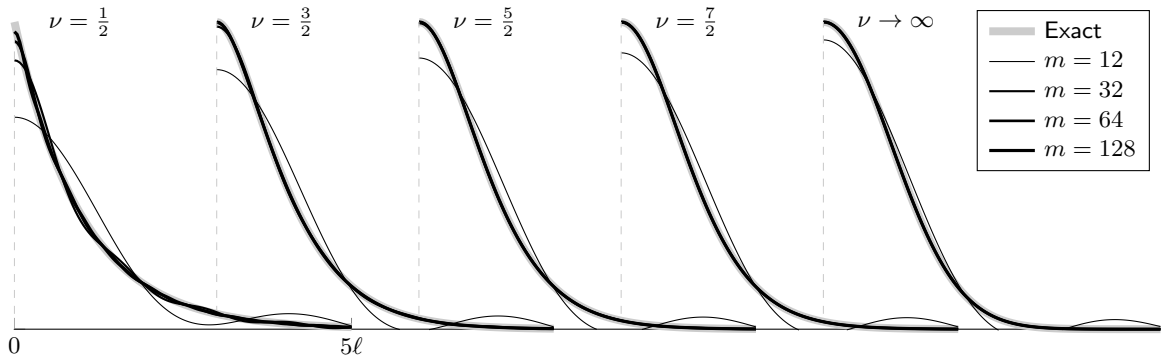}

  \caption{Approximations to covariance functions of the Mat{\'e}rn class
           of various degrees of smoothness; $\nu=1/2$ corresponds to the exponential 
           Ornstein--Uhlenbeck covariance function, and $\nu\to\infty$ to the
           squared exponential (exponentiated quadratic) covariance function.
           Approximations are shown for 12, 32, 64, and 128 eigenfunctions.}
  \label{fig:covfun_approx}
\end{figure*}

As an example, Figure~\ref{fig:covfun_approx} shows Mat{\'e}rn covariance functions of various degrees of smoothness $\nu$ \citep[see, \eg,][Ch.~4]{Rasmussen+Williams:2006} and approximations for different numbers of basis functions in the approximation.  The basis consists of the eigenfunctions of the Laplacian in \eqref{eq:eigenf_eqs} with $\Omega = [-L,L]$ which gives the eigenfunctions $\phi_j(x) = L^{-1/2} \sin(\pi j (x + L)/ (2L))$ and the eigenvalues $\lambda_j = (\pi \, j / (2L))^2$. In the figure we have set $L = 1$ and $\ell = 0.1$. For the squared exponential the approximation is indistinguishable from the exact curve already at $m=12$, whereas the less smooth functions require more terms.

\subsection{Inner Product Point of View} \label{sec:inner_product_point_of_view}
Instead of considering a compact bounded set $\Omega$, we can consider the same approximation in terms of an inner product defined by an input density \citep{Williams+Seeger:2000}.
Let the inner product be defined as
\begin{equation} 
  \innerp{f,g} = \int f(\vect{x})\,g(\vect{x})\,w(\vect{x}) \dd \vect{x},
\end{equation} 
where $w(\vect{x})$ is some positive weight function such that $\int w(\vect{x}) \dd \vect{x} < \infty$. In terms of this inner product, we define the operator 
\begin{equation} 
  \op{K}f = \int k(\cdot,\vect{x}) \, f(\vect{x}) \, w(\vect{x}) \dd \vect{x}.
\end{equation} 
This operator is self-adjoint with respect to the inner product, $\innerp{\op{K}f,g} = \innerp{f,\op{K}g}$, and according to the spectral theorem there exists an orthonormal set of basis functions and positive constants, $\{\varphi_j(\vect{x}), \gamma_j \mid j=1,2,\ldots\}$, that satisfies the eigenvalue equation
\begin{equation} \label{eq:eigd}
  (\op{K} \varphi_j)(\vect{x}) = \gamma_j \, \varphi_j(\vect{x}).
\end{equation}
Thus $k(\vect{x},\vect{x}')$ has the series expansion
\begin{equation} 
  k(\vect{x},\vect{x}') 
    = \sum_j \gamma_j \, \varphi_j(\vect{x}) \, \varphi_j(\vect{x}').
\end{equation} 
Similarly, we also have the \emph{Karhunen--Loeve expansion} for a sample function $f(\vect{x})$ with zero mean and the above covariance function:
\begin{equation} 
  f(\vect{x}) = \sum_j f_j \, \varphi_j(\vect{x}),
\end{equation} 
where $f_j$s are independent zero mean Gaussian random variables with variances $\gamma_j$ \citep[see, \eg,][]{Lenk:1991}.

For the negative Laplacian the corresponding definition is
\begin{equation} \label{eq:lapw}
  \op{D} f = -\nabla^2 [f \, w],
\end{equation} 
which implies
\begin{equation} 
  \innerp{\op{D} f,g}
    = -\int f(\vect{x}) \, w(\vect{x}) \nabla^2[g(\vect{x})\, w(\vect{x})] \dd \vect{x},
\end{equation} 
and the operator defined by \eqref{eq:lapw} can be seen to be self-adjoint. Again, there exists an orthonormal basis $\{ \phi_j(\vect{x}) \mid j=1,2,\ldots \}$ and positive eigenvalues $\lambda_j$ which satisfy the eigenvalue equation 
\begin{equation} 
  (\mathcal{D} \, \phi_j)(\vect{x}) = \lambda_j \, \phi_j(\vect{x}).
\end{equation} 
Thus the kernel of $\mathcal{D}$ has a series expansion similar to Equation~\eqref{eq:kernel} and thus an approximation can be given in the same form as in Equation~\eqref{eq:approximation}. In this case the approximation error comes from approximating the Laplace operator with the more smooth operator,
\begin{equation} 
  \nabla^2 f \approx \nabla^2 [f \, w],
\end{equation} 
which is closely related to assumption of an input density $w(\vect{x})$ for the
Gaussian process. However, when the weight function $w(\cdot)$ is close to constant in the area where the inputs points are located, the approximation is accurate.

\begin{figure*}
  \centering
  \begin{subfigure}[b]{0.3\textwidth}
    \centering
    \includegraphics[width=\textwidth,trim=3mm 3mm 3mm 3mm,clip]{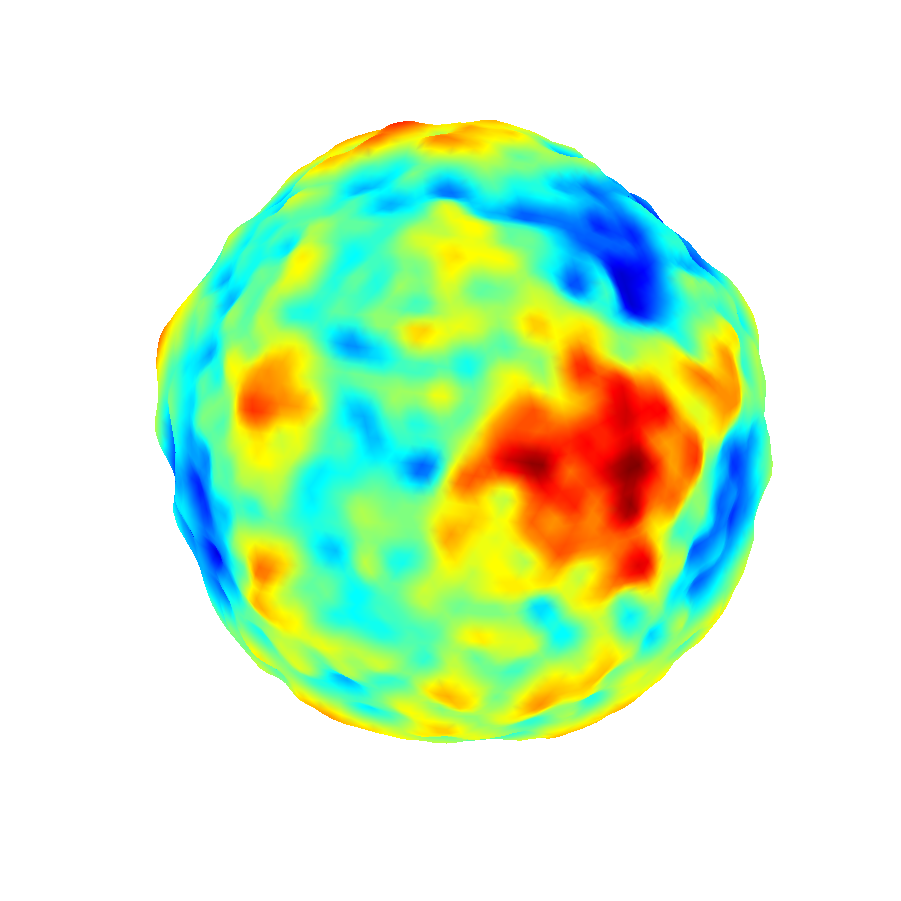}
    \caption{$\nu = \frac{1}{2}$ and $\ell = 0.5$}
  \end{subfigure}%
  ~ %
  \begin{subfigure}[b]{0.3\textwidth}
    \centering
    \includegraphics[width=\textwidth,trim=3mm 3mm 3mm 3mm,clip]{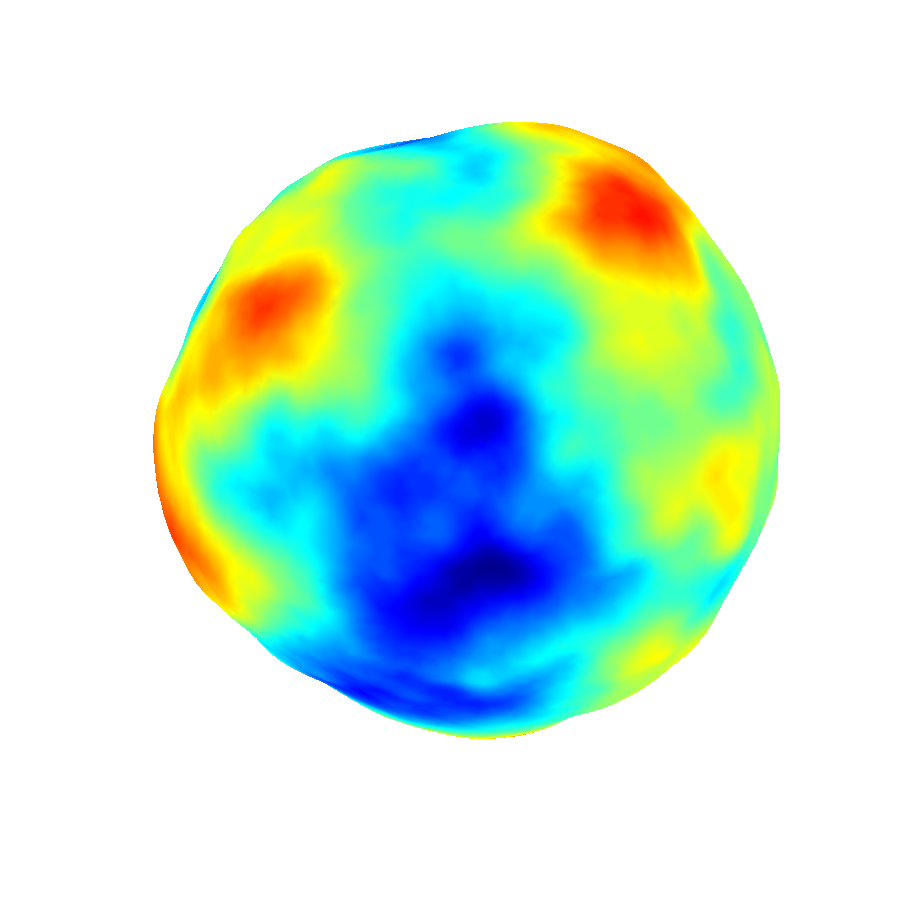}
    \caption{$\nu = \frac{3}{2}$ and $\ell = 0.5$}
  \end{subfigure}%
  ~ %
  \begin{subfigure}[b]{0.3\textwidth}
    \centering
    \includegraphics[width=\textwidth,trim=3mm 3mm 3mm 3mm,clip]{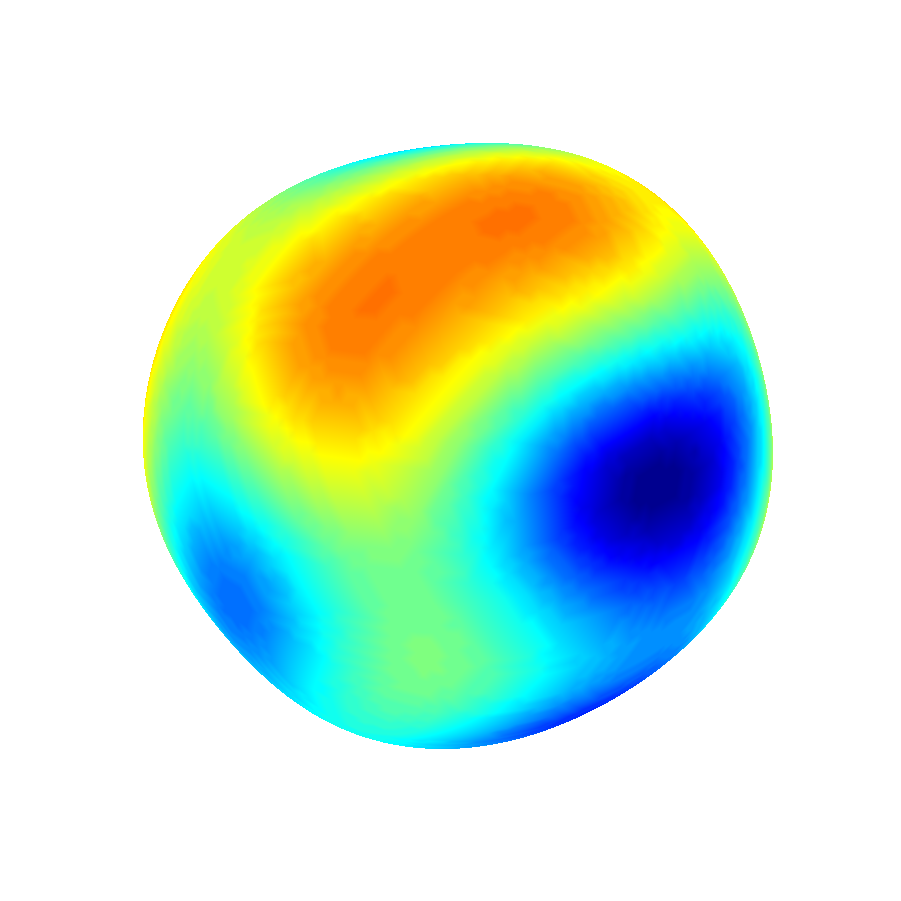}
    \caption{$\nu \to \infty$ and $\ell = 0.5$}
  \end{subfigure}%
  \caption{Approximate random draws of Gaussian processes with the 
           Mat{\'e}rn covariance function on the hull of a unit sphere. 
           The color scale and radius follow the process.}
  \label{fig:spheres}
\end{figure*}

\subsection{Connection to Sturm--Liouville Theory}
The presented methodology is also related to the Sturm--Liouville theory arising in the theory of partial differential equations \citep{Courant+Hilbert:2008}. When the input $x$ is scalar, the eigenvalue problem in Equation~\eqref{eq:eigd} can be written in Sturm--Liouville form as follows:
\begin{equation}
\begin{split}
  &-\frac{\dd}{\dd x} \left[w^2(x) \, \frac{\dd \phi_j(x)}{\dd x} \right] 
  - w(x) \, \frac{\dd^2 w(x)}{\dd x^2} \, \phi_j(x) \\
  &\qquad = \lambda_j \, w(x) \, \phi_j(x).
\end{split}
\end{equation}
The above equation can be solved for $\phi_j(x)$ and $\lambda_j$ using numerical methods for Sturm--Liouville equations. Also note that if we select $w(x) = 1$ in a finite set, we obtain the equation
  $-\dd^2/\dd x^2 \,\phi_j(x) = \lambda_j \, \phi_j(x)$
which is compatible with the results in the previous section.

We consider the case where $\vect{x} \in \R^d$ and $w(\vect{x})$ is symmetric around the origin and thus is only a function of the norm $r = \norm{\vect{x}}$ (\ie\ has the form $w(r)$). The Laplacian in spherical coordinates is
\begin{equation} \label{eq:spherical-laplacian}
  \nabla^2 f
  = \frac{1}{r^{d-1}} \frac{\partial}{\partial r} \left( r^{d-1} \,
     \frac{\partial f}{\partial r} \right)
  + \frac{1}{r^2} \, \Delta_{S^{d-1}} f,
\end{equation}
where $\Delta_{S^{d-1}}$ is the Laplace--Beltrami operator on $S^{d-1}$.  Let us assume that $\phi_j(r,\vectb{\xi}) = h_j(r) \, g(\vectb{\xi})$, where $\vectb{\xi}$ denotes the angular variables. After some algebra, writing the equations into Sturm--Liouville form yields for the radial part
\begin{multline}
  - \frac{\dd}{\dd r} \left( w^2(r) \, r \, \frac{\dd h_j(r)}{\dd r} \right) \\
  - \left( \frac{\dd w(r)}{\dd r} \, w(r) 
  + \frac{\dd^2 w(r)}{\dd r^2} \, w(r) \, r \right) \, h_j(r) \\
  = \lambda_j \, w(r) \, r \, h_j(r),
\end{multline}
and $\Delta_{S^{d-1}} g(\vectb{\xi}) = 0$ for the angular part. The solutions to the angular part are the Laplace's spherical harmonics. Note that if we assume that we have $w(r) = 1$ on some area of finite radius, the first equation becomes (when $d > 1$):
\begin{equation}
  r^2 \, \frac{\dd^2 h_j(r)}{\dd r^2} 
    + r \, \frac{\dd h_j(r)}{\dd r} 
    + r^2 \, \lambda_j \, h_j(r) = 0.
\end{equation}
Figure~\ref{fig:spheres} shows example Gaussian random field draws on a unit sphere, where the basis functions are the Laplace spherical harmonics and the covariance functions of the Mat{\'e}rn class with different degrees of smoothness $\nu$. Our approximation is straight-forward to apply in any domain, where the eigendecomposition of the Laplacian can be formed.

\section{Application of the Method to GP Regression}
\label{sec:application-of-the-method-to-gp-regression}
In this section we show how the approximation \eqref{eq:approximation} can be used in Gaussian process regression. We also write down the expressions needed for hyperparameter learning and discuss the computational requirements of the methods.

\subsection{Gaussian Process Regression}
GP regression is usually formulated as predicting an unknown scalar output $f(\vect{x}_*)$ associated with a known input $\vect{x}_* \in \R^d$, given a training data set $\mathcal{D} = \{(\vect{x}_i, y_i) \mid i=1,2,\ldots,n\}$. The model functions $f$ are assumed to be realizations of a Gaussian random process prior and the observations corrupted by Gaussian noise:
\begin{equation} \label{eq:GP-regression}
\begin{split}
  f &\sim \GP(0, k(\vect{x},\vect{x}')) \\
  y_i &= f(\vect{x}_i) + \varepsilon_i,
\end{split}
\end{equation}
where $\varepsilon_i \sim \N(0,\sigma_\mathrm{n}^2)$. For notational simplicity the functions in the above model are \emph{a~priori} zero mean and the measurement errors are independent Gaussian, but the results of this paper can be easily generalized to arbitrary mean functions and dependent Gaussian errors. The direct solution to the GP regression problem \eqref{eq:GP-regression} gives the predictions $p(f(\vect{x}_*) \mid \mathcal{D}) = \N(f(\vect{x}_*) \mid \mathbb{E}[f(\vect{x}_*)], \mathbb{V}[f(\vect{x}_*)])$. The conditional mean and variance can be computed in closed-form as \citep[see, \eg,][p.~17]{Rasmussen+Williams:2006}
\begin{align}
\begin{split} \label{eq:GP-solution}
  \mathbb{E}[f(\vect{x}_*)] &= 
    \vect{k}_*\T (\vect{K} + \sigma_\mathrm{n}^2 \vect{I})^{-1} \vect{y}, \\
  \mathbb{V}[f(\vect{x}_*)] &= 
    k(\vect{x}_*,\vect{x}_*) - \vect{k}_*\T 
    (\vect{K} + \sigma_\mathrm{n}^2 \vect{I})^{-1} \vect{k}_*,
\end{split}
\end{align}
where $\vect{K}_{ij} = k(\vect{x}_i,\vect{x}_j)$, $\vect{k}_*$ is an $n$-dimensional vector with the $i$th entry being $k(\vect{x}_*,\vect{x}_i)$, and $\vect{y}$ is a vector of the $n$ observations.

In order to avoid the $n \times n$ matrix inversion in \eqref{eq:GP-solution}, we use the approximation scheme presented in the previous section and project the process to a truncated set of $m$ basis functions of the Laplacian as given in Equation~\eqref{eq:approximation} such that
\begin{equation}
  f(\vect{x}) \approx \sum_{j=1}^m f_j \, \phi_j(\vect{x}),
\end{equation}
where $f_j \sim \N(0,S(\sqrt{\lambda_j}))$. We can then form an approximate eigendecomposition of the matrix $\vect{K} \approx \vectb{\Phi}\vectb{\Lambda}\vectb{\Phi}\T$, where $\vectb{\Lambda}$ is a diagonal matrix of the leading $m$ approximate eigenvalues such that $\vectb{\Lambda}_{jj} = S(\sqrt{\lambda_j}), j=1,2,\ldots,m$. Here $S(\cdot)$ is the spectral density of the Gaussian process and $\lambda_j$ the $j$th eigenvalue of the Laplace operator. The corresponding eigenvectors in the decomposition are given by the eigenvectors $\phi_j(\vect{x})$ of the Laplacian such that $\vectb{\Phi}_{ij} = \phi_j(\vect{x}_i)$.

Using the matrix inversion lemma we rewrite \eqref{eq:GP-solution} as follows:
\begin{align}
\begin{split} \label{eq:GP-solution-approx}
  \mathbb{E}[f_*] &\approx 
    \vectb{\phi}_*\T
    (\vectb{\Phi}\T\vectb{\Phi} + \sigma_\mathrm{n}^2 \vectb{\Lambda}^{-1})^{-1}
    \vectb{\Phi}\T \vect{y}, \\
  \mathbb{V}[f_*] &\approx 
    \sigma_\mathrm{n}^2 \vectb{\phi}_*\T
    (\vectb{\Phi}\T\vectb{\Phi} + \sigma_\mathrm{n}^2 \vectb{\Lambda}^{-1})^{-1}
    \vectb{\phi}_*,
\end{split}
\end{align}
where $\vectb{\phi}_*$ is an $m$-dimensional vector with the $j$th entry being $\phi_j(\vect{x}_*)$. Thus, when the size of the training set is higher than the number of required basis functions $n > m$, the use of this approximation is advantageous.

\subsection{Learning the Hyperparameters} \label{sec:hyperpar}
A common way to learn the hyperparameters $\vectb{\theta}$ of the covariance function (suppressed earlier in the notation for brevity) and the noise variance $\sigma_\mathrm{n}^2$ is by maximizing the marginal likelihood function \citep{Rasmussen+Williams:2006,Quinonero-Candela+Rasmussen:2005b}. Let $\vect{Q} = \vect{K} + \sigma_\mathrm{n}^2 \vect{I}$ for the full model, then the negative log marginal likelihood and its derivatives are
\begin{align}
  \mathcal{L} &= 
      {1 \over 2} \log |\vect{Q}| 
    + {1 \over 2} \vect{y}\T\vect{Q}^{-1}\vect{y} 
    + {n \over 2} \log(2\pi), \\
  {\partial \mathcal{L} \over \partial \theta_k} &= 
    {1 \over 2} \tr\! \left( \vect{Q}^{-1} 
      {\partial \vect{Q} \over \partial \theta_k} \right) 
    - {1 \over 2} \vect{y}\T\vect{Q}^{-1}{\partial \vect{Q} \over 
      \partial \theta_k}\vect{Q}^{-1}\vect{y}, \\
  {\partial \mathcal{L} \over \partial \sigma_\mathrm{n}^2} &= 
    {1 \over 2} \tr\! \left( \vect{Q}^{-1} \right) 
    - {1 \over 2} \vect{y}\T\vect{Q}^{-1}\vect{Q}^{-1}\vect{y},
\end{align}
and they can be combined with a conjugate gradient optimizer. The problem in this case is the inversion of $\vect{Q}$, which is an $n \times n$ matrix. And thus each step of running the optimizer is $\O(n^3)$.
For our approximation scheme, let $\tilde{\vect{Q}} = \vectb{\Phi}\vectb{\Lambda}\vectb{\Phi}\T + \sigma_\mathrm{n}^2 \vect{I}$. Now replacing $\vect{Q}$ with $\tilde{\vect{Q}}$ in the above expressions gives us the following:
\begin{align}
  \tilde{\mathcal{L}} &= 
      {1 \over 2} \log |\tilde{\vect{Q}}| 
    + {1 \over 2} \vect{y}\T\tilde{\vect{Q}}^{-1}\vect{y} 
    + {n \over 2} \log(2\pi), \\
  {\partial \tilde{\mathcal{L}} \over \partial \theta_k} &= 
    {1 \over 2}
  {\partial \log |\tilde{\vect{Q}}| \over \partial \theta_k}
+   {1 \over 2}
  {\partial \vect{y}\T\tilde{\vect{Q}}^{-1}\vect{y} \over \partial \theta_k}, \\  
  {\partial \tilde{\mathcal{L}} \over \partial \sigma_\mathrm{n}^2} &= 
    {1 \over 2}
  {\partial \log |\tilde{\vect{Q}}| \over \partial \sigma_\mathrm{n}^2}
+ {1 \over 2}
  {\partial \vect{y}\T\tilde{\vect{Q}}^{-1}\vect{y} \over \partial \sigma_\mathrm{n}^2},
\end{align}
where for the terms involving $\log|\tilde{\vect{Q}}|$:
\begin{align}
  \log |\tilde{\vect{Q}}| 
     &= (n-m)\log \sigma_\mathrm{n}^2 
        + \log |\vect{Z}| \nonumber \\
        &\qquad+ \sum_{j=1}^m \log S(\sqrt{\lambda_j}), \\
  {\partial \log |\tilde{\vect{Q}}| \over \partial \theta_k}
     &= \sum_{j=1}^m S(\sqrt{\lambda_j})^{-1} 
           {\partial S(\sqrt{\lambda_j}) \over \partial \theta_k} \nonumber \\
     &\qquad   - \sigma_\mathrm{n}^2\tr\!\left( \vect{Z}^{-1} \vect{\Lambda}^{-2}
          {\partial \vect{\Lambda} \over \partial \theta_k} \right), \\
  {\partial \log |\tilde{\vect{Q}}| \over \partial \sigma_\mathrm{n}^2}
     &= {n-m \over \sigma_\mathrm{n}^2} 
        + \tr\!\left( \vect{Z}^{-1}\vectb{\Lambda}^{-1}\right),
  \intertext{and for the terms involving $\tilde{\vect{Q}}^{-1}$:}
  \vect{y}\T\tilde{\vect{Q}}^{-1}\vect{y}
     &= \frac{1}{\sigma_\mathrm{n}^2}
        \left( \vect{y}\T\vect{y} - 
        \vect{y}\T\vect{\Phi} \vect{Z}^{-1} \vect{\Phi}\T\vect{y}  \right), \\
  {\partial \vect{y}\T\tilde{\vect{Q}}^{-1}\vect{y} \over \partial \theta_k}
     &= -\vect{y}\T \vect{\Phi} \vect{Z}^{-1}
        \left(\vectb{\Lambda}^{-2} 
          {\partial \vect{\Lambda} \over \partial \theta_k} \right)
        \vect{Z}^{-1} \vect{\Phi}\T \vect{y}, \\
  {\partial \vect{y}\T\tilde{\vect{Q}}^{-1}\vect{y} \over \partial \sigma_\mathrm{n}^2}
     &= \frac{1}{\sigma_\mathrm{n}^2} \vect{y}\T\vectb{\Phi}\vect{Z}^{-1}
        \vectb{\Lambda}^{-1}\vect{Z}^{-1}\vectb{\Phi}\T\vect{y}
        -\frac{1}{\sigma_\mathrm{n}^4}\vect{y}\T\vect{y},
\end{align}
where $\vect{Z} = \sigma_\mathrm{n}^2\vect{\Lambda}^{-1} + \vect{\Phi}\T\vect{\Phi}$. For efficient implementation, matrix-to-matrix multiplications can be avoided in many cases, and the inversion of $\vect{Z}$ can be carried out through Cholesky factorization for numerical stability. This factorization ($\vect{L}\vect{L}\T = \vect{Z}$) can also be used for the term $\log |\vect{Z}| = 2\sum_j \log \vect{L}_{jj}$, and $\tr\!\left( \vect{Z}^{-1}\vectb{\Lambda}^{-1}\right) = \sum_j 1/(\vect{Z}_{jj}\vect{\Lambda}_{jj})$ can be evaluated by element-wise multiplication.

Once the marginal likelihood and its derivatives are available, it is also possible to use other methods for parameter inference such as Markov chain Monte Carlo methods \citep{Liu:2001,Brooks+Gelman+Jones+Meng:2011} including Hamiltonian Monte Carlo  \citep[HMC,][]{Duane+Kennedy+Pendleton+Roweth:1987,Neal:2011} as well as numerous others.

\subsection{Discussion on the Computational Complexity}
As can be noted from Equation~\eqref{eq:approximation}, the basis functions in the reduced-rank approximation do not depend on the hyperparameters of the covariance function. Thus it is enough to calculate the product $\vectb{\Phi}\T\vectb{\Phi}$ only once, which means that the method has a overall asymptotic computational complexity of $\O(nm^2)$. After this initial cost, evaluating the marginal likelihood and the marginal likelihood gradient is an $\O(m^3)$ operation---which in practice comes from the Cholesky factorization of $\vect{Z}$ on each step.

If the number of observations $n$ is so large that storing the $n \times m$ matrix $\vectb{\Phi}$ is not feasible, the computations of $\vectb{\Phi}\T\vectb{\Phi}$ can be carried out in blocks. Storing the evaluated eigenfunctions in $\vectb{\Phi}$ is not necessary, because the $\phi_j(\vect{x})$ are closed-form expressions that can be evaluated when necessary. In practice, it might be preferable to cache the result of $\vectb{\Phi}\T\vectb{\Phi}$ (causing a memory requirement scaling as $\O(m^2)$), but this is not required.

The computational complexity of conventional sparse GP approximations typically scale as  $\O(nm^2)$ in time for each step of evaluating the marginal likelihood. The scaling in demand of storage is $\O(nm)$. This comes from the inevitable cost of re-evaluating all results involving the basis functions on each step and storing the matrices required for doing this. This applies to all the methods that will be discussed in Section~\ref{sec:relationship-to-other-methods}, with the exception of SSGP, where the storage demand can be relaxed by re-evaluating the basis functions on demand.

We can also consider the rather restricting, but in certain applications often encountered case, where the measurements are constrained to a regular grid. This causes the product of the orthonormal eigenfunction matrices $\vectb{\Phi}\T\vectb{\Phi}$ to be diagonal, avoiding the calculation of the matrix inverse altogether. This relates to the FFT-based methods for GP regression \mbox{\citep{Paciorek:2007, Fritz+Neuweiler+Nowak:2009}}, and the projections to the basis functions can be evaluated by fast Fourier transform in $\O(n \log n)$ time complexity.

\subsection{Inverse Problems and Latent Force Models} \label{sec:invlfm}
We can also use the methodology to models of the form
\begin{equation} \label{eq:invprob}
\begin{split}
  f(\vect{x}) &\sim \GP(0, k(\vect{x},\vect{x}')), \\
  y_i &= (\mathcal{H} f)(\vect{x}_i) + \varepsilon_i,
\end{split}
\end{equation}
where $\mathcal{H}$ is a linear operator acting on functions depending on the $\vect{x}$ variable. This kind of models appear both in inverse problems literature and machine learning \citep[see, \eg,][]{Tarantola:2004,Kaipio+Somersalo:2005,Sarkka:2011b}. The Gaussian process regression solution now becomes
\begin{align}
\begin{split} \label{eq:invprob-solution}
  \mathbb{E}[f(\vect{x}_*)] &= 
    \vect{k}_{*h}\T (\vect{K}_h + \sigma_\mathrm{n}^2 \vect{I})^{-1} \vect{y}, \\
  \mathbb{V}[f(\vect{x}_*)] &= 
    k(\vect{x}_*,\vect{x}_*) - \vect{k}_{*h}\T 
    (\vect{K}_h + \sigma_\mathrm{n}^2 \vect{I})^{-1} \vect{k}_{*h},
\end{split}
\end{align}
where $[\vect{K}_h]{ij} = (\mathcal{H} \, \mathcal{H}' \, k)(\vect{x}_i,\vect{x}_j)$, the $i$th entry of vector $\vect{k}_{*h}$ is $(\mathcal{H}'\, k(\vect{x}_*,\cdot))(\vect{x}_i)$, and $\vect{y}$ is the vector of observations. Here $\mathcal{H}'$ denotes that the operator is applied to the second variable $\vect{x}'$ of the argument. With the series expansion \eqref{eq:approximation} we can easily approximate
\begin{equation} \label{eq:ip_approximation}
\begin{split} 
  (\mathcal{H} \, \mathcal{H}' \, k)(\vect{x},\vect{x}')
   & \approx \sum_j S(\sqrt{\lambda_j}) \,
   (\mathcal{H}\, \phi_j)(\vect{x}) \, (\mathcal{H}\, \phi_j)(\vect{x}'), \\
   (\mathcal{H}'\, k(\vect{x}_*,\cdot))(\vect{x}')
   & \approx \sum_j S(\sqrt{\lambda_j}) \,
   \phi_j(\vect{x}_*) \, (\mathcal{H}\, \phi_j)(\vect{x}').
\end{split} 
\end{equation} 
After applying the matrix inversion lemma \eqref{eq:invprob-solution} becomes
\begin{align}
\begin{split} \label{eq:ip-solution-approx}
  \mathbb{E}[f_*] &\approx 
    \vectb{\phi}_*\T
    (\tilde{\vectb{\Phi}}\T\tilde{\vectb{\Phi}} + \sigma_\mathrm{n}^2 \vectb{\Lambda}^{-1})^{-1}
    \tilde{\vectb{\Phi}}\T \vect{y}, \\
  \mathbb{V}[f_*] &\approx 
    \sigma_\mathrm{n}^2 \vectb{\phi}_*\T
    (\tilde{\vectb{\Phi}}\T\tilde{\vectb{\Phi}} + \sigma_\mathrm{n}^2 \vectb{\Lambda}^{-1})^{-1}
    \vectb{\phi}_*,
\end{split}
\end{align}
where $\tilde{\vectb{\Phi}}_{ij} = (\mathcal{H} \phi_j)(\vect{x}_i)$ and $\vectb{\phi}_*$ is as defined in \eqref{eq:GP-solution-approx}. The hyperparameter estimation methods discussed in Section~\ref{sec:hyperpar} can also be easily extended to this case.

Another (related) type of model is the following model arising in the context of latent force models \citep[LFM,][]{Alvarez+Luengo+Lawrence:2013}
\begin{equation} \label{eq:lfm}
\begin{split}
  f(\vect{x}) &\sim \GP(0, k(\vect{x},\vect{x}')), \\
  \mathcal{L} g &= f, \\
  y_i &= g(\vect{x}_i) + \varepsilon_i,
\end{split}
\end{equation}
where $\mathcal{L}$ is a linear operator. We can now write $\mathcal{H} = \mathcal{L}^{-1}$, where $\mathcal{L}^{-1}$ is the Green's operator associated with the operator $\mathcal{L}$ and hence the model becomes a special case of \eqref{eq:invprob}. The approximation to the operator $\mathcal{L}^{-1}$ on the given basis can be easily formed by using, for example, by projecting it onto the basis or by using point collocation. A particularly simple cases arises when the operator itself contains of Laplace operators, for example, when it has the form $\mathcal{L} = \nabla^2$. In that case the projection of the operator becomes diagonal.

\section{Convergence Analysis}
\label{sec:convergence-analysis}
In this section we analyze the convergence of the proposed approximation when the size of the domain $\Omega$ and the number of terms in the series grows to infinity. We start by analyzing a univariate problem in the domain $\Omega = [-L,L]$ and with Dirichlet boundary conditions and then generalize the result to $d$-dimensional cubes $\Omega = [-L_1,L_1] \times \cdots \times [-L_d,L_d]$. Then we analyze the truncation error as function of the number of terms in the series. We also discuss how the analysis could be extended to other types of basis functions.

\subsection{Univariate Dirichlet Case} \label{sec:dirichlet1d}
In the univariate case, the $m$-term approximation has the form
\begin{equation}
  \widetilde{k}_m(x,x')
  = \sum_{j=1}^m S(\sqrt{\lambda_j}) \, \phi_j(x) \, \phi_j(x'),
\label{eq:1d_kapp_m}
\end{equation}
where the eigenfunctions and eigenvalues for $j=1,2,\ldots$ are:
\begin{equation}
\begin{split}
 \phi_j(x) 
    &= \frac{1}{\sqrt{L}} \, \sin\left(\frac{\pi \, j \, (x + L)}{2L} \right)
 \quad \text{and} \quad
  \lambda_j
    = \left(\frac{\pi \, j}{2L} \right)^2.
\end{split}
\end{equation}
The true covariance function $k(x,x')$ is assumed to be stationary and have a spectral density with the following properties. It is uniformly bounded $S(\omega) = B < \infty$ and has at least one bounded derivative $|S'(\omega)| = D < \infty$ on $\omega > 0$. The following integrals are also assumed to be bounded: $\int_{0}^{\infty} S(\omega) \dd\omega  = A < \infty$ and $\int_{0}^{\infty} |S'(\omega)| \dd\omega = C < \infty$. We also assume that our training and test sets are constrained in the area $[-\widetilde{L},\widetilde{L}]$, where $\widetilde{L} < L$, and thus we are only interested in the case $x,x' \in [-\widetilde{L},\widetilde{L}]$. For the purposes of analysis we also assume that $L$ is bounded below by a constant.

The univariate convergence result can be summarized as the following theorem which is proved in Appendix \ref{app:conv_proof1}.
\begin{mytheorem} \label{the:1d_kapp_inf}
There exists a constant $E$ (independent of $m$, $x$, and $x'$) such that
\begin{equation}
  \left| k(x,x') - \widetilde{k}_m(x,x') \right| \leq \frac{E}{L} 
  + \frac{2}{\pi} \,
    \int_{\frac{\pi \, m}{2L}}^{\infty}
    S(\omega) \dd\omega,
\end{equation}
which in turn implies that uniformly
\begin{equation}
  \lim_{L \to \infty} 
    \left[ \lim_{\vphantom{L} m \to \infty} \widetilde{k}_m(x,x') \right] = k(x,x').
\end{equation}
\end{mytheorem}
\begin{myremark}
Note that we cannot simply exchange the order of the limits in the above theorem. However, the theorem does ensure the convergence of the approximation in the joint limit $m,L \to \infty$ provided that we add terms to the series fast enough such that $m / L \to \infty$. That is, in this limit, the approximation $\widetilde{k}_m(x,x')$ converges uniformly to $k(x,x')$.
\end{myremark}

As such, the results above only ensure the convergence of the prior covariance functions. However, it turns out that this also ensures the convergence of the posterior as is summarized in the following corollary.
\begin{mycorollary}
  Because the Gaussian process regression equations only involve pointwise evaluations of the kernels, it also follows that the posterior mean and covariance functions converge uniformly to the exact solutions in the limit $m,L \to \infty$.
\end{mycorollary}

\begin{proof}
Analogous to proof of Theorem 2.2 in \citet{Sarkka+Piche:2014}.
\end{proof}

\subsection{Multivariate Cartesian Dirichlet Case}
In order to generalize the results from the previous section, we turn our attention to a $d$-dimensional inputs space with rectangular domain $\Omega = [-L_1,L_1] \times \cdots \times [-L_d,L_d]$ with Dirichlet boundary conditions. In this case we consider a truncated $m = {\hat m}^d$ term approximation of the form
\begin{multline}
  \widetilde{k}_m(\vect{x},\vect{x}') \\
  = \sum_{j_1,\ldots,j_d=1}^{\hat m}
   S(\sqrt{\lambda_{j_1,\ldots,j_d}}) \,
   \phi_{j_1,\ldots,j_d}(\vect{x}) \, \phi_{j_1,\ldots,j_d}(\vect{x}')
\label{eq:nd_kapp_m}
\end{multline}
with the eigenfunctions and eigenvalues
\begin{equation}
\begin{split}
 \phi_{j_1,\ldots,j_d}(x) 
  = \prod_{k=1}^d \frac{1}{\sqrt{L_k}} \,
    \sin\left(\frac{\pi \, j_k \, (x_k + L_k)}{2L_k} \right)
\end{split}
\end{equation}
and
\begin{equation}
\begin{split}
  \lambda_{j_1,\ldots,j_d}
  = \sum_{k=1}^d \left(\frac{\pi \, j_k}{2L_k} \right)^2.
\end{split}
\end{equation}
The true covariance function $k(\vect{x},\vect{x}')$ is assumed to be homogeneous (stationary) and have a spectral density $S(\vectb{\omega})$ which satisfies the one-dimensional assumptions listed in the previous section in each variable. Furthermore, we assume that the training and test sets are contained in the $d$-dimensional cube $[-\widetilde{L},\widetilde{L}]^d$ and that $L_k$s are bounded from below.

The following result for this $d$-dimensional case is proved in Appendix~\ref{app:conv_proof2}.
\begin{mytheorem} \label{the:nd_kapp_inf}
There exists a constant $E$ (independent of $m$, $d$, $\vect{x}$, and $\vect{x}'$) such that
\begin{equation}
\begin{split}
  \left| k(\vect{x},\vect{x}') - \widetilde{k}_m(\vect{x},\vect{x}') \right|
  \leq \frac{E \, d}{L}
  + \frac{1}{\pi^d} 
  \int_{\norm{\vectb{\omega}} \geq \frac{\pi \, {\hat m}}{2L}} S(\vectb{\omega}) \dd\vectb{\omega},
\end{split}
\end{equation}
where $L = \min_k L_k$, which in turn implies that uniformly
\begin{equation}
\begin{split}
  \lim_{L_1,\ldots,L_d \to \infty}
    \left[ \lim_{\vphantom{L_1} m \to \infty} \widetilde{k}_m(\vect{x},\vect{x}') \right] = k(\vect{x},\vect{x}').
\end{split}
\end{equation}
\end{mytheorem}
\begin{myremark}
  Analogously as in the one-dimensional case we cannot simply exchange the order of the limits above. Furthermore, we need to add terms fast enough so that ${\hat m} / L_k \to \infty$ when $m,L_1,\ldots,L_d \to \infty$.
\end{myremark}
\begin{mycorollary}
As in the one-dimensional case, the uniform convergence of the prior covariance function also implies uniform convergence of the posterior mean and covariance in the limit $m,L_1,\ldots,L_d \to \infty$.
\end{mycorollary}

\subsection{Scaling of Error with Increasing $\hat m$} \label{sec:scaling}
Using the Dirichlet eigenfunction basis, we can also investigate the truncation error with an increasing number of series expansion terms $m = \hat{m}^d$. If we take a look at the bound in Theorem~\ref{the:nd_kapp_inf}, we can see that it has the form
\begin{equation}
\begin{split}
\frac{E \, d}{L}
  + \frac{1}{\pi^d} 
  \int_{\norm{\vectb{\omega}} \geq \frac{\pi \, {\hat m}}{2L}} S(\vectb{\omega}) \dd\vectb{\omega},
\end{split}
\end{equation}
where the first term is independent of $\hat m$ and is a linear function of $d$. The latter term in turn depends on $\hat m$ and in that sense defines the scaling of error in the number of series terms.

It is worth noting that due to Remarks~\ref{rem:1d_kapp_inf2} and \ref{rem:nd_kapp_inf2} we could actually tighten the bound by introducing $\hat m$-dependence to $E$, but it does not affect the order of scaling, because the dependence on the dimensionality in that term is linear. Furthermore, the latter term actually depends on the ratio $\hat m / L$ and hence there is a coupling between the number of terms and the size of the domain $L$. However, we can still get idea of the convergence speed by fixing $L$. 

Let us start by considering the case when $S(\|\vect{\omega}\|)$ is bounded by a reciprocal of a polynomial which is the case, for example, for the Mat\'ern covariance function. We get the following theorem.

\begin{mytheorem} \label{the:rate1}
Assume that where exists a constant $D$ such that $S(\norm{\vectb{\omega}}) \le \frac{D}{\norm{\vectb{\omega}}^{d + a}}$ for some $a > 0$. Then we have
\begin{equation}
\begin{split}
  \int_{\norm{\vectb{\omega}} \geq \frac{\pi \, {\hat m}}{2L}} \, S(\norm{\vectb{\omega}}) \,
  \dd\vectb{\omega} \le
  \frac{D'}{m^{a/d}},
\end{split}
\end{equation}
where $m = {\hat m}^d$ for some constant $D'$ (which depends on $L$ and $d$).
\end{mytheorem}

\begin{proof} First recall that 
\begin{equation}
\begin{split}
  \int_{\norm{\vectb{\omega}} \geq \frac{\pi \, {\hat m}}{2L}} \, S(\norm{\vectb{\omega}}) \,
  \dd\vectb{\omega}
  &= 
  \frac{2 \pi^{d/2}}{\Gamma(d/2)} \,
   \int_{\frac{\pi \, \hat{m}}{2L}}^{\infty} S(r) \, r^{d-1} \, \dd r,
\end{split}
\end{equation}
where $\Gamma(\cdot)$ is the gamma function, and hence to analyze the scaling as function of $m$, it is enough to investigate the scaling of the term $\int_{\frac{\pi \, \hat{m}}{2L}}^{\infty} S(r) \, r^{d-1} \, dr$. We now get
\begin{equation}
\begin{split}
   \int_{\frac{\pi \, \hat{m}}{2L}}^{\infty} S(r) \, r^{d-1} \, dr 
  &\le
   \int_{\frac{\pi \, \hat{m}}{2L}}^{\infty} \frac{D}{r^{a+1}} \, dr \\
   &=
  \left( \frac{1}{a} \right) \, \left( \frac{2L}{\pi \, \hat{m}} \right)^{a}   \\
   &=
  \underbrace{\left( \frac{(2L)^a}{\pi^a \, a} \right)}_{D'} \, \left( \frac{1}{m^{a/d}} \right), \\
\end{split}
\end{equation}
where we have recalled that $m = \hat{m}^d$.
\end{proof}
The result in the above theorem tells that by selecting an appropriate differentiation order for the covariance function, we can make the convergence speed arbitrarily large. In particular, if we select $a = d/2$, we get the Monte Carlo rate, and with $a = d$, we get a convergence rate of $\sim 1/m$. 

In order to analyze the squared exponential covariance function with spectral density
\begin{equation}  \label{eq:se_spec}
\begin{split}
  S(\vect{\omega}) & = \prod_{i=1}^d \left[ s^2 \, \sqrt{2 \pi} \,\ell \,
  \exp\left(-\frac{\ell^2 \, \omega_i^2}{2} \right) \right],
\end{split}
\end{equation}
we recall that the integral $\int_{\norm{\vectb{\omega}} \geq \frac{\pi \, {\hat m}}{2L}} \, S(\norm{\vectb{\omega}}) \,  \dd\vectb{\omega}$ was actually used for bounding a more tight bound $\int_{\frac{\pi \, {\hat m}}{2L_1}}^{\infty} \cdots \int_{\frac{\pi \, {\hat m}}{2L_d}}^{\infty} S( \omega_1, \ldots, \omega_d)  \dd\omega_1 \cdots \dd\omega_d$ appearing in Equation~\eqref{eq:tighter_bound}. In terms of that (original) bound, we get the following theorem.

\begin{mytheorem} \label{the:se_rate}
Assume that the spectral density is of the squared exponential form \eqref{eq:se_spec}. Then we have
\begin{equation} 
\begin{split}
&\int_{\frac{\pi \, {\hat m}}{2L_1}}^{\infty} \cdots \int_{\frac{\pi \, {\hat m}}{2L_d}}^{\infty} S( \omega_1, \ldots, \omega_d)  \dd\omega_1 \cdots \dd\omega_d \\
&\le D'' \, \frac{\exp(-\gamma \, d \, m^{2/d})}{m} \le \frac{D''}{m},
\end{split}
\end{equation}
for some constants $D'', \gamma > 0$ (which depend on $d$ and $L$).
\end{mytheorem}

\begin{proof}
Due to separability of the spectral density, we have
\begin{equation}
\begin{split}
  &
  \int_{\frac{\pi \, {\hat m}}{2L_1}}^{\infty} \cdots \int_{\frac{\pi \, {\hat m}}{2L_d}}^{\infty}
  S( \omega_1, \ldots, \omega_d) 
  \dd\omega_1 \cdots \dd\omega_d \\
  &= 
  \left( s^2 \, \sqrt{2 \pi} \,\ell \right)^d \, \prod_{i=1}^d \int_{\frac{\pi \, {\hat m}}{2L_i}}^{\infty} \exp\left(-\frac{\ell^2 \, \omega_i^2}{2} \right) \dd\omega_i \\
  &\le
  \left( s^2 \, \sqrt{2 \pi} \,\ell \right)^d \, \left[ \int_{\frac{\pi \, {\hat m}}{2L}}^{\infty} \exp\left(-\frac{\ell^2 \, \omega_i^2}{2} \right) \dd\omega_i \right]^d,
\end{split}
\end{equation}
where $L = \min_k L_k$. By using the bound from \citet{Feller:1968}, Section~VII.1, Lemma~2, we get that is this
\begin{equation}
\begin{split}
  &\le
  \left( s^2 \, \sqrt{2 \pi} \,\ell^2 \right)^d \,
  \left[  \exp\left(-\frac{1}{2} \left[ \frac{\pi \, \hat m}{2L \, \ell} \right]^2 \right) \,
  \frac{2L \, \ell}{\pi \, \hat m} \right]^d \\
  &= D'' \, \frac{\exp(-\gamma \, d \, m^{2/d})}{m}.
\end{split}
\end{equation}
\end{proof}
The above theorem tells that the convergence in the squared exponential case is faster than $\sim 1 / m$, independently of the dimensionality $d$. It is worth noting though that the bound is not independent of the dimensionality in the sense that the constants do depend on it. Strictly speaking, the convergence rate is $h(d) / m$, for some function $h$ which depends on $d$. However, as function of $m$, this rate is independent of the dimensionality.

\subsection{Other Domains}
It would also be possible carry out similar convergence analysis, for example, in a spherical domain. In that case the technical details become slightly more complicated, because instead sinusoidals we will have Bessel functions and the eigenvalues no longer form a uniform grid. This means that instead of Riemann integrals we need to consider weighted integrals where the distribution of the zeros of Bessel functions is explicitly accounted for. It might also be possible to use some more general theoretical results from mathematical analysis to obtain the convergence results. However, due to these technical challenges more general convergence proof will be developed elsewhere.

There is also a similar technical challenge in the analysis when the basis functions are formed by assuming an input density (see Section~\ref{sec:inner_product_point_of_view}) instead of a bounded domain. Because explicit expressions for eigenfunctions and eigenvalues cannot be obtained in general, the elementary proof methods which we used here cannot be applied. Therefore the convergence analysis of this case is also left as a topic for future research.

\section{Relationship to Other Methods}
\label{sec:relationship-to-other-methods}
In this section we compare our method to existing sparse GP methods from a theoretical point of view. We consider two different classes of approaches: a class of inducing input methods based on the Nystr{\"o}m approximation \citep[following the interpretation of][]{Quinonero-Candela+Rasmussen:2005,Bui+Yan+Turner:2017}, and direct spectral approximations.

\subsection{Methods from the Nystr{\"o}m Family}
A crude but rather effective scheme for approximating the eigendecomposition of the Gram matrix is the Nystr{\"o}m method \citep[see, \eg,][for the integral approximation scheme]{Baker:1977}. This method is based on choosing a set of $m$ inducing inputs $\vect{x}_\mathrm{u}$ and scaling the corresponding eigendecomposition of their corresponding covariance matrix $\vect{K}_{\mathrm{u},\mathrm{u}}$ to match that of the actual covariance. The Nystr{\"o}m approximations to the $j$th eigenvalue and eigenfunction are 
\begin{align} 
  \widetilde{\lambda}_j &= \frac{1}{m} \, \lambda_{\mathrm{u},j}, 
    \label{eq:nystrom-lambda}\\
  \widetilde{\phi}_j(\vect{x}) &= \frac{\sqrt{m}}{\lambda_{\mathrm{u},j}} \,
     k(\vect{x},\vect{x}_{\mathrm{u}}) \, \vectb{\phi}_{\mathrm{u},j},
     \label{eq:nystrom-phi}
\end{align}
where $\lambda_{\mathrm{u},j}$ and $\vectb{\phi}_{\mathrm{u},j}$ correspond to the $j$th eigenvalue and eigenvector of $\vect{K}_{\mathrm{u},\mathrm{u}}$. This scheme was originally introduced to the GP context by \citet{Williams+Seeger:2001}. They presented a sparse scheme, where the resulting approximate prior covariance over the latent variables is $\vect{K}_{\mathrm{f},\mathrm{u}} \vect{K}_{\mathrm{u},\mathrm{u}}^{-1} \vect{K}_{\mathrm{u},\mathrm{f}}$, which can be derived directly from Equations~\eqref{eq:nystrom-lambda} and \eqref{eq:nystrom-phi}.

As discussed by \citet{Quinonero-Candela+Rasmussen:2005}, the Nystr{\"o}m method by \citet{Williams+Seeger:2001} does not correspond to a well-formed probabilistic model. However, several methods modifying the inducing point approach are widely used. The \emph{Subset of Regressors} \citep[SOR,][]{Smola+Bartlett:2001} method uses the Nystr{\"o}m approximation scheme for approximating the whole covariance function, 
\begin{equation}
  k_\mathrm{SOR}(\vect{x},\vect{x}') =
    \sum_{j=1}^m \widetilde{\lambda}_j \, 
      \widetilde{\phi}_j(\vect{x}) \, \widetilde{\phi}_j(\vect{x}'),
\end{equation}
whereas the sparse Nystr{\"o}m method \citep{Williams+Seeger:2001} only replaces the training data covariance matrix. The SOR method is in this sense a complete Nystr{\"o}m approximation to the full GP problem. A method in-between is the \emph{Deterministic Training Conditional} \citep[DTC, ][]{Csato+Opper:2002, Seeger+Williams+Lawrence:2003} method which retains the true covariance for the training data, but uses the approximate cross-covariances between training and test data. For DTC, tampering with the covariance matrix causes the result not to actually be a Gaussian process. The \emph{Variational Approximation} \citep[VAR,][]{Titsias:2009} method modifies the DTC method by an additional trace term in the likelihood that comes from the variational bound.

The \emph{Fully Independent (Training) Conditional} \citep[FIC,][]{Quinonero-Candela+Rasmussen:2005} method \citep[originally introduced as \emph{Sparse Pseudo-Input GP} by][]{Snelson+Ghahramani:2006} is also based on the Nystr{\"o}m approximation but contains an additional diagonal term replacing the diagonal of the approximate covariance matrix with the values from the true covariance. The corresponding prior covariance function for FIC, is thus
\begin{equation}
\begin{split}
  &k_\mathrm{FIC}(\vect{x}_i,\vect{x}_j) \\ &= 
  k_\mathrm{SOR}(\vect{x}_i,\vect{x}_j) + 
    \delta_{i,j} (k(\vect{x}_i,\vect{x}_j) - k_\mathrm{SOR}(\vect{x}_i,\vect{x}_j)),
\end{split}
\end{equation}
where $\delta_{i,j}$ is the Kronecker delta.

Figure~\ref{fig:inducing} illustrates the effect of the approximations compared to the exact correlation structure in the GP. The dashed contours show the exact correlation contours computed for three locations with the squared exponential covariance function. Figure~\ref{fig:inducing-fic} shows the results for the FIC approximation with 16 inducing points (locations shown in the figure). It is clear that the number of inducing points or their locations are not sufficient to capture the correlation structure. For similar figures and discussion on the effects of the inducing points, see \citet{Vanhatalo+Pietilainen+Vehtari:2010}. This behavior is not unique to SOR or FIC, but applies to all the methods from the Nystr{\"o}m family.

\begin{figure*}[!t]
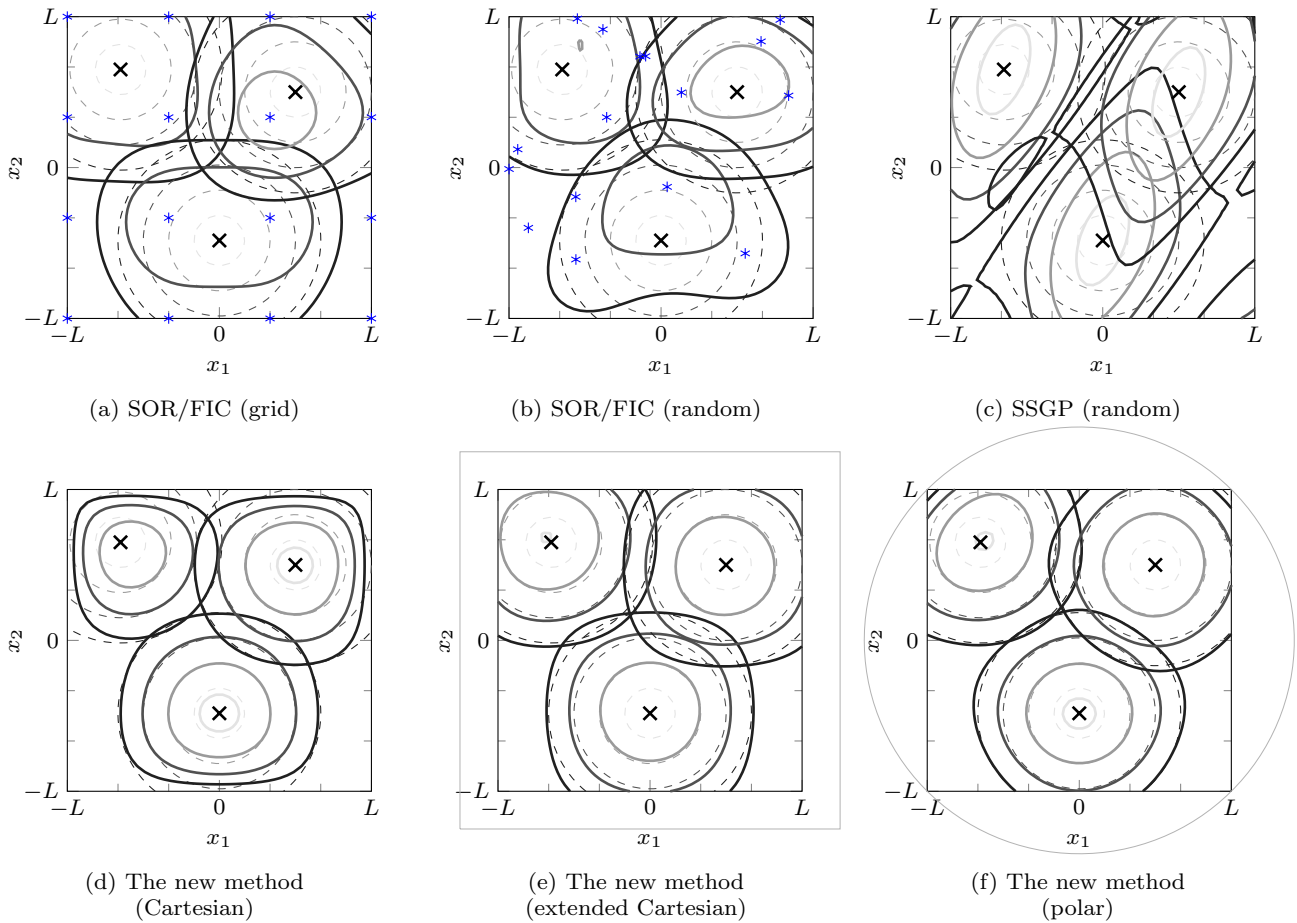


  \setlength\figureheight{0.23\textwidth}
  \setlength\figurewidth{0.23\textwidth} 

  \pgfplotsset{ytick align=inside, 
               xtick align=inside,
               every axis y label/.style={at={(ticklabel
               cs:0.5)},rotate=90,anchor=center},
               minor tick num=2}

  \hspace*{\fill}%
  \begin{subfigure}[b]{0.32\textwidth}
    \centering
    \tikzsetnextfilename{tikz-inducing-fic}
    \sffamily \footnotesize
    \input{fig/inducing-fic.tex}
    
    \caption{SOR/FIC (grid)}
    \label{fig:inducing-fic}
  \end{subfigure}%
  \hspace*{\fill}%
  \begin{subfigure}[b]{0.32\textwidth}
    \centering 
    \tikzsetnextfilename{tikz-inducing-fic-rand}
    \sffamily \footnotesize
    \input{fig/inducing-fic-rand.tex}
    
    \caption{SOR/FIC (random)}
    \label{fig:inducing-fic-rand}
  \end{subfigure}%
  \hspace*{\fill}%
  \begin{subfigure}[b]{0.32\textwidth}
    \centering 
    \tikzsetnextfilename{tikz-inducing-ssgp}
    \sffamily \footnotesize
    \input{fig/inducing-ssgp.tex}
    
    \caption{SSGP (random)}
    \label{fig:inducing-ssgp}
  \end{subfigure}%
  \hspace*{\fill}
  \\
  \hspace*{\fill}%
  \begin{subfigure}[b]{0.32\textwidth}
    \centering
    \tikzsetnextfilename{tikz-inducing-new}
    \sffamily \footnotesize
    \input{fig/inducing-new.tex}
    
    \caption{Our method \\ (Cartesian)}
    \label{fig:inducing-new}
  \end{subfigure}%
  \hspace*{\fill}%
  \begin{subfigure}[b]{0.32\textwidth}
    \centering 
    \captionsetup{justification=centering}
    \tikzsetnextfilename{tikz-inducing-new-bc}
    \sffamily \footnotesize
    \input{fig/inducing-new-bc.tex}
    
    \caption{Our method \\ (extended Cartesian)}
    \label{fig:inducing-new-bc}
  \end{subfigure}%
  \hspace*{\fill}%
  \begin{subfigure}[b]{0.32\textwidth}
    \centering
    \tikzsetnextfilename{tikz-inducing-new-disk}
    \sffamily \footnotesize
    \input{fig/inducing-new-disk.tex}
    
    \caption{Our method \\ (polar)}
    \label{fig:inducing-new-disk}
  \end{subfigure}%
  \hspace*{\fill}%
  \tikzexternaldisable
  \caption{Correlation contours computed for three locations 
    (\ref{addplot:inducing-fic0}) corresponding to the squared exponential 
    covariance function (exact contours dashed). 
    The rank of each approximation is $m=16$, 
    and the locations of the inducing inputs are marked with blue stars 
    (\ref{addplot:inducing-fic1}). The hyperparameters are the same in each
    figure. The domain boundary is shown in thin grey 
    (\ref{addplot:inducing-new-bc1}) if extended outside the box.}
  \label{fig:inducing}
  \tikzexternalenable
\end{figure*}

\subsection{Direct Spectral Methods}
The spectral analysis and series expansions of Gaussian processes has a long history. A classical result \citep[see, \eg,][and references therein]{Loeve:1963,Van-Trees:1968,Adler:1981,Cramer+Leadbetter:1967} is that in a compact set $\vect{x},\vect{x}' \in \Omega \subset \mathbb{R}^d$ defined continuous covariance function can be expanded into a Mercer series
\begin{equation}
  K(\vect{x},\vect{x}') = \sum_j \gamma_j \, \varphi_j(\vect{x}) \,  \varphi_j(\vect{x}'),
  \label{eq:mercer_again}
\end{equation}
where $\gamma_j$ and $\varphi_j$ are the eigenvalues and the orthonormal eigenfunctions of the covariance function, respectively, defined as
\begin{equation}
  \int_{\Omega} K(\vect{x},\vect{x}') \, \varphi_j(\vect{x}') \, \dd \vect{x}'
  = \gamma_j \, \varphi_j(\vect{x}).
  \label{eq:eigen_again}
\end{equation}
Furthermore, the convergence happens absolutely and uniformly \citep{Adler:1981}. This also means that we can approximate the covariance function with a finite truncation of the series and the approximation is guaranteed to converge to the exact covariance function when the number of terms is increased.

In the case of Gaussian processes we get that a zero mean Gaussian process with the covariance function $K(\vect{x},\vect{x}')$ has the following Karhunen--Loeve series expansion in the domain $\Omega$:
\begin{equation}
  f(\vect{x}) = \sum_j f_j \, \varphi_j(\vect{x}),
\end{equation}
where $f_j$ are independent zero-mean Gaussian random variables with variances $\gamma_j$. The (also classical) generalization of this classical result to more general inner products was already discussed in Section~\ref{sec:inner_product_point_of_view}. 

In the case that $\Omega$ is not compact, but covers the whole $\mathbb{R}^d$, and when the covariance function is homogeneous, the eigenvalues defined by \eqref{eq:eigen_again} are no longer discrete, but they can only be expressed as the spectral density $S(\vectb{\omega})$ which can be seen as a continuum of eigenvalues. The eigenfunctions become complex exponentials, that is, sines and cosines -- which in turn are a subset of eigenfunctions of Laplace operator. In this background, what \eqref{eq:approximation} essentially says is that we can approximate the Mercer expansion \eqref{eq:mercer_again} by using the basis consisting of the Laplacian eigenfunctions $\varphi_j(\vect{x}) \approx \phi_j(\vect{x})$ and point-wise evaluations of the spectral density at the Laplacian eigenvalues $\gamma_j \approx S(\sqrt{-\lambda_j})$.  

Another related classical connection is to the works in the relationship of spline interpolation and Gaussian process priors \citep{Wahba:1978,Kimeldorf+Wahba:1970,Wahba:1990}. In particular, it is well-known \citep[see, \eg,][]{Wahba:1990} that spline smoothing can be seen as Gaussian process regression with a specific choice of covariance function. The relationship of the spline regularization with Laplace operators then leads to series expansion representations that are closely related to the approximations considered here.

In more recent machine learning context, the sparse spectrum GP (SSGP) method introduced by \citet{Lazaro-Gredilla+Quinonero-Candela+Rasmussen+Figueiras-Vidal:2010} uses the spectral representation of the covariance function for drawing random samples from the spectrum. These samples are used for representing the GP on a trigonometric basis
\begin{multline}
  \vectb{\phi}(\vect{x}) = 
  \big( 
   \cos(2\pi \, \vect{s}_1\T \vect{x}) ~ \sin(2\pi \, \vect{s}_1\T \vect{x}) ~ \ldots \\
   \cos(2\pi \, \vect{s}_h\T \vect{x}) ~ \sin(2\pi \, \vect{s}_h\T \vect{x})
  \big),
\end{multline}
where the spectral points $\vect{s}_r, r = 1,2,\ldots,h$ ($2h=m$), are sampled from the spectral density of the original stationary covariance function (following the normalization convention used in the original paper). The covariance function corresponding to the SSGP scheme is now of the form
\begin{align} \label{eq:k_SSGP}
  k_\mathrm{SSGP}(\vect{x},\vect{x}') 
  &= \frac{2 \sigma^2}{m} \, \vectb{\phi}(\vect{x}) \, \vectb{\phi}\T(\vect{x}') \nonumber \\
  &= \frac{\sigma^2}{h} 
    \sum_{r=1}^h \cos\left( 2\pi \, \vect{s}_r\T (\vect{x}-\vect{x}') \right),
\end{align}
where $\sigma^2$ is the magnitude scale hyperparameter. This representation of the sparse spectrum method converges to the full GP in the limit of the number of spectral points going to infinity, and it is the preferred formulation of the method in one or two dimensions \citep[see][for discussion]{Lazaro-Gredilla:2010}. We can interpret the SSGP method in \eqref{eq:k_SSGP} as a Monte Carlo approximation of the Wiener--Khintchin integral. In order to have a representative sample of the spectrum, the method typically requires the number of spectral points to be large. For high-dimensional inputs the number of required spectral points becomes overwhelming, and optimizing the spectral locations along with the hyperparameters attractive. However, as argued by \citet{Lazaro-Gredilla+Quinonero-Candela+Rasmussen+Figueiras-Vidal:2010}, this option does not converge to the full GP and suffers from overfitting to the training data \citep[see][for discussion on overfitting]{Gal+Turner:2015}.

Contours for the sparse spectrum SSGP method are visualized in Figure~\ref{fig:inducing-ssgp}. Here the spectral points were chosen at random following \citet{Lazaro-Gredilla:2010}. Because the basis functions are spanned using both sines and cosines, the number of spectral points was $h=8$ in order to match the rank $m=16$. These results agree well with those presented in the \citet{Lazaro-Gredilla+Quinonero-Candela+Rasmussen+Figueiras-Vidal:2010} for a one-dimensional example. For this particular set of spectral points some directions of the contours happen to match the true values very well, while other directions are completely off. Increasing the rank from 16 to 100 would give comparable results to the other methods.

Recently \citet{Hensman+Durrande+Solin:2018} presented a variational Fourier feature approximation for Gaussian processes that was derived for the Mat\'ern class of kernels, where the approximation structure is set up by a low-rank plus diagonal structure. The key differences here are the fully diagonal (independent) structure in the $\vect{K}_{u,u}$ matrix (giving rise to additional speed-up) and the generality of only requiring the spectral density function to be known.

While SSGP is based on a sparse spectrum, the reduced-rank method proposed in this paper aims to make the spectrum as `full' as possible at a given rank. While SSGP can be interpreted as a Monte Carlo integral approximation, the corresponding interpretation to the proposed method would be a numerical quadrature-based integral approximation (\cf\ the convergence proof in Appendix~\ref{app:conv_proof1}). Figure~\ref{fig:inducing-new} shows the same contours obtained by the proposed reduced-rank method. Here the eigendecomposition of the Laplace operator has been obtained for the square $\Omega = [-L,L] \times [-L,L]$ with Dirichlet boundary conditions. The contours match well with the full solution towards the middle of the domain. The boundary effects drive the process to zero, which is seen as distortion near the edges.

Figure~\ref{fig:inducing-new-bc} shows how extending the boundaries just by 25\% and keeping the number of basis functions fixed at 16, gives good results. The last Figure~\ref{fig:inducing-new-disk} corresponds to using a disk shaped domain instead of the rectangular. The eigendecomposition of the Laplace operator is done in polar coordinates, and the Dirichlet boundary is visualized by a circle in the figure.

\subsection{Structure Exploiting and Decomposition Methods}
Other methods for scalable Gaussian processes include many structure exploiting techniques that, similarly to general inducing input methods, aim to be agnostic to the choice of covariance function. They rather exploit the structure of the inputs \citep[see][for discussion on Kronecker and Toeplitz algebra]{Saatcci:2012}, and not the GP prior per se. Most notably, Scalable Kernel Interpolation \citep[SKI, ][]{Wilson+Nickisch:2015} is an inducing point method that achieves $O(n + m \log m)$ time complexity and $\mathcal{O}(n + m)$ space complexity. Through local cubic kernel interpolation, the SKI framework is used in KISS-GP \citep[see][for details]{Wilson+Nickisch:2015} which uses Kronecker and Toeplitz algebra on grids of inducing inputs to speed up inference. 

The computational complexity of the SKI approach scales cubically in the input dimenionality $d$. Other recent methods \citep[\eg,][]{Gardner+Pleiss+Wu+Weinberger+Wilson:2018,Izmailov+Novikov+Kropotov:2018} have reduced the time complexity to linear in $d$ as well (\eg, $\mathcal{O}(d n + d m \log m)$). These methods typically leverage parallelization (well suited for GPU calculations) or iterative methods. 

Furthermore, general methods form numerical linear algebra for approximately solving eigenvalue and singular value problems allow for fast low-rank decompositions. These methods ignore the kernel learning perspective, but can provide useful tools in practice. For example, the pivoted Cholesky decomposition \citep{Harbrecht+Peters+Schneider:2012,Bach:2013} allows constructing a low-rank approximation to an $n \times n$ positive definite matrix in $O(n m^2)$ time. There are also methods for fast randomized singular value decompositions based on subsampled Hadamard transformations \citep[\eg, ][]{Boutsidis+Gittens:2013}, with some further details in \citet{Le+Sarlos+Smola:2013}.
These methods provide speedup to the general linear algebraic problem, but ignore the well structured nature of the specific application to Gaussian process regression with stationary prior covariance functions.

\begin{figure}[!t]

  \setlength{\figurewidth}{.85\columnwidth}
  \setlength{\figureheight}{.8\figurewidth}
  \pgfplotsset{
    legend cell align=right, ...
    every axis y label/.style={at={(0,0.5)},xshift=-2em,rotate=90,anchor=center}, 
    tick label style={font=\footnotesize},
    y tick label style={rotate=90},
    legend pos=north east,
    legend style={inner xsep=1pt,inner ysep=0.5pt,nodes={inner sep=1pt,text depth=0.1em},font=\tiny},
    minor x tick num=1,
    scale only axis
  }
  \centering
  \sffamily \footnotesize %
  \tikzsetnextfilename{tikz-varying-L-KL}
  \input{fig/varying-L-KL.tex}
  \tikzexternalenable
  \caption{The Kullback--Leibler divergence between the approximative and exact GP posterior by varying the boundary $L$ and keeping all other parameters fixed.}
  \label{fig:varying-L}
\end{figure}

\begin{figure*}[!t]
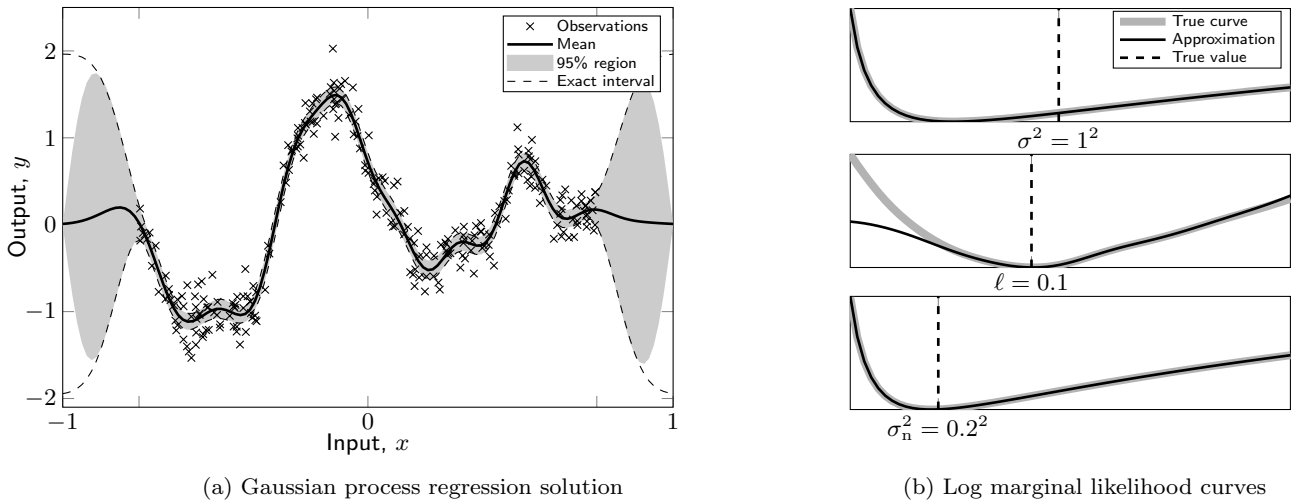

  \tikzexternalenable

  \begin{subfigure}[b]{0.62\textwidth}

    \setlength\figureheight{5.3cm}
    \setlength\figurewidth{0.85\textwidth}

    \sffamily \tiny \hspace{-5mm} \vspace{-1mm}
    \tikzsetnextfilename{tikz-toy-ex-gp}
    \input{fig/toy_ex_gp.tex}
    
    \caption{Gaussian process regression solution}
    \label{fig:toy_example-gp}
  \end{subfigure}%
  ~ %
  \begin{subfigure}[b]{0.38\textwidth}

    \setlength\figureheight{1.5cm}
    \setlength\figurewidth{\textwidth}
    \centering \sffamily \tiny \hspace{-5mm}
    \tikzsetnextfilename{tikz-toy-ex-lik}
    \input{fig/toy_ex_lik.tex}
    
    \caption{Log marginal likelihood curves}
    \label{fig:toy_example-lik}
  \end{subfigure}%

  \caption{(a)~256 data points generated from a GP with hyperparameters 
    $(\sigma^2,\ell,\sigma_\mathrm{n}^2) = (1^2, 0.1, 0.2^2)$, the full 
    GP solution, and an approximate solution with $m=32$. (b)~Negative 
    marginal likelihood curves for the signal variance $\sigma^2$, 
    length-scale $\ell$, and noise variance $\sigma_\mathrm{n}^2$.}
  \label{fig:toy_example}
\end{figure*}

\section{Experiments}
\label{sec:experiments}
In this section we aim to test the convergence results of the method in practice, provide examples of the practical use of the proposed method, and compare it against other methods that are typically used in a similar setting. We start with small simulated one-dimensional datasets, and then provide more extensive comparisons by using real-world data. We also consider an example of data, where the input domain is the surface of a sphere, and conclude our comparison by using a very large dataset to demonstrate what possibilities the computational benefits open.

\subsection{Variation of Domain Size}
\label{sec:domain-size}
In addition to the theoretical analysis of approximation error, we provide a study of the effect of choosing the domain size. We set up an experiment where we simulate data ($n=100$ and all results averaged over 10 independent draws) from GP priors with a squared exponential covariance function with unit hyperparameters and corrupting additive Gaussian noise with variance $\sigma_\mathrm{n}^2 = 0.1^2$. The inputs are chosen uniformly randomly in $[-\tilde{L},\tilde{L}]$ with $\tilde{L}=1$. We study the effect of varying the boundary location $L \in (1,10]$.

Figure~\ref{fig:varying-L} shows the Kullback--Leibler (KL) divergence \citep[see, \eg,][Appendix~A for the the identities for the KL between two multivariate Gaussians]{Rasmussen+Williams:2006} between the approximative GP posterior and the exact GP posterior evaluated over ten uniformly spaced points. The same curve is recalculated for $m=5, 10, 15,$ and $20$. The figure shows that the KL has a single minimum that describes the trade-off of being far enough from the data but close enough not to start losing representative power with the given number of basis functions $m$. Even though the KL suggests there would be a single best choice for $L$, the practical sensitivity to the choice of $L$ is low. Already for $m=5$, the MSE in the posterior mean is $10^{-5}$ (note that the data has unit magnitude scale) when $L$ is chosen one to two length-scales from the data boundary $\tilde{L}$.

\subsection{Comparison Study}
\label{sec:toy-example}
For assessing the performance of different methods we use 10-fold cross-validation and evaluate the following measures based on the validation set: the \emph{standardized mean squared error} (SMSE) and the \emph{mean standardized log loss} (MSLL), respectively defined as:
\begin{equation}
  \mathrm{SMSE} = \sum_{i=1}^{n_*} {(y_{*i} - \mu_{*i})^2 \over \mathrm{Var}[y]},
\end{equation} 
and
\begin{equation}
  \mathrm{MSLL} = \frac{1}{2 n_*} \sum_{i=1}^{n_*} \bigg(
                    { (y_{*i} - \mu_{*i})^2 \over \sigma_{*i}^2} 
                    + \log 2\pi\sigma_{*i}^2 \bigg),
\end{equation} 
where $\mu_{*i} = \mathbb{E}[f(\vect{x}_{*i})]$ and $\sigma_{*i}^2 = \mathbb{V}[f(\vect{x}_{*i})] + \sigma_\mathrm{n}^2$ are the predictive mean and variance for test sample $i=1,2,\ldots,n_*$, and $y_{*i}$ is the actual test value. The training data variance is denoted by $\mathrm{Var}[y]$. For all experiments, the values reported are averages over ten repetitions.

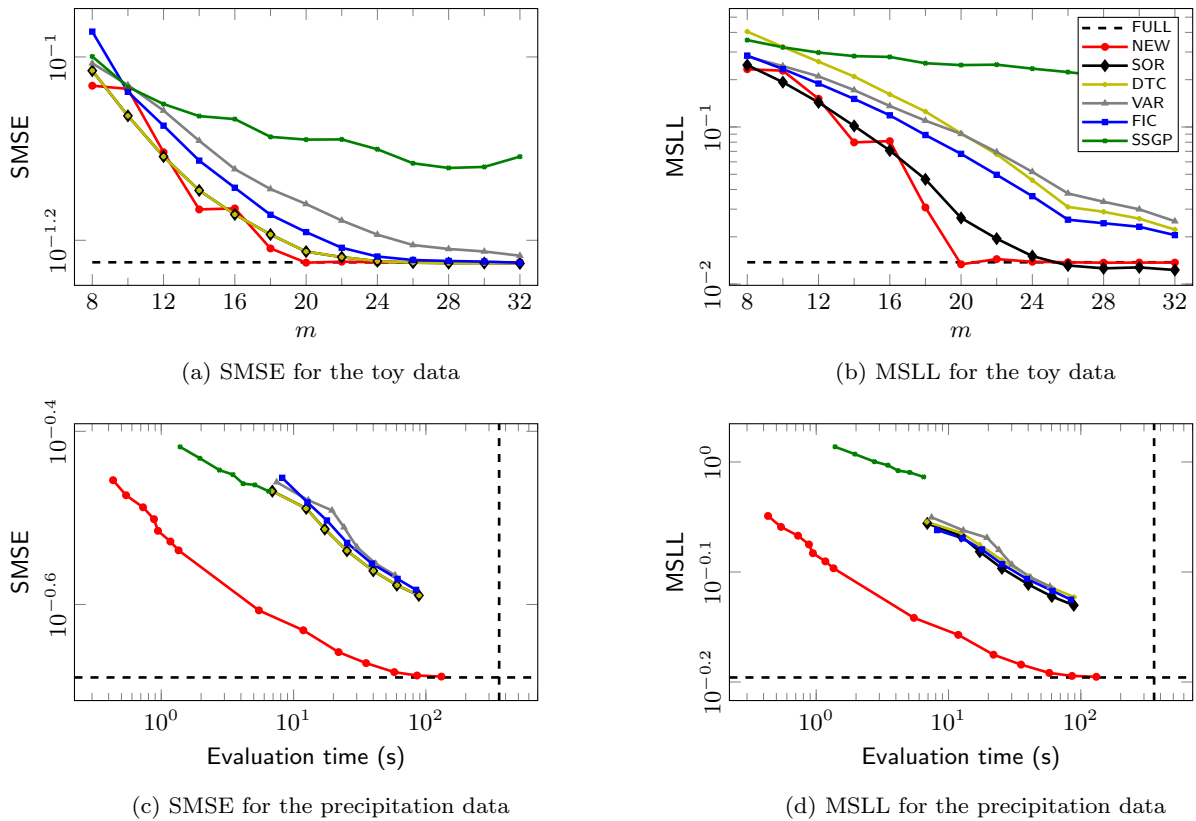
\begin{figure*}[!t]

  \centering

  \setlength{\figureheight}{.17\textheight}
  \setlength{\figurewidth}{.40\textwidth}
  \pgfplotsset{
    legend cell align=right, ...
    every axis y label/.style={at={(0,0.5)},xshift=-2em,rotate=90,anchor=center}, 
    tick label style={font=\footnotesize},
    y tick label style={rotate=90},
    legend pos=north east,
    legend style={inner xsep=1pt,inner ysep=0.5pt,nodes={inner sep=1pt,text depth=0.1em},font=\tiny},
    minor x tick num=1
  }
  \hspace*{\fill}
  \begin{subfigure}[t]{0.48\textwidth}

    \pgfplotsset{every axis legend/.code={\let\addlegendentry\relax}}
    \sffamily \footnotesize %
    \tikzsetnextfilename{tikz-results-toy-smse}
%
%
%
%
%
\definecolor{mycolor1}{rgb}{0.75,0.75,0}%
\begin{tikzpicture}

\begin{axis}[%
width=\figurewidth,
height=\figureheight,
scale only axis,
xmin=7,
xmax=33,
xtick={ 8, 12, 16, 20, 24, 28, 32},
xlabel={$m$},
ymode=log,
ymin=0.0563461350332864,
ymax=0.112925027858055,
yminorticks=true,
ylabel={SMSE},
legend style={draw=black,fill=white,legend cell align=left}
]
\addplot [
color=black,
dashed,
line width=1.0pt
]
table[row sep=crcr]{
8 0.0597069157923099\\
10 0.0597069157923099\\
12 0.0597069157923099\\
14 0.0597069157923099\\
16 0.0597069157923099\\
18 0.0597069157923099\\
20 0.0597069157923099\\
22 0.0597069157923099\\
24 0.0597069157923099\\
26 0.0597069157923099\\
28 0.0597069157923099\\
30 0.0597069157923099\\
32 0.0597069157923099\\
};
\addlegendentry{FULL};

\addplot [
color=red,
solid,
line width=1.0pt,
mark size=1.0pt,
mark=*,
mark options={solid,fill=red}
]
table[row sep=crcr]{
8 0.0930289751975447\\
10 0.0922932940020993\\
12 0.0787278448386207\\
14 0.0681879773046299\\
16 0.0683561005397187\\
18 0.0618294139451827\\
20 0.0596508397528384\\
22 0.0597956944436071\\
24 0.0597175490483689\\
26 0.0597052907002579\\
28 0.0596955432496609\\
30 0.0596999999599942\\
32 0.0597001128287244\\
};
\addlegendentry{NEW};

\addplot [
color=black,
solid,
line width=1.0pt,
mark size=1.7pt,
mark=diamond*,
mark options={solid,fill=black}
]
table[row sep=crcr]{
8 0.0966029250907511\\
10 0.0862492354203134\\
12 0.0778764252003604\\
14 0.0715258061006253\\
16 0.0673397249937524\\
18 0.0640260978647223\\
20 0.0613294533589487\\
22 0.0604832398276781\\
24 0.0598673513057676\\
26 0.0596550550911782\\
28 0.0595825178132748\\
30 0.0595926532731678\\
32 0.059535356544276\\
};
\addlegendentry{SOR};

\addplot [
color=mycolor1,
solid,
line width=1.0pt,
mark size=1.0pt,
mark=+,
mark options={solid,fill=mycolor1}
]
table[row sep=crcr]{
8 0.0966029250907511\\
10 0.0862492354203134\\
12 0.0778764252003604\\
14 0.0715258061006253\\
16 0.0673397249937524\\
18 0.0640260978647223\\
20 0.0613294533589487\\
22 0.0604832398276781\\
24 0.0598673513057676\\
26 0.0596550550911782\\
28 0.0595825178132748\\
30 0.0595926532731678\\
32 0.059535356544276\\
};
\addlegendentry{DTC};

\addplot [
color=gray,
solid,
line width=1.0pt,
mark size=0.7pt,
mark=triangle*,
mark options={solid,fill=gray}
]
table[row sep=crcr]{
8 0.0984124906975452\\
10 0.0931720165059231\\
12 0.0873957012700543\\
14 0.0810203744000252\\
16 0.0754437047841245\\
18 0.0717590207976766\\
20 0.0691246403998752\\
22 0.0663027868112051\\
24 0.0640027346840454\\
26 0.0623495823423567\\
28 0.0617388234241876\\
30 0.0613493820612706\\
32 0.060703038236498\\
};
\addlegendentry{VAR};

\addplot [
color=blue,
solid,
line width=1.0pt,
mark size=0.7pt,
mark=square*,
mark options={solid,fill=blue}
]
table[row sep=crcr]{
8 0.106568707894089\\
10 0.0916128100213541\\
12 0.0841425763701308\\
14 0.0770864548265948\\
16 0.0720035810257198\\
18 0.0672748608642466\\
20 0.064414025541847\\
22 0.0619245815041343\\
24 0.0605809283455465\\
26 0.0600489654717761\\
28 0.0599069129106113\\
30 0.0598081749064949\\
32 0.0596442281614628\\
};
\addlegendentry{FIC};

\addplot [
color=green!50!black,
solid,
line width=1.0pt,
mark size=1.0pt,
mark=x,
mark options={solid,fill=green!50!black}
]
table[row sep=crcr]{
8 0.100113603404585\\
10 0.0927640163151947\\
12 0.088845954552827\\
14 0.0861905642332804\\
16 0.0855503883372889\\
18 0.0818148016117144\\
20 0.0812514489415951\\
22 0.0812900494319207\\
24 0.0793185770926649\\
26 0.0765667964365957\\
28 0.0756720817043526\\
30 0.0758780417811181\\
32 0.0778636701554052\\
};
\addlegendentry{SSGP};

\end{axis}
\end{tikzpicture}%

    \caption{SMSE for the toy data}
    \label{fig:toy-SMSE}
  \end{subfigure}%
  \hspace*{\fill}
  \begin{subfigure}[t]{0.48\textwidth}
    \sffamily \footnotesize %
    \tikzsetnextfilename{tikz-results-toy-msll}
%
%
%
%
%
\definecolor{mycolor1}{rgb}{0.75,0.75,0}%
\begin{tikzpicture}

\begin{axis}[%
width=\figurewidth,
height=\figureheight,
scale only axis,
xmin=7,
xmax=33,
xtick={ 8, 12, 16, 20, 24, 28, 32},
xlabel={$m$},
ymode=log,
ymin=0.00982075356822962,
ymax=0.567481452650072,
yminorticks=true,
ylabel={MSLL},
legend style={draw=black,fill=white,legend cell align=left}
]
\addplot [
color=black,
dashed,
line width=1.0pt
]
table[row sep=crcr]{
8 0.0137708913451369\\
10 0.0137708913451369\\
12 0.0137708913451369\\
14 0.0137708913451369\\
16 0.0137708913451369\\
18 0.0137708913451369\\
20 0.0137708913451369\\
22 0.0137708913451369\\
24 0.0137708913451369\\
26 0.0137708913451369\\
28 0.0137708913451369\\
30 0.0137708913451369\\
32 0.0137708913451369\\
};
\addlegendentry{FULL};

\addplot [
color=red,
solid,
line width=1.0pt,
mark size=1.0pt,
mark=*,
mark options={solid,fill=red}
]
table[row sep=crcr]{
8 0.233153229245066\\
10 0.228245222371451\\
12 0.151233113466669\\
14 0.0798337558764026\\
16 0.0812558629984853\\
18 0.0307443178778784\\
20 0.0133676233194376\\
22 0.0144254688461662\\
24 0.0138760984229297\\
26 0.0137684742820506\\
28 0.013681540121571\\
30 0.0137082018513682\\
32 0.0137078499510644\\
};
\addlegendentry{Hilbert-GP (ours)};

\addplot [
color=black,
solid,
line width=1.0pt,
mark size=1.7pt,
mark=diamond*,
mark options={solid,fill=black}
]
table[row sep=crcr]{
8 0.247755797904985\\
10 0.192899704770246\\
12 0.143481795807436\\
14 0.101523998261427\\
16 0.0708550583560786\\
18 0.0464263119205671\\
20 0.0263650517211808\\
22 0.0194611446986186\\
24 0.0150730850431016\\
26 0.0131391419555893\\
28 0.0125987901797085\\
30 0.0127759423664762\\
32 0.0122970438110433\\
};
\addlegendentry{SOR};

\addplot [
color=mycolor1,
solid,
line width=1.0pt,
mark size=1.0pt,
mark=+,
mark options={solid,fill=mycolor1}
]
table[row sep=crcr]{
8 0.404701145433511\\
10 0.323333204947854\\
12 0.260617299080432\\
14 0.209438072567357\\
16 0.161393018956725\\
18 0.125435178860575\\
20 0.0910526471084743\\
22 0.0670231793015462\\
24 0.045823786526511\\
26 0.0309772453559979\\
28 0.0288563095575734\\
30 0.0260866337440907\\
32 0.0222552384986409\\
};
\addlegendentry{DTC};

\addplot [
color=gray,
solid,
line width=1.0pt,
mark size=0.7pt,
mark=triangle*,
mark options={solid,fill=gray}
]
table[row sep=crcr]{
8 0.280716106972888\\
10 0.244025470429048\\
12 0.210127386951632\\
14 0.171707907181547\\
16 0.136044426753204\\
18 0.10974116866221\\
20 0.0903739965283309\\
22 0.0692358732703289\\
24 0.051792951366057\\
26 0.0377710779394409\\
28 0.0334187173468001\\
30 0.0299352608733584\\
32 0.0251366114842413\\
};
\addlegendentry{VAR};

\addplot [
color=blue,
solid,
line width=1.0pt,
mark size=0.7pt,
mark=square*,
mark options={solid,fill=blue}
]
table[row sep=crcr]{
8 0.284631258321139\\
10 0.233852963734052\\
12 0.188992076298888\\
14 0.150875749653669\\
16 0.118673758896662\\
18 0.0888872567792326\\
20 0.0674328761791883\\
22 0.049605635158511\\
24 0.0362813636593844\\
26 0.0256974922090203\\
28 0.0244035007476116\\
30 0.0231676116853157\\
32 0.0205049434218721\\
};
\addlegendentry{FIC};

\addplot [
color=green!50!black,
solid,
line width=1.0pt,
mark size=1.0pt,
mark=x,
mark options={solid,fill=green!50!black}
]
table[row sep=crcr]{
8 0.357316177572842\\
10 0.321658570317463\\
12 0.297732685330197\\
14 0.282556087655918\\
16 0.279097082011509\\
18 0.254558633700725\\
20 0.247712808803435\\
22 0.249196579951357\\
24 0.235113366180629\\
26 0.223554237409333\\
28 0.210273281194116\\
30 0.213789542545961\\
32 0.223270211234156\\
};
\addlegendentry{SSGP};

\end{axis}
\end{tikzpicture}%

    \caption{MSLL for the toy data}
    \label{fig:toy-MSLL}
  \end{subfigure}%
  \hspace*{\fill}
  \\
  \hspace*{\fill}
  \begin{subfigure}[t]{0.48\textwidth}

    \pgfplotsset{every axis legend/.code={\let\addlegendentry\relax}}
    \sffamily \footnotesize %
    \tikzsetnextfilename{tikz-results-precipitation-smsevstime}
%
%
%
%
%
\definecolor{mycolor1}{rgb}{0.75,0.75,0}%
\begin{tikzpicture}

\begin{axis}[%
width=\figurewidth,
height=\figureheight,
unbounded coords=jump,
scale only axis,
xmode=log,
xmin=0.220828440238951,
xmax=700.622612406229,
xminorticks=true,
xlabel={Evaluation time (s)},
ymode=log,
ymin=0.194586757118381,
ymax=0.406193652877505,
yminorticks=true,
ylabel={SMSE},
legend style={draw=black,fill=white,legend cell align=left}
]
\addplot [
color=black,
dashed,
line width=1.0pt
]
table[row sep=crcr]{
357.8478299 0.206894177197309\\
357.8478299 0.206894177197309\\
357.8478299 0.206894177197309\\
357.8478299 0.206894177197309\\
357.8478299 0.206894177197309\\
357.8478299 0.206894177197309\\
357.8478299 0.206894177197309\\
NaN NaN\\
NaN NaN\\
NaN NaN\\
NaN NaN\\
NaN NaN\\
NaN NaN\\
NaN NaN\\
};
\addlegendentry{FULL};

\addplot [
color=red,
solid,
line width=1.0pt,
mark size=1.0pt,
mark=*,
mark options={solid,fill=red}
]
table[row sep=crcr]{
0.43235528 0.349584851180282\\
0.54179052 0.335863490222314\\
0.72859041 0.3252710961436\\
0.88101871 0.315177766369839\\
0.94417498 0.305559147950911\\
1.17285002 0.296974267241042\\
1.3515241 0.290112253510083\\
5.47359724 0.247279054985067\\
11.84352711 0.234549823574962\\
21.87838973 0.221325008654503\\
35.33335339 0.214934583917858\\
57.71844645 0.209857230889822\\
85.76338078 0.207964786590508\\
130.96872456 0.20737307735697\\
};
\addlegendentry{NEW};

\addplot [
color=black,
solid,
line width=1.0pt,
mark size=1.7pt,
mark=diamond*,
mark options={solid,fill=black}
]
table[row sep=crcr]{
6.91186995 0.33952435420954\\
12.46649065 0.324329337531339\\
17.23031424 0.306933566501785\\
25.32769193 0.289728793080509\\
40.04201051 0.274709518689843\\
60.33628402 0.264418021855358\\
88.44382186 0.257311951424124\\
NaN NaN\\
NaN NaN\\
NaN NaN\\
NaN NaN\\
NaN NaN\\
NaN NaN\\
NaN NaN\\
};
\addlegendentry{SOR};

\addplot [
color=mycolor1,
solid,
line width=1.0pt,
mark size=1.0pt,
mark=+,
mark options={solid,fill=mycolor1}
]
table[row sep=crcr]{
6.86779864 0.33952435420954\\
12.42955275 0.324329337531339\\
17.18143851 0.306933566501785\\
25.3286059 0.289728793080509\\
40.05780311 0.274709518689843\\
60.57946772 0.264418021855358\\
88.9228960299999 0.257311951424124\\
NaN NaN\\
NaN NaN\\
NaN NaN\\
NaN NaN\\
NaN NaN\\
NaN NaN\\
NaN NaN\\
};
\addlegendentry{DTC};

\addplot [
color=gray,
solid,
line width=1.0pt,
mark size=0.7pt,
mark=triangle*,
mark options={solid,fill=gray}
]
table[row sep=crcr]{
7.45241003 0.347766060284452\\
12.93160558 0.331502874035263\\
19.52125624 0.322603587736653\\
24.04202378 0.308472005557648\\
30.09577025 0.292619270287672\\
41.73536731 0.280010426274874\\
58.46272212 0.271538428908957\\
NaN NaN\\
NaN NaN\\
NaN NaN\\
NaN NaN\\
NaN NaN\\
NaN NaN\\
NaN NaN\\
};
\addlegendentry{VAR};

\addplot [
color=blue,
solid,
line width=1.0pt,
mark size=0.7pt,
mark=square*,
mark options={solid,fill=blue}
]
table[row sep=crcr]{
8.20434239 0.351806567332954\\
12.73282399 0.329570076191929\\
17.87639423 0.314047857258598\\
25.404052 0.295903923388727\\
39.08805755 0.280038577581372\\
61.08784696 0.268883179396872\\
84.21667801 0.261138321954564\\
NaN NaN\\
NaN NaN\\
NaN NaN\\
NaN NaN\\
NaN NaN\\
NaN NaN\\
NaN NaN\\
};
\addlegendentry{FIC};

\addplot [
color=green!50!black,
solid,
line width=1.0pt,
mark size=1.0pt,
mark=x,
mark options={solid,fill=green!50!black}
]
table[row sep=crcr]{
1.38866948 0.382030595284103\\
1.95842083 0.370614431960156\\
2.75340384 0.359097752355242\\
3.48201779 0.354817967989979\\
4.15312703 0.346403043438074\\
5.09724071 0.345238264543245\\
6.49033088 0.339447906020766\\
NaN NaN\\
NaN NaN\\
NaN NaN\\
NaN NaN\\
NaN NaN\\
NaN NaN\\
NaN NaN\\
};
\addlegendentry{SSGP};

\addplot [
color=black,
dashed,
line width=1.0pt,
forget plot
]
table[row sep=crcr]{
0.220828440238951 0.206894177197309\\
700.622612406229 0.206894177197309\\
};
\addplot [
color=black,
dashed,
line width=1.0pt,
forget plot
]
table[row sep=crcr]{
357.8478299 0.194586757118381\\
357.8478299 0.406193652877505\\
};
\end{axis}
\end{tikzpicture}%

    \caption{SMSE for the precipitation data}
    \label{fig:precipitation-SMSE}
  \end{subfigure}%
  \hspace*{\fill}
  \begin{subfigure}[t]{0.48\textwidth}

    \pgfplotsset{every axis legend/.code={\let\addlegendentry\relax}}
    \sffamily \footnotesize %
    \tikzsetnextfilename{tikz-results-precipitation-msllvstime}
%
%
%
%
%
\definecolor{mycolor1}{rgb}{0.75,0.75,0}%
\begin{tikzpicture}

\begin{axis}[%
width=\figurewidth,
height=\figureheight,
unbounded coords=jump,
scale only axis,
xmode=log,
xmin=0.220828440238951,
xmax=700.622612406229,
xminorticks=true,
xlabel={Evaluation time (s)},
ymode=log,
ymin=0.606942784019318,
ymax=1.0833691690049,
yminorticks=true,
ylabel={MSLL},
legend style={draw=black,fill=white,legend cell align=left}
]
\addplot [
color=black,
dashed,
line width=1.0pt
]
table[row sep=crcr]{
357.8478299 0.636966823154081\\
357.8478299 0.636966823154081\\
357.8478299 0.636966823154081\\
357.8478299 0.636966823154081\\
357.8478299 0.636966823154081\\
357.8478299 0.636966823154081\\
357.8478299 0.636966823154081\\
NaN NaN\\
NaN NaN\\
NaN NaN\\
NaN NaN\\
NaN NaN\\
NaN NaN\\
NaN NaN\\
};
\addlegendentry{FULL};

\addplot [
color=red,
solid,
line width=1.0pt,
mark size=1.0pt,
mark=*,
mark options={solid,fill=red}
]
table[row sep=crcr]{
0.43235528 0.893198789886186\\
0.54179052 0.872989243082058\\
0.72859041 0.857000847790584\\
0.88101871 0.841485839604773\\
0.94417498 0.826215660451244\\
1.17285002 0.812158440892403\\
1.3515241 0.80050589456682\\
5.47359724 0.721656566941169\\
11.84352711 0.696386892855968\\
21.87838973 0.668099721541768\\
35.33335339 0.654347986810418\\
57.71844645 0.643084241856535\\
85.76338078 0.639052673122286\\
130.96872456 0.637896149663509\\
};
\addlegendentry{NEW};

\addplot [
color=black,
solid,
line width=1.0pt,
mark size=1.7pt,
mark=diamond*,
mark options={solid,fill=black}
]
table[row sep=crcr]{
6.91186995 0.879397276145392\\
12.46649065 0.856373359391225\\
17.23031424 0.828841003321573\\
25.32769193 0.800400709108173\\
40.04201051 0.774130811988827\\
60.33628402 0.754725833769716\\
88.44382186 0.740871853632727\\
NaN NaN\\
NaN NaN\\
NaN NaN\\
NaN NaN\\
NaN NaN\\
NaN NaN\\
NaN NaN\\
};
\addlegendentry{SOR};

\addplot [
color=mycolor1,
solid,
line width=1.0pt,
mark size=1.0pt,
mark=+,
mark options={solid,fill=mycolor1}
]
table[row sep=crcr]{
6.86779864 0.88214800427247\\
12.42955275 0.862209525827942\\
17.18143851 0.840874774779558\\
25.3286059 0.814162561972542\\
40.05780311 0.788044817437658\\
60.57946772 0.768404215516429\\
88.9228960299999 0.75383775323829\\
NaN NaN\\
NaN NaN\\
NaN NaN\\
NaN NaN\\
NaN NaN\\
NaN NaN\\
NaN NaN\\
};
\addlegendentry{DTC};

\addplot [
color=gray,
solid,
line width=1.0pt,
mark size=0.7pt,
mark=triangle*,
mark options={solid,fill=gray}
]
table[row sep=crcr]{
7.45241003 0.890591439351248\\
12.93160558 0.866722172398576\\
19.52125624 0.853520339845231\\
24.04202378 0.83201820901998\\
30.09577025 0.806740850969092\\
41.73536731 0.785294857435696\\
58.46272212 0.770643483302571\\
NaN NaN\\
NaN NaN\\
NaN NaN\\
NaN NaN\\
NaN NaN\\
NaN NaN\\
NaN NaN\\
};
\addlegendentry{VAR};

\addplot [
color=blue,
solid,
line width=1.0pt,
mark size=0.7pt,
mark=square*,
mark options={solid,fill=blue}
]
table[row sep=crcr]{
8.20434239 0.867248127630572\\
12.73282399 0.852422962137021\\
17.87639423 0.832687685304191\\
25.404052 0.807535914503198\\
39.08805755 0.783091150494705\\
61.08784696 0.763652520342296\\
84.21667801 0.74965165046696\\
NaN NaN\\
NaN NaN\\
NaN NaN\\
NaN NaN\\
NaN NaN\\
NaN NaN\\
NaN NaN\\
};
\addlegendentry{FIC};

\addplot [
color=green!50!black,
solid,
line width=1.0pt,
mark size=1.0pt,
mark=x,
mark options={solid,fill=green!50!black}
]
table[row sep=crcr]{
1.38866948 1.03230352924907\\
1.95842083 1.01659679935496\\
2.75340384 1.00070865679471\\
3.48201779 0.993210957451386\\
4.15312703 0.982092720360688\\
5.09724071 0.978344915622941\\
6.49033088 0.96921070458998\\
NaN NaN\\
NaN NaN\\
NaN NaN\\
NaN NaN\\
NaN NaN\\
NaN NaN\\
NaN NaN\\
};
\addlegendentry{SSGP};

\addplot [
color=black,
dashed,
line width=1.0pt,
forget plot
]
table[row sep=crcr]{
0.220828440238951 0.636966823154081\\
700.622612406229 0.636966823154081\\
};
\addplot [
color=black,
dashed,
line width=1.0pt,
forget plot
]
table[row sep=crcr]{
357.8478299 0.606942784019318\\
357.8478299 1.0833691690049\\
};
\end{axis}
\end{tikzpicture}%

    \caption{MSLL for the precipitation data}
    \label{fig:precipitation-MSLL}
  \end{subfigure}%
  \hspace*{\fill}

  \caption{Standardized mean squared error (SMSE) and mean standardized 
           log loss (MSLL) results for the toy data ($d=1$, $n=256$) from 
           Figure~\ref{fig:toy_example} and the precipitation data 
           ($d=2$, $n=5776$) evaluated by 10-fold cross-validation and 
           averaged over ten repetitions. The evaluation time includes 
           hyperparameter learning.}
  \label{fig:results}
\end{figure*}

We compare our solution to SOR, DTC, VAR, and FIC using the implementations provided in the GPstuff software package \mbox{\citep[version~4.3.1, see][]{Vanhatalo+Riihimaki+Hartikainen+Jylanki+Vehtari:2013}} for Mathworks Matlab. The sparse spectrum SSGP method \citep{Lazaro-Gredilla+Quinonero-Candela+Rasmussen+Figueiras-Vidal:2010} was implemented into the GPstuff toolbox for the comparisons.\footnote{The implementation is based on the code available from Miguel L{\'a}zaro-Gredilla:  \url{http://www.tsc.uc3m.es/~miguel/downloads.php}.} The reference implementation was modified such that also non-ARD covariances could be accounted for.

The $m$ inducing inputs for SOR, DTC, VAR, and FIC were chosen at random as a subset from the training data and kept fixed between the methods. For low-dimensional inputs, this tends to lead to good results and avoid over-fitting to the training data, while optimizing the input locations alongside hyperparameters becomes the preferred approach in high input dimensions \citep{Quinonero-Candela+Rasmussen:2005}. The results are averaged over ten repetitions in order to present the average performance of the methods. In Sections~\ref{sec:toy-example} and \ref{sec:precipitation-data}, we used a Cartesian domain with Dirichlet boundary conditions for the new reduced-rank method. To avoid boundary effects,  the domain was extended by 10\% outside the inputs in each direction.

In the comparisons we followed the guidelines given by \citet{Chalupka+Williams+Murray:2013} for making comparisons between the actual performance of different methods. For hyperparameter optimization we used the \texttt{fminunc} routine in Matlab with a Quasi-Newton optimizer. We also tested several other algorithms, but the results were not sensitive to the choice of optimizer. The optimizer was run with a termination tolerance of $10^{-5}$ on the target function value and on the optimizer inputs.  The number of required target function evaluations stayed fairly constant for all the comparisons, making the comparisons for the hyperparameter learning bespoke.

Figure~\ref{fig:toy_example} shows a simulated example, where 256 data points are drawn from a Gaussian process prior with a squared exponential covariance function. We use the same parametrization as \citet{Rasmussen+Williams:2006} and denote the signal variance $\sigma^2$, length-scale $\ell$, and noise variance $\sigma_\mathrm{n}^2$. Figure~\ref{fig:toy_example-lik} shows the negative marginal log likelihood curves both for the full GP and the approximation with $m=32$ basis functions. The likelihood curve approximations are almost exact and only differs from the full GP likelihood for small length-scales (roughly for values smaller than $2L/m$). Figure~\ref{fig:toy_example-gp} shows the approximate GP solution. The mean estimate follows the exact GP mean, and the shaded region showing the 95\% confidence area differs from the exact solution (dashed) only near the boundaries.

\begin{figure*}[!t]
  \centering
  \begin{subfigure}[t]{0.32\textwidth}
    \centering
    \includegraphics[trim=0 -0.336in 0 0]{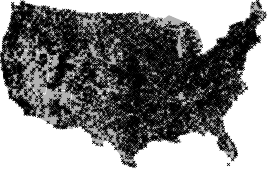}
    \caption{Observation locations}
    \label{fig:precipitation-locations}
  \end{subfigure}
  \hspace*{\fill}
  \begin{subfigure}[t]{0.32\textwidth}
    \centering
    \includegraphics[]{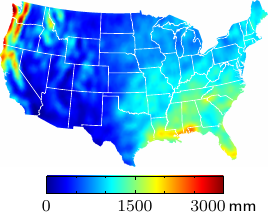}
    \caption{The full GP}
    \label{fig:precipitation-full}
  \end{subfigure}
  \hspace*{\fill}
  \begin{subfigure}[t]{0.32\textwidth}
    \centering
    \includegraphics[trim=0 -0.348in 0 0]{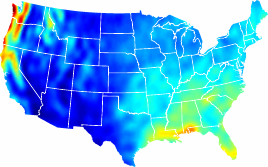}
    \caption{The reduced-rank method}
    \label{fig:precipitation-reduced}
  \end{subfigure}
  \caption{Interpolation of the yearly precipitation levels using reduced-rank GP regression. Subfigure~\ref{fig:precipitation-locations} shows the $n=5776$ weather station locations. Subfigures~\ref{fig:precipitation-full} and \ref{fig:precipitation-reduced} show the results for the full GP model and the new reduced-rank GP method.}
  \label{fig:precipitation-map}
\end{figure*}

Figures~\ref{fig:toy-SMSE} and \ref{fig:toy-MSLL} show the SMSE and MSLL values for $m=8,10,\ldots,32$ inducing inputs and basis functions for the toy dataset from Figure~\ref{fig:toy_example}. The convergence of the proposed reduced rank method is fast and as soon as the number of eigenfunctions is large enough ($m=20$) to account for the short length-scales, the approximation converges to the exact full GP solution (shown by the dashed line).

In this case the SOR method that uses the Nystr{\"o}m approximation to directly approximate the spectrum of the full GP (see \mbox{Section~\ref{sec:relationship-to-other-methods}}) seems to give good results. However, as the resulting approximation in SOR corresponds to a singular Gaussian distribution, the predictive variance is underestimated. This can be seen in \mbox{Figure~\ref{fig:toy-MSLL}}, where SOR seems to give better results than the full GP. These results are however due to the smaller predictive variance on the test set. DTC tries to fix this shortcoming of SOR---they are identical in other respects except predictive variance evaluation---and while SOR and DTC give identical results in terms of SMSE, they differ in MSLL. We also note that additional trace term in the marginal likelihood in VAR makes the likelihood surface flat, which explains the differences in the results in comparison to DTC.

The sparse spectrum SSGP method did not perform well on average. Still, it can be seen that it converges towards the performance of the full GP. The dependence on the number of spectral points differs from the rest of the methods, and a rank of $m=32$ is not enough to meet the other methods. However, in terms of best case performance over the ten repetitions with different inducing inputs and spectral points, both FIC and SSGP outperformed SOR, DTC, and VAR. Because of its `dense spectrum' approach, the proposed reduced-rank method is not sensitive to the choice of spectral points, and thus the performance remained the same between repetitions. In terms of variance over the 10-fold cross-validation folds, the methods in order of growing variance in the figure legend (the variance approximately doubling between FULL and SSGP).

\subsection{Precipitation Data}
\label{sec:precipitation-data}
As a real-data example, we consider a precipitation data set that contain US annual precipitation summaries for year 1995 \citep[$d=2$ and $n=5776$, available online, see][]{Vanhatalo+Riihimaki+Hartikainen+Jylanki+Vehtari:2013}. The observation locations are shown on a map in Figure~\ref{fig:precipitation-locations}. 

We limit the number of inducing inputs and spectral points to $m=128,192,\ldots,512$. For the our Hilbert-GP method we additionally consider ranks $m = 1024,1536, \ldots, 4096$, and show that this causes a computational burden of the same order as the conventional sparse GP methods with smaller $m$s. To avoid boundary effects,  the domain was extended by 10\% outside the inputs in each direction.

In order to demonstrate the computational benefits of the proposed model, we also present the running time of the GP inference (including hyperparameter optimization). All methods were implemented under a similar framework in the GPstuff package, and they all employ similar reformulations for numerical stability. The key difference in the evaluation times comes from hyperparameter optimization, where SOR, DTC, VAR, FIC, and SSGP scale as  $\O(nm^2)$ for each evaluation of the marginal likelihood. The proposed reduced-rank method scales as $\O(m^3)$ for each evaluation (after an initial cost of $\O(nm^2)$).

Figures~\ref{fig:precipitation-SMSE} and \ref{fig:precipitation-MSLL} show the SMSE and MSLL results for this data against evaluation time. On this scale we note that the evaluation time and accuracy, both in terms of SMSE and MSLL, are alike for SOR, DTC, VAR, and FIC. SSGP is faster to evaluate in comparison with the Nystr{\"o}m family of methods, which comes from the simpler structure of the approximation. Still, the number of required spectral points to meet a certain average error level is larger for SSGP. 

The results for the proposed reduced-rank method (Hilbert-GP) show that with two input dimensions, the required number of basis functions is larger. For the first seven points, we notice that even though the evaluation is two orders of magnitude faster, the method performs only slightly worse in comparison to conventional sparse methods. By considering higher ranks (the next seven points), our method converges to the performance of the full GP (both in SMSE and MSLL), while retaining a computational time comparable to the conventional methods. This type of spatial medium-size GP regression problems can thus be solved in seconds.

Figures~\ref{fig:precipitation-full} and \ref{fig:precipitation-reduced} show interpolation of the precipitation levels using a full GP model and the reduced-rank method ($m = 1728$), respectively. The results are practically identical, as is easy to confirm from the color surfaces. Obtaining the reduced-rank result (including initialization and hyperparameter learning) took slightly less than 30~seconds on a laptop computer (MacBook Air, Late~2010 model, 2.13~GHz, 4~GB RAM), while the full GP inference took approximately 18~minutes.

\begin{figure*}[!t]
  \centering
  \hspace*{\fill} %
  \begin{subfigure}[t]{0.45\textwidth}
    \centering
    \includegraphics
     {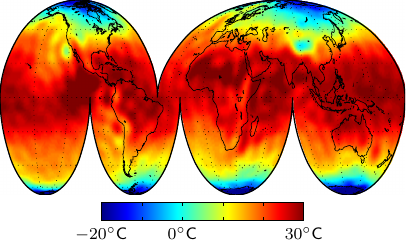}
    \caption{The mean temperature}
    \label{fig:temperature-mean}
  \end{subfigure}%
  \hspace*{\fill} %
  \begin{subfigure}[t]{0.45\textwidth}
    \centering
    \includegraphics
     {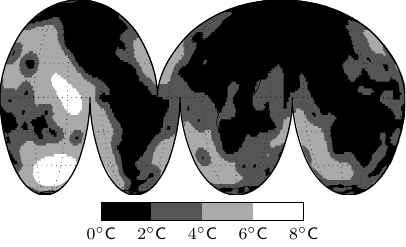}
    \caption{Standard deviation contours}
    \label{fig:temperature-std}
  \end{subfigure}%
  \hspace*{\fill} %
  \caption{Modeling of the yearly mean temperature on the spherical 
           surface of the Earth ($n = 11\,028$). Figure~\ref{fig:temperature-std}
           shows the standard deviation contours which match well with the continents.}
  \label{fig:temperature}
\end{figure*}

\subsection{Temperature Data on the Surface of the Globe}
\label{sec:temperature-data}
We also demonstrate the use of the method in non-Cartesian coordinates. We consider modeling of the spatial mean temperature over a number of $n = 11\,028$ locations around the globe.\footnote{The data are available for download from US National Climatic Data Center: \url{http://www7.ncdc.noaa.gov/CDO/cdoselect.cmd} (accessed January~3, 2014).}

As earlier demonstrated in Figure~\ref{fig:spheres}, we use the Laplace operator in spherical coordinates as defined in \eqref{eq:spherical-laplacian}. The eigenfunctions for the angular part are the Laplace's spherical harmonics. The evaluation of the approximation does not depend on the coordinate system, and thus all the equations presented in the earlier sections remain valid. We use the squared exponential covariance function and $m = 1089$ basis functions.

Figure~\ref{fig:temperature} visualizes the modeling outcome. The results are visualized using an interrupted projection (an adaption of the Goode homolosine projection) in order to preserve the length-scale structure across the map. The uncertainty is visualized in Figure~\ref{fig:temperature-std}, which corresponds to the $n = 11{,}028$  observation locations that are mostly spread over the continents and western countries (the white areas in Figure~\ref{fig:temperature-std} contain no observations). Obtaining the reduced-rank result (including initialization and hyperparameter learning) took approximately 50~seconds on a laptop computer (MacBook Air, Late~2010 model, 2.13~GHz, 4~GB RAM), which scales with $n$ in comparison to the evaluation time in the previous section.

\begin{table*}
  \centering
  \caption{\label{tbl:airline}
  Predictive mean squared errors (MSEs) and negative log predictive densities (NLPDs) with one standard deviation on the airline arrival delays experiment (input dimensionality $d=8$) for a number of data points ranging up to almost 6~million. The Hilbert-GP method is on par with the VFF method albeit being clearly faster due to the diagonalizable structure (solving the regression problem including hyperparameter optimization in 41 seconds on a CPU-only laptop computer).}  
  \resizebox{\textwidth}{!}{
  \begin{tabular}{lcccccccc}
\toprule
$n$ & \multicolumn{2}{c}{10,000} & \multicolumn{2}{c}{100,000} & \multicolumn{2}{c}{1,000,000} & \multicolumn{2}{c}{5,929,413} \\
& MSE & {NLPD} & MSE & {NLPD} & MSE & {NLPD} & MSE & {NLPD} \\
\midrule
Hilbert-GP & $0.97 \pm 0.14$ & $1.404 \pm 0.071$ & $0.80 \pm 0.06$ & $1.311 \pm 0.038$ & $0.83 \pm 0.02$ & $1.329 \pm 0.011$ & $0.827 \pm 0.005$ & $1.324 \pm 0.003$ \\
VFF & 0.89 $\pm$ 0.15 & {1.362 $\pm$ 0.091} & 0.82 $\pm$ 0.05 & {1.319 $\pm$ 0.030} & 0.83 $\pm$ 0.01 & {1.326 $\pm$ 0.008} & 0.827 $\pm$ 0.004 & {1.324 $\pm$ 0.003} \\
SVIGP & {0.89 $\pm$ 0.16} & {1.354 $\pm$ 0.096} & {0.79 $\pm$ 0.05} & {1.299 $\pm$ 0.033} & {0.79 $\pm$ 0.01} & {1.301 $\pm$ 0.009} & {0.791 $\pm$ 0.005} & {1.300 $\pm$ 0.003} \\
Full-RBF & 0.89 $\pm$ 0.16 & {1.349 $\pm$ 0.098} & N/A & {N/A} & N/A & {N/A} & N/A & {N/A} \\
Full-additive & 0.89 $\pm$ 0.16 & {1.362 $\pm$ 0.096} & N/A & {N/A} & N/A & {N/A} & N/A & {N/A} \\
\bottomrule
  \end{tabular}%
  }
\end{table*}

\subsection{Additive Modelling of Airline Delays}
\label{sec:airline-delays}
In order to fully use the computational benefits and also underline a way of applying the method to high-dimensional inputs, we consider a large dataset for predicting airline delays. The US flight delay prediction example \citep[originally considered by][]{Hensman+Fusi+Lawrence:2013} is a standard test data set in Gaussian process regression. This is due to the clearly non-stationary behavior and its massive size, with nearly 6~million records.

We aim to replicate and extend to the results previously presented in the work by \citet{Hensman+Durrande+Solin:2018} for the Variational Fourier Features (VFF) method. This example has also been used by \citet{Deisenroth+Ng:2015}, where it was solved using distributed Gaussian processes, and by \citet{Samo+Roberts:2016} who use this example for demonstrating the computational efficiency of string Gaussian processes. \citet{Adam+Hensman+Sahani:2016} used this data set as an example where the model can be formed by the addition of multiple underlying components.

The data consists of flight arrival and departure times for every commercial flight in the USA for the year 2008. We use the standard eight covariates $\vect{x}$  \citep[see][]{Hensman+Fusi+Lawrence:2013} which are the age of the aircraft (number of years since deployment), route distance, airtime, departure time, arrival time, day of the week, day of the month, and month. The target is to predict the delay of the aircraft at landing (in minutes), $y$.

This regression task is set up similarly as in \citet{Hensman+Durrande+Solin:2018} and \citet{Adam+Hensman+Sahani:2016}, as a Gaussian process regression model with a prior covariance structure given as a sum of ovariance functions for each input dimension and assuming the observations are corrupted by independent Gaussian noise, $\varepsilon_i \sim \mathcal{N}(0,\sigma_\mathrm{n}^2)$. The model is
\begin{equation}
\begin{split}
    f(\vect{x}) &\sim \mathcal{GP}\left(0, \sum_{d=1}^8 k_\mathrm{se}(x_d, x_d')\right),\\
    y_i &= f(\vect{x}_i) + \varepsilon_i,
\end{split}
\end{equation}
for $i=1,2,\ldots,m$. We used $m=40$ basis functions per input dimension. The boundary is set to a distance of two times the range of the data for each dimension.

We consider several subset sizes of the data, each selected uniformly at random: $n = 10{,}000$, 100,000, 1,000,000, and 5,929,413 (all data). In each case, two thirds of the data is used for training and one third for testing.  For each subset size the training is repeated ten times. The random splits are exactly the same as in \citet{Hensman+Durrande+Solin:2018}.

Table~\ref{tbl:airline} shows the (normalized) predictive mean squared errors (MSEs) and the negative log predictive densities (NLPDs) with one standard deviation on the airline arrival delays experiment. The table shows that the Hilbert-GP method is directly on par with the Variational Fourier Features (VFF) method. For the smaller subsets some variability in the results is visible, even though the MSEs and NLPDs are within one standard deviation of one another for VFF and Hilbert-GP. For the data sets in the millions, VFF and Hilbert-GP perform practically equally well.
Further analysis and interpretation of the data and model can be found in \citet{Hensman+Durrande+Solin:2018}. We have omitted reporting results for the String GP method \citet{Samo+Roberts:2016}, the Bayesian committee machine \citep[BCM,][]{Tresp:2000}, and the robust Bayesian committee machine \citep[rBCM,][]{Deisenroth+Ng:2015}. Each of these performed worse than any of the included methods, and the resulting numbers can be found listed in \citet{Hensman+Durrande+Solin:2018} and \citet{Samo+Roberts:2016}.

Running the Hilbert-GP method in this experiment (including hyperparameter training and prediction) with all $5.93$~million data took $41 \pm 2$~seconds ($120\pm 7$~s CPU time) on a MacBook Pro laptop (with all calculation done on the CPU). The is clarly faster than the VFF method with $265 \pm 6$~seconds ($626\pm11$~s CPU time), where our computational gain comes from the fully diagonal structure of the covariance. For comparison, the SVIGP method \citep{Hensman+Fusi+Lawrence:2013} required $5.1\pm0.1$~hours of computing ($27.0\pm0.8$~h CPU time) on a cluster. \citet{Samo+Roberts:2016} report that running the String GP took 91.0~hours total CPU time (or 15~h of wall-clock time on an 8-core machine). \citet{Izmailov+Novikov+Kropotov:2018} also report results for the airline dataset, where  one pass over the data taking 5200~seconds, when running on a Nvidia Tesla K80 GPU and not  assuming additive structure.

\begin{figure*}[!t]
  \centering
  \hspace*{\fill} %
  \begin{subfigure}[t]{0.45\textwidth}
    \centering
    \includegraphics[width=\textwidth]{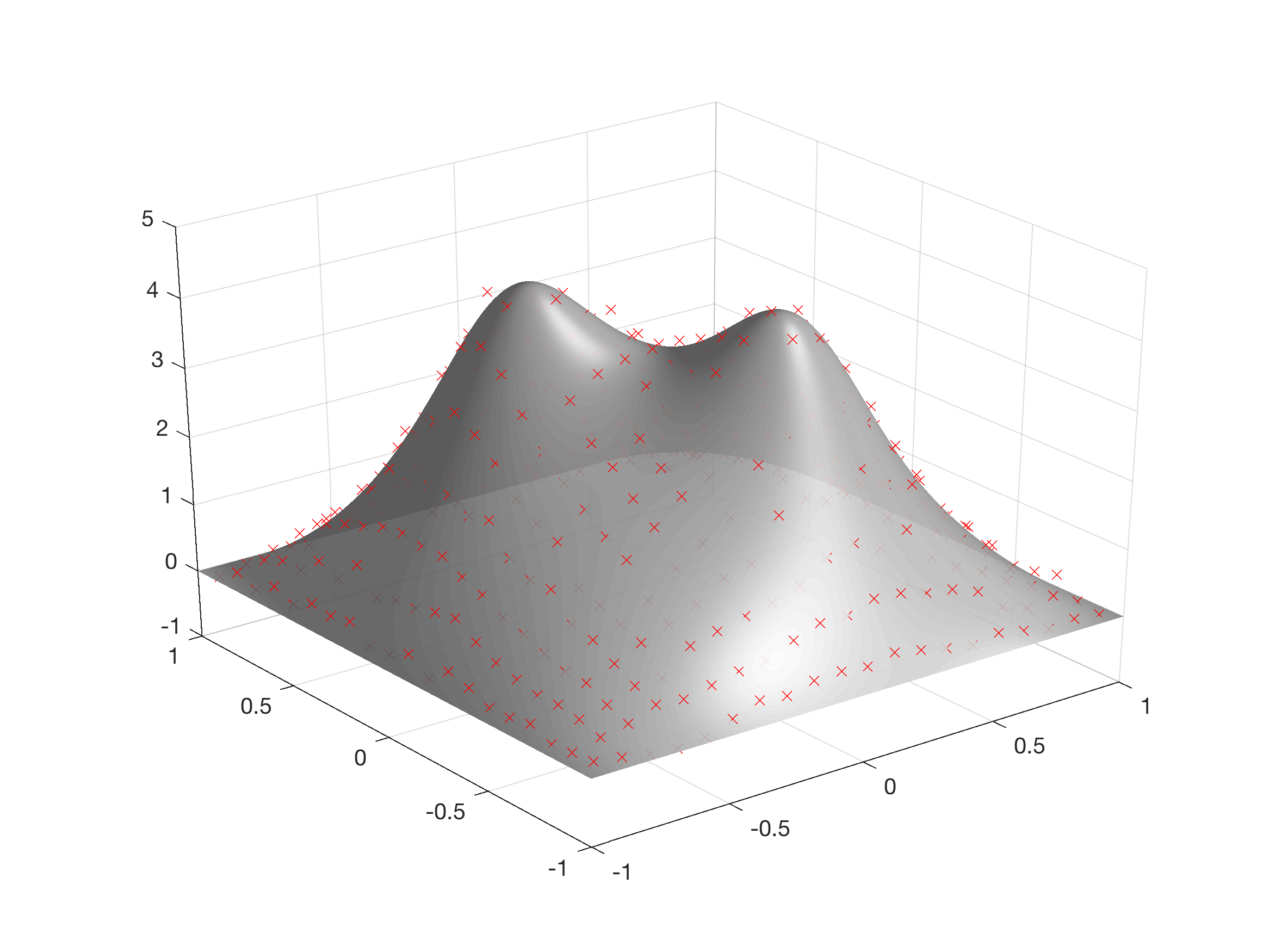}
    \caption{The true solution $g(x_1,x_2)$ and the measurements.}
    \label{fig:poisson-orig-f}
  \end{subfigure}%
  \hspace*{\fill} %
  \begin{subfigure}[t]{0.45\textwidth}
    \centering
    \includegraphics[width=\textwidth]{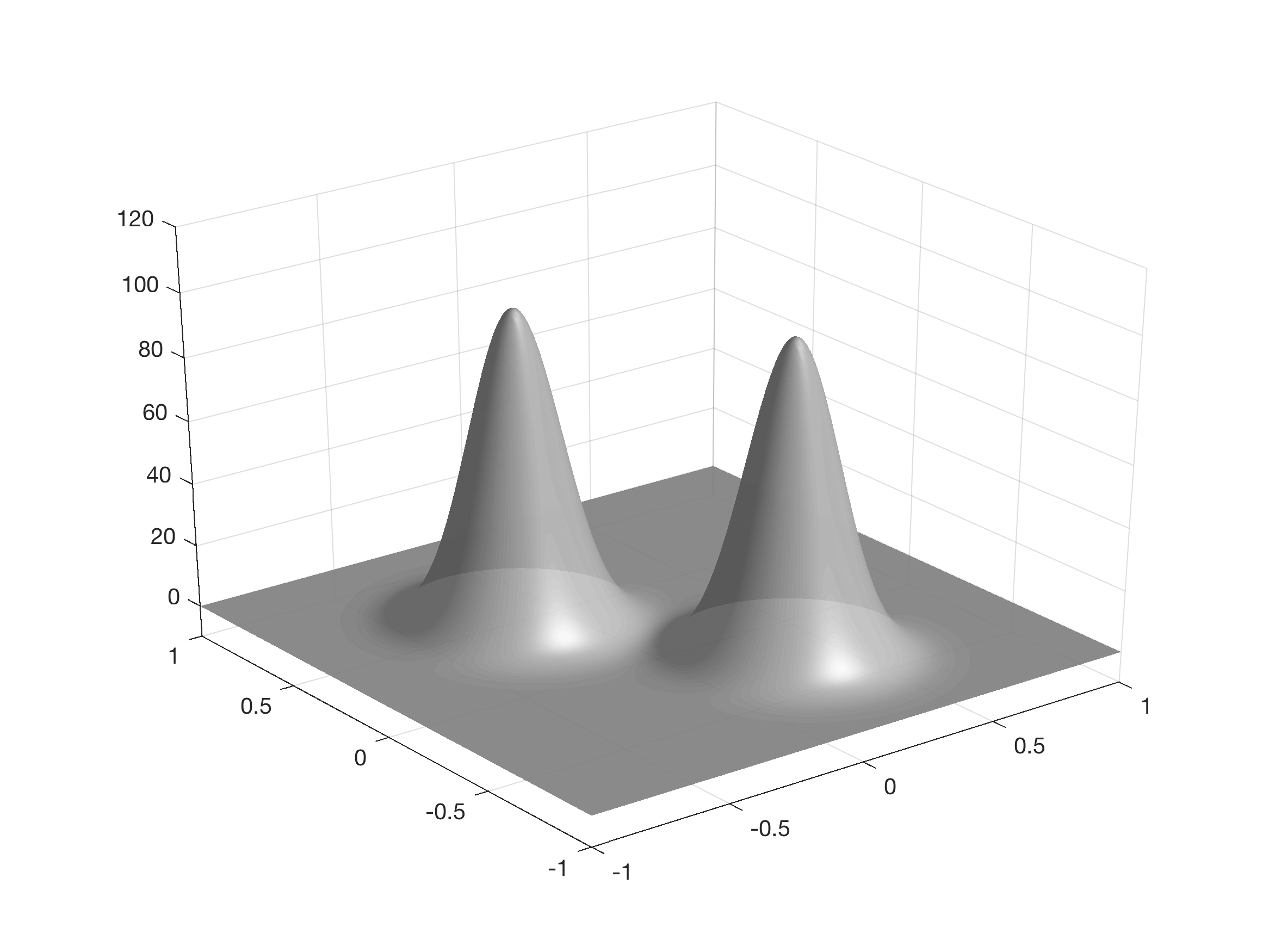}
    \caption{The true input $f(x_1,x_2)$.}
    \label{fig:poisson-orig-u}
  \end{subfigure}%
  \hspace*{\fill} \\
  \hspace*{\fill} 
  \begin{subfigure}[t]{0.45\textwidth}
    \centering
    \includegraphics[width=\textwidth]{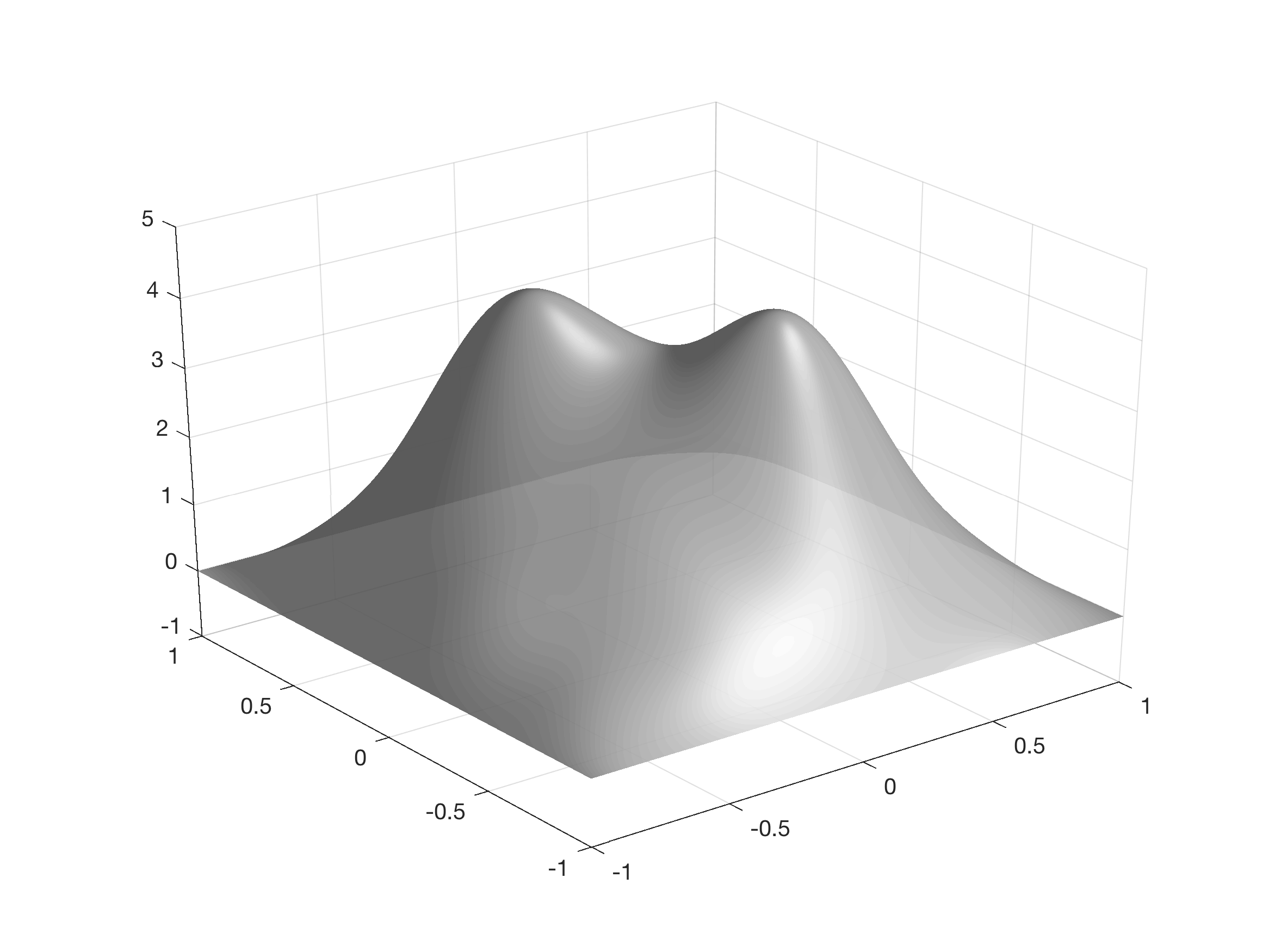}
    \caption{The estimate of solution $g(x_1,x_2)$.}
    \label{fig:poisson-est-f}
  \end{subfigure}%
  \hspace*{\fill} 
  \begin{subfigure}[t]{0.45\textwidth}
    \centering
    \includegraphics[width=1.1\textwidth]{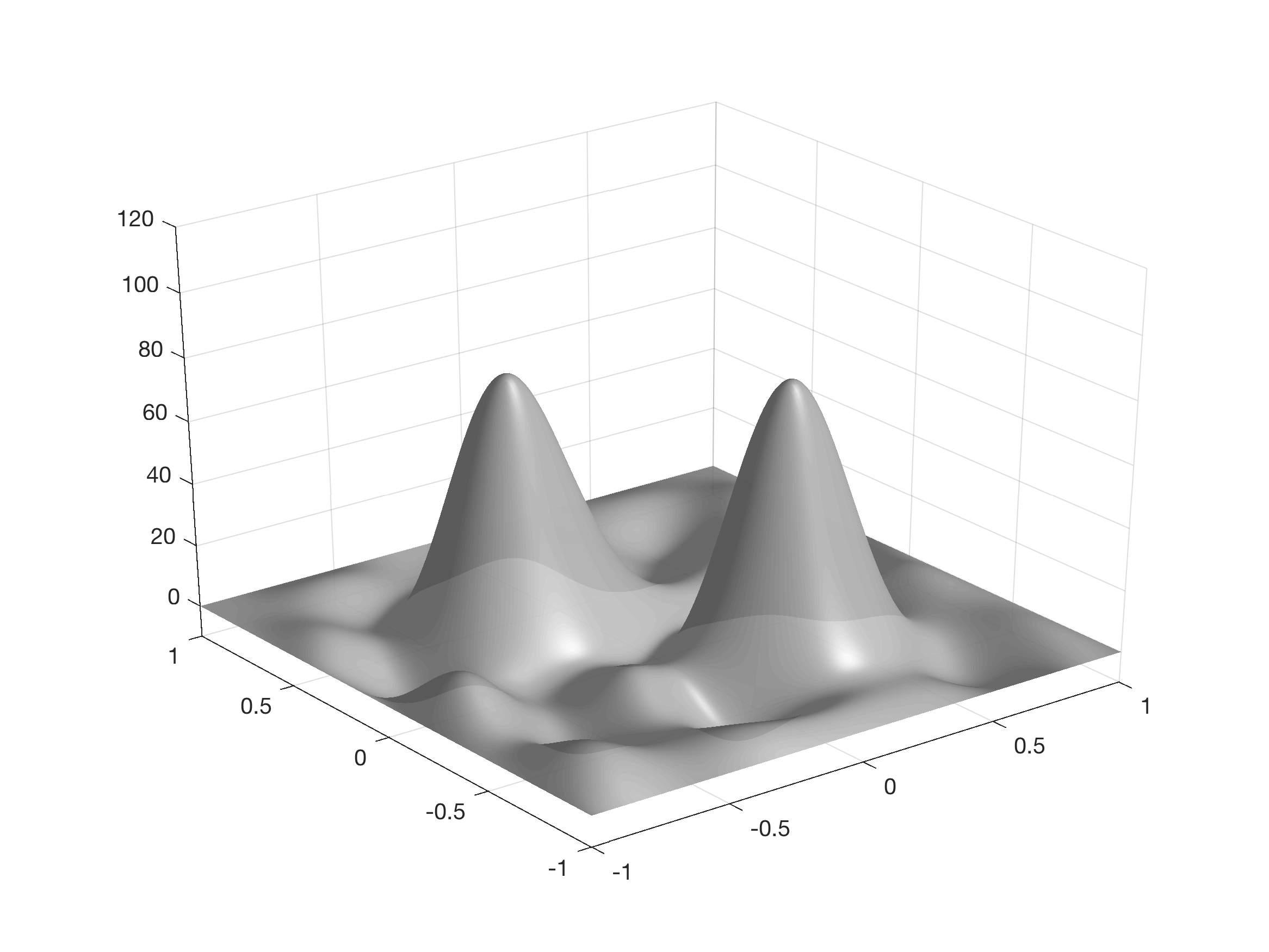}
    \caption{The estimate of input $f(x_1,x_2)$.}
    \label{fig:poisson-est-u}
  \end{subfigure}%
  \hspace*{\fill} 
  \caption{Gaussian process inference on the Poisson equation.}
  \label{fig:poisson}
\end{figure*}

\subsection{Gaussian Process Driven Poisson Equation}
As discussed in Section~\ref{sec:invlfm}, our framework also directly extends to inverse problems and latent force models. As this final experiment, we demonstrate the use of the approximation in the latent force model (LFM)
\begin{equation}
\begin{split}
  -\nabla^2 g(\vect{x}) &= f(\vect{x}), \\
  y_i &= g(\mathbf{x}_i) + \varepsilon_i,
\end{split}
\end{equation}
where $\vect{x} \in \mathbb{R}^2$ and $f(\vect{x}) \sim \GP(0, k(\vect{x}, \vect{x}'))$ is the input with a squared exponential covariance function prior. This problem can also be interpreted as a inverse problem where the measurement operator is the Green's operator $\mathcal{H} = (-\nabla^2)^{-1}$:
\begin{equation}
\begin{split}
  y_i &= (\mathcal{H} f)(\mathbf{x}_i) + \varepsilon_i.
\end{split}
\end{equation}
If we assume that the boundary conditions of the problem are the same as we used for forming the basis functions in \eqref{eq:eigenf_eqs}, then if we put $g(\vect{x}) \approx \sum_{j=1}^m g_j \, \phi_j(\vect{x})$, we get
\begin{equation}
\begin{split}
  -\nabla^2 g(\vect{x}) &\approx
  -\sum_{j=1}^m g_j \, \nabla^2 \phi_j(\vect{x}) 
  = \sum_{j=1}^m g_j \, \lambda_j \, \phi_j(\vect{x})
\end{split}
\end{equation}
and thus by further putting $f(\vect{x}) \approx \sum_{j=1}^m f_j \, \phi_j(\vect{x})$, the approximation to the equation $-\nabla^2 g(\vect{x}) = f(\vect{x})$ becomes
\begin{equation}
\begin{split}
  \sum_{j=1}^m g_j \, \lambda_j \, \phi_j(\vect{x})
  = \sum_{j=1}^m f_j \, \phi_j(\vect{x})
\end{split}
\end{equation}
which allows us to solve $f_j = g_j / \lambda_j$. This implies that we approximately have $(-\nabla^2)^{-1} \phi_j = \phi_j / \lambda_j$ which reduces Equations \eqref{eq:ip_approximation} to %
\begin{equation} 
\begin{split} 
  (\mathcal{H} \, \mathcal{H}' \, k)(\vect{x},\vect{x}')
   & \approx \sum_j \lambda_j^{-2} \, S(\sqrt{\lambda_j}) \,
   \phi_j(\vect{x}) \, \phi_j(\vect{x}'), \\
   (\mathcal{H}' k(\vect{x}_*,\cdot))(\vect{x}')
   & \approx \sum_j \lambda_j^{-1} \, S(\sqrt{\lambda_j}) \,
   \phi_j(\vect{x}_*) \, \phi_j(\vect{x}'),
\end{split} 
\end{equation} 
after which we can proceed with \eqref{eq:invprob-solution}. Alternatively we can directly use \eqref{eq:ip-solution-approx} with $\tilde{\vectb{\Phi}}_{ij} = \phi_j(\vect{x}_i) / \lambda_j$.

Figure~\ref{fig:poisson} shows the result of applying the proposed method to this model with the input function shown in Figure~\ref{fig:poisson-orig-u}. The true solution and the simulated measurements (with standard deviation of $1/10$) are shown in Figure~\ref{fig:poisson-orig-f}. The scale $\sigma^2$ and length scale $\ell$ of the SE covariance function were estimated by maximum likelihood method and the number of basis functions used for solving the GP regression problem was 100 (for simulation we used 255 basis functions). The estimates of the input and the solution function are shown in Figures~\ref{fig:poisson-orig-u} and \ref{fig:poisson-orig-f}, respectively. As can be seen in the figures, the estimate of the solution is very good, as can be expected from the fact that we obtain direct (although noisy) measurements from it. The estimate of the input is less accurate, but still approximates the true input well.

\section{Conclusion and Discussion}
\label{sec:conclusion-and-discussion}
In this paper we have proposed a novel approximation scheme for forming approximate eigendecompositions of covariance functions in terms of the Laplace operator eigenbasis and the spectral density of the covariance function. The eigenfunction decomposition of the Laplacian can easily be formed in various domains, and the eigenfunctions are independent of the choice of hyperparameters of the covariance. 

An advantage of the method is that it has the ability to approximate the eigendecomposition using only the eigendecomposition of the Laplacian and the spectral density of the covariance function, both of which are closed-from expressions. This together with having the eigenvectors in $\vectb{\Phi}$ mutually orthogonal and independent of the hyperparameters, is the key to efficiency. This allows an implementation with a computational cost of $\O(nm^2)$ (initial) and $\O(m^3)$ (marginal likelihood evaluation), with negligible memory requirements.

Of the infinite number of possible basis functions only an extremely small subset are of any relevance to the GP being approximated. In GP regression the model functions are conditioned on a covariance function (kernel), which imposes desired properties on the solutions. We choose the basis functions such that they are as close as possible (w.r.t.\ the Frobenius norm) to those of the particular covariance function. Our method gives the exact eigendecomposition of a GP that has been constrained to be zero at the boundary of the domain.

The method allows for theoretical analysis of the error induced by the truncation of the series and the boundary effects. This is something new in this context and extremely important, for example, in medical imaging applications. The approximative eigendecomposition also opens a range of interesting possibilities for further analysis. In \emph{learning curve} estimation, the eigenvalues of the Gaussian process can now be directly approximated. For example, we can approximate the Opper--Vivarelli bound \citep{Opper+Vivarelli:1999} as 
\begin{equation}
  \epsilon_\mathrm{OV}(n)
    \approx \sigma_\mathrm{n}^2 \,
       \sum_j {S(\sqrt{\lambda_j}) \over
               \sigma_\mathrm{n}^2 + n \, S(\sqrt{\lambda_j})}.
\end{equation}
Sollich's eigenvalue based bounds \citep{Sollich+Halees:2002} can be approximated and analyzed in an analogous way.

However, some of these abilities come with a cost. As demonstrated throughout the paper, restraining the domain to boundary conditions introduces edge effects. These are, however, known and can be accounted for. Extrapolating with a stationary covariance function outside the training inputs only causes the predictions to revert to the prior mean and variance. Therefore we consider the boundary effects a minor problem for practical use.

Although at first sight the method appears to have a bad (exponential) scaling with respect to the input dimensionality, as shown by the analysis in Section~\ref{sec:scaling}, this is not true. By increasing the differentiability order of the covariance function appropriately we can keep the convergence rate at the level ${\sim}1/m^a$, for a given constant $a > 0$ and with total of $m$ terms in the series, regardless of the input dimensionality. Furthermore, Theorem~\ref{the:se_rate} shows that for squared exponential covariance function the convergence rate is always better than ${\sim}1/m$, independently of the input dimensionality.

Further resources related to the proposed method and implementation details in form of code are available at \url{https://github.com/AaltoML/hilbert-gp}.

\subsubsection*{Acknowledgments}
The authors would like to thank James Hensman and Manon Kok for feedback on an early version of this paper, Mauricio {\'A}lvarez for help in the latent force model application as well as the referees for providing valuable ideas for improving the article. This research was supported by the Academy of Finland grants 308640 and 313708. We acknowledge the computational resources provided by the Aalto Science-IT project.

\appendix

\section{Proofs of Convergence Theorems} 
\subsection{Auxiliary Lemmas}
In this section we present a few lemmas that will be needed in the proofs in the next sections. The lemmas are quite classical results on the convergence of Riemannian sums, but as it is hard to find exactly the same results in other literature, for completeness we prove the lemmas here.

\begin{mylemma} \label{lem:riemann1}
Let $\Delta > 0$ and $\alpha \in [0,1)$ be given constants, $m = 0,1,2,\ldots$ some nonnegative integer, and assume that the $f(\omega)$ is a bounded integrable function defined on $\omega \ge m \, \Delta$ with bounded derivative on $\omega > m \, \Delta$ such that $\int_{m \, \Delta}^{\infty} |f'(\omega)| \, d\omega = C^{(m)} < \infty$. Then we have
\begin{equation}
  \left| \int_{m \, \Delta}^\infty f(\omega) \, d\omega
  - \sum_{j=m+1}^\infty f(j \, \Delta - \alpha \, \Delta) \, \Delta \right| \le C^{(m)} \, \Delta.
\end{equation}
Furthermore, provided that $\int_{0}^{\infty} |f'(\omega)| \, d\omega = C^{(0)} < \infty$, this bound can be made independent of $m$:
\begin{equation}
  \left| \int_{m \, \Delta}^\infty f(\omega) \, d\omega
  - \sum_{j=m+1}^\infty f(j \, \Delta - \alpha \, \Delta) \, \Delta \right| \le C^{(0)} \, \Delta.
\end{equation}
\end{mylemma}
\begin{proof}
We can write
\begin{equation}
\begin{split}
  \int_{m \, \Delta}^\infty f(\omega) \, d\omega = \sum_{j=m+1}^\infty \int_{(j-1) \, \Delta}^{j \, \Delta} f(\omega) \, d\omega.
\end{split}
\end{equation}
By the fundamental theorem of calculus we get
\begin{equation}
\begin{split}
  f(\omega) = f(j \ \Delta - \alpha \, \Delta)
  + \int_{j \, \Delta - \alpha \, \Delta}^\omega f'(\omega) \, d\omega,
\end{split}
\end{equation}
which gives for $\omega \in ((j-1) \, \Delta, j \, \Delta]$
\begin{equation}
\begin{split}
  \left| f(\omega) - f(j \, \Delta - \alpha \, \Delta) \right|
 & \le \left| \int_{j \, \Delta - \alpha \, \Delta}^\omega f'(\omega) \, d\omega \right| \\
 & \le \left| \int_{j \, \Delta - \alpha \, \Delta}^\omega \left| f'(\omega) \right| d\omega \right| \\
 & \le \int_{(j-1) \, \Delta}^{j \, \Delta} \left| f'(\omega) \right| \, d\omega
\end{split}
\label{eq:fw_m_fjd}
\end{equation}
We now get
\begin{equation}
\begin{split}
 & \left|\int_{m \, \Delta}^\infty f(\omega) \, d\omega
  - \sum_{j=m+1}^\infty f(j \, \Delta - \alpha \, \Delta) \, \Delta \right| \\
  &= \left|\sum_{j=m+1}^\infty \int_{(j-1) \, \Delta}^{j \, \Delta}
  [f(\omega) - f(j \, \Delta - \alpha \, \Delta)] \, d\omega  \right| \\
  &\le \sum_{j=m+1}^\infty \int_{(j-1) \, \Delta}^{j \, \Delta}
  \left| f(\omega) - f(j \, \Delta - \alpha \, \Delta)  \right| \, d\omega \\
  &\le \sum_{j=m+1}^\infty \int_{(j-1) \, \Delta}^{j \, \Delta}
  \left[ \int_{(j-1) \, \Delta}^{j \, \Delta} \left| f'(\omega) \right| \, d\omega  \right] \, d\omega \\
  &= \sum_{j=m+1}^\infty \int_{(j-1) \, \Delta}^{j \, \Delta} \left|f'(\omega) \right| \, d\omega \, \Delta \\
  &= \underbrace{\int_{m \, \Delta}^{\infty} \left| f'(\omega) \right| \, d\omega}_{C^{(m)}} \, \Delta \\
  &\le \underbrace{\int_{0}^{\infty} \left| f'(\omega) \right| \, d\omega}_{C^{(0)}} \, \Delta,
\end{split}
\end{equation}
which concludes the proof.
\end{proof}

\begin{mylemma} \label{lem:riemann2}
Assume that $f(\omega)$ is a bounded integrable function defined on $\omega \ge 0$ with bounded derivative on $\omega > 0$ and $g(\omega)$ is a bounded function defined on $\omega \ge 0$ such that $|g(w)| \le D$. Further assume that $\int_0^{\infty} |f'(\omega)| \, d\omega = C < \infty$. Then for any $\alpha,\beta \in [0,1)$ and $\Delta > 0$ we have
\begin{equation}
  \left| \sum_{j=1}^\infty [f(j \, \Delta) - f(j \, \Delta - \alpha \, \Delta)] \, g(j \, \Delta - \beta \, \Delta) \right| \le C \, D.
\end{equation}
\end{mylemma}

\begin{proof}
By using \eqref{eq:fw_m_fjd} with $\omega = j \, \Delta - \alpha \, \Delta$ we get
\begin{equation}
\begin{split}
  \left| f(j \, \Delta) - f(j \, \Delta - \alpha \, \Delta) \right|
  \le \int_{(j-1) \, \Delta}^{j \, \Delta} \left| f'(\omega) \right| \, d\omega,
\end{split}
\label{eq:fjd_m_falpha}
\end{equation}
and further
\begin{equation}
\begin{split}
  &\left| \sum_{j=1}^\infty [f(j \, \Delta) - f(j \, \Delta - \alpha \, \Delta)] \,
  g(j \, \Delta - \beta \, \Delta) \right| \\
  &\le \sum_{j=1}^\infty \left| f(j \, \Delta) - f(j \, \Delta - \alpha \, \Delta) \right| \,
  |g(j \, \Delta - \beta \, \Delta)| \\
  &\le \sum_{j=1}^\infty \int_{(j-1) \, \Delta}^{j \, \Delta} \left| f'(\omega) \right| \, d\omega \, D \\
  &= \int_{0}^{\infty} \left| f'(\omega) \right| \, d\omega \, D \\
  &= C \, D. \\
\end{split}
\end{equation}
\end{proof}

\begin{mylemma} \label{lem:riemann3}
Assume that $f(\omega) \ge 0$ is a positive bounded integrable function defined on $\omega \ge 0$ with bounded derivative on $\omega > 0$ such that $\int_0^{\infty} f(\omega) \, d\omega = C_0 < \infty$ and $\int_0^{\infty} |f'(\omega)| \, d\omega = C_1 \le \infty$, and $g(\omega)$ is a bounded integrable function defined on $\omega \ge 0$ with bounded derivative on $\omega > 0$ such that $|g'(\omega)| \le D$. Then for any $\alpha,\beta \in [0,1)$ and $\Delta > 0$ we have for $C = C_1 + C_0$:
\begin{equation}
  \left| \sum_{j=1}^\infty f(j \, \Delta - \alpha \, \Delta) \, [g(j \, \Delta) - g(j \, \Delta - \beta \, \Delta)] 
   \right| \le C \, D.
\end{equation}
\end{mylemma}

\begin{proof}
By applying the mean value theorem to \eqref{eq:fjd_m_falpha} we get that for some $\omega^*_j \in [j \, \Delta - \alpha \, \Delta, j \, \Delta]$ we have
\begin{equation}
\begin{split}
  \left| g(j \, \Delta) - g(j \, \Delta - \beta \, \Delta) \right|
  \le |g'(\omega^*_j)| \, \beta \, \Delta
  \le |g'(\omega^*_j)| \, \Delta \le D \, \Delta.
\end{split}
\end{equation}
By using Lemma~\ref{lem:riemann1} we get
\begin{equation}
\begin{split}
  &\sum_{j=1}^\infty f(j \, \Delta - \alpha \, \Delta) \, \Delta \\
  &=
  \left| \sum_{j=1}^\infty f(j \, \Delta - \alpha \, \Delta) \, \Delta 
  - \int_0^\infty f(\omega) \, d\omega
  + \int_0^\infty f(\omega) \, d\omega \right| \\
  &\le
  \left| \sum_{j=1}^\infty f(j \, \Delta - \alpha \, \Delta) \, \Delta 
  - \int_0^\infty f(\omega) \, d\omega  \right|
  + \left| \int_0^\infty f(\omega) \, d\omega \right| \\
  &\le C_1 + C_0 = C.
\end{split}
\end{equation}
Hence, 
\begin{equation}
\begin{split}
  &\left| \sum_{j=1}^\infty f(j \, \Delta - \alpha \, \Delta) \,
  [g(j \, \Delta) - g(j \, \Delta - \beta \, \Delta)] 
   \right| \\
  &\le \sum_{j=1}^\infty f(j \, \Delta - \alpha \, \Delta) \,
  \left| g(j \, \Delta) - g(j \, \Delta - \beta \, \Delta) \right| \,
   \\  
  &\le \sum_{j=1}^\infty f(j \, \Delta - \alpha \, \Delta) \, \Delta \, D \\  
  &\le C \, D.
\end{split}
\end{equation}
\end{proof}

\subsection{Proof of Theorem \ref{the:1d_kapp_inf}} \label{app:conv_proof1}

The Wiener--Khinchin identity and the symmetry of the spectral density allows us to write
\begin{align}
  k(x,x') &= \frac{1}{2\pi} \,
  \int_{-\infty}^{\infty} S(\omega) \,
  \exp(-\imag \, \omega \, (x - x')) \dd\omega \nonumber \\
  &= \frac{1}{\pi} \,
  \int_{0}^{\infty} S(\omega) \,
  \cos(\omega \, (x - x')) \dd\omega.
\end{align}
In a one-dimensional domain $\Omega = [-L,L]$ with Dirichet boundary conditions we have an $m$-term approximation of the form
\begin{equation}
\begin{split}
 &\widetilde{k}_m(x,x') \\ &=
 \sum_{j=1}^m S\left(\frac{\pi \, j}{2L}\right)
 \, \frac{1}{L} \, \sin\left(\frac{\pi \, j \, (x + L)}{2L} \right)
 \, \sin\left(\frac{\pi \, j \, (x' + L)}{2L} \right).
\end{split}
\end{equation}
We start by showing the convergence by growing the domain and therefore first consider an approximation with an infinite number of terms $m=\infty$:
\begin{equation}
\begin{split}
 &\widetilde{k}_\infty(x,x') \\ &=
 \sum_{j=1}^\infty S\left(\frac{\pi \, j}{2L}\right)
 \, \frac{1}{L} \, \sin\left(\frac{\pi \, j \, (x + L)}{2L} \right)
 \, \sin\left(\frac{\pi \, j \, (x' + L)}{2L} \right).
\end{split}
\label{eq:1d_kapp_inf}
\end{equation}
For that purpose we rewrite the summation above in \eqref{eq:1d_kapp_inf} as
\begin{align}
  &\sum_{j=1}^{\infty} S\left(\frac{\pi \, j}{2L}\right)
 \, \frac{1}{L} \, \sin\left(\frac{\pi \, j \, (x + L)}{2L} \right)
 \, \sin\left(\frac{\pi \, j \, (x' + L)}{2L} \right) \nonumber \\
    &= \sum_{j=1}^\infty S\left( \frac{\pi \, j}{2L} \right) \, 
  \cos\left(\frac{\pi \, j \, (x - x')}{2L} \right)
  \, \frac{1}{2L} \nonumber \\
  &- 
  \frac{1}{2L} \, \sum_{j=1}^\infty
     \left[ S\left( \frac{\pi \, 2j}{2L} \right) \, 
   - S\left( \frac{\pi \, (2j-1)}{2L} \right) \right] \, 
   \cos\left(\frac{\pi \, 2j \, (x + x')}{2L} \right) \nonumber \\
  &-
  \frac{1}{2L} \, \sum_{j=1}^\infty
  S\left( \frac{\pi \, (2j-1)}{2L} \right) \, 
   \Bigg[ \cos\left(\frac{\pi \, 2j \, (x + x')}{2L} \right) \nonumber \\
   &\qquad \qquad \qquad
   - \cos\left(\frac{\pi \, (2j-1) \, (x + x')}{2L} \right) \Bigg].
  \label{eq:three-sums}
\end{align}
and consider the three summations above separately. The analysis of them is done in the next three lemmas.

\begin{mylemma} \label{lem:1d_kapp_inf1a}
  Assume that on $\omega \ge 0$ we have $S(\omega) \le B < \infty$ and $\int_0^\infty S(w) \, \dd \omega = A < \infty$, and on $\omega > 0$ $S(\omega)$ has a bounded derivative $|S'(\omega)| \le D < \infty$ and that $\int_0^\infty |S'(\omega)| \, \dd \omega = C < \infty$. Then there exists a constant $D_2$ such that for all $x,x' \in [-\widetilde{L},\widetilde{L}]$ we have
\begin{multline}
  \Bigg| \sum_{j=1}^\infty S\left( \frac{\pi \, j}{2L} \right) \, 
  \cos\left(\frac{\pi \, j \, (x - x')}{2L} \right)
  \, \frac{1}{2L} 
 \\
 - \frac{1}{\pi} \, \int_{0}^{\infty} S(\omega) \,
    \cos(\omega \, (x - x')) \dd\omega \Bigg| \leq \frac{D_2}{L}.
\end{multline}
\end{mylemma}

\begin{proof}
By using Lemma~\ref{lem:riemann1} with $\Delta = \frac{\pi}{2L}$, $f(\omega) = \frac{1}{\pi} \, S(\omega) \,
    \cos(\omega \, (x - x')) \dd\omega$, $m=0$, and $\alpha = 0$ as well as the assumptions on $S(\omega)$ and boundedness of sine and cosine we get that
\begin{equation}
\begin{split}
  &\Bigg| \sum_{j=1}^\infty S\left( \frac{\pi \, j}{2L} \right) \, 
  \cos\left(\frac{\pi \, j \, (x - x')}{2L} \right)
  \, \frac{1}{2L} \\
  &\qquad - \frac{1}{\pi} \int_0^\infty S(\omega) \,
  \cos(\omega \, (x - x')) \dd\omega \Bigg| \\
  &\leq
  \frac{1}{\pi} \int \big| S'(w) \, \cos(\omega \, (x - x')) \dd\omega \\
  &\qquad - S(w) \, (x - x') \, \sin(\omega \, (x - x'))  \big| \, \dd\omega \, \frac{\pi}{2L} \\
  &\leq
  \frac{1}{2L} \int \big| S'(w) \, \cos(\omega \, (x - x')) \big| \dd\omega \\
  &\quad + \frac{1}{2L} \int \big| S(w) \, (x - x') \, \sin(\omega \, (x - x'))  \big| \, \dd\omega \\
  &\leq
  \frac{1}{2L} \int \big| S'(w) \big| \dd\omega 
  + \frac{|x - x'| }{2L} \, \int \big| S(w)   \big| \, \dd\omega \\
  &\leq
  \frac{1}{2L} \int \big| S'(w) \big| \dd\omega 
  + \frac{\widetilde{L} }{L} \, \int \big| S(w)   \big| \, \dd\omega \\
  &\leq
  \frac{1}{2L} C + \frac{\widetilde{L} }{L} \, A,
\end{split}
\end{equation}
which gives the result with $D_2 = \frac{C}{2} + \widetilde{L} \, A$.
\end{proof}
    
\begin{mylemma} \label{lem:1d_kapp_inf1b}
Assume that for $\omega \ge 0$, $S(\omega)$ is a bounded integrable function with a bounded derivative on $\omega > 0$ such that $\int_0^\infty |S'(\omega)| \, \dd \omega = C < \infty$, then there exists a constant $D_3$ such that
\begin{equation}
\begin{split}
  &\left| \frac{1}{2L} \, \sum_{j=1}^\infty
     \left[ S\left( \frac{\pi \, 2j}{2L} \right) \, 
   - S\left( \frac{\pi \, (2j-1)}{2L} \right) \right] \, 
   \cos\left(\frac{\pi \, 2j \, (x + x')}{2L} \right) \right| \\
   &\le \frac{D_3}{L}.
\end{split}
\end{equation}
\end{mylemma}

\begin{proof}
The result follows by using Lemma~\ref{lem:riemann2} with $\Delta = \frac{\pi}{L}$, $\alpha = 1/2$, $\beta = 0$, $f(\omega) = S(\omega)$, and $g(\omega) = \cos(\omega \, (x + x'))$ and by recalling that  $|\cos(\omega \, (x + x'))| \le 1$, which gives the constant $D_3 = \frac{C}{2}$.
\end{proof}

\begin{mylemma} \label{lem:1d_kapp_inf1c}
Assume that for $\omega \ge 0$, $S(\omega)$ is a bounded positive integrable function with bounded derivative on $\omega > 0$ such that $\int_0^\infty S(\omega) \, \dd \omega = A < \infty$ and $\int_0^\infty |S'(\omega)| \, \dd \omega = C < \infty$. Then there exists a constant $D_4$ such that
\begin{equation}
\begin{split}
  &\Bigg| \frac{1}{2L} \, \sum_{j=1}^\infty
  S\left( \frac{\pi \, (2j-1)}{2L} \right) \, 
   \Bigg[ \cos\left(\frac{\pi \, 2j \, (x + x')}{2L} \right) \nonumber \\
   &\qquad \qquad \qquad
   - \cos\left(\frac{\pi \, (2j-1) \, (x + x')}{2L} \right) \Bigg|
   \le \frac{D_4}{L}.
\end{split}
\end{equation}
\end{mylemma}

\begin{proof}
By using Lemma~\ref{lem:riemann3} with $\Delta = \frac{\pi}{L}$, $\alpha = 1/2$, $\beta = 1/2$, $f(\omega) = S(\omega)$, and $g(\omega) = \cos(\omega \, (x + x'))$  we get
\begin{equation}
\begin{split}
  &\Bigg| \frac{1}{2L} \, \sum_{j=1}^\infty
  S\left( \frac{\pi \, (2j-1)}{2L} \right) \, 
   \Bigg[ \cos\left(\frac{\pi \, 2j \, (x + x')}{2L} \right) \nonumber \\
   &\qquad \qquad \qquad
   - \cos\left(\frac{\pi \, (2j-1) \, (x + x')}{2L} \right) \Bigg| \\
   &\le \frac{(A + C) \, D'}{2L},
\end{split}
\end{equation}
where $D'$ is an upper bound for $|(x + x') \, \sin(\omega \, (x + x'))|$. We can now select $D' = 2\widetilde{L}$, which gives $D_4 = (A + C) \, \widetilde{L}$.
\end{proof}

Next, we combine the above lemmas to get the following result.
\begin{mylemma} \label{lem:1d_kapp_inf1}
Let the assumptions of Lemmas~\ref{lem:1d_kapp_inf1a}, \ref{lem:1d_kapp_inf1b}, and \ref{lem:1d_kapp_inf1c} be satisfied. Then there exists a constant $D_1$ such that for all $x,x' \in [-\widetilde{L},\widetilde{L}]$ we have
\begin{multline}
  \Bigg| \sum_{j=1}^{\infty} S\left(\frac{\pi \, j}{2L}\right)
 \, \frac{1}{L} \, \sin\left(\frac{\pi \, j \, (x + L)}{2L} \right)
 \, \sin\left(\frac{\pi \, j \, (x' + L)}{2L} \right) \\
 - \frac{1}{\pi} \, \int_{0}^{\infty} S(\omega) \,
    \cos(\omega \, (x - x')) \dd\omega \Bigg| \leq \frac{D_1}{L}.
\label{eq:1d_kapp_inf_lemma}
\end{multline}
That is, 
\begin{equation}
  \left| \widetilde{k}_{\infty}(x,x') - k(x,x') \right| \leq \frac{D_1}{L},
  \quad \text{ for } x,x' \in [-\widetilde{L},\widetilde{L}].
\end{equation}
Furthermore, the explicit expression for the constant is given as
\begin{equation}
  D_1 = C + (2 A + C) \, \widetilde{L}.
\end{equation}
\end{mylemma}

\begin{proof}
Using triangle inequality to the differece of \eqref{eq:three-sums} and $\frac{1}{\pi} \, \int_{0}^{\infty} S(\omega) \, \cos(\omega \, (x - x')) \dd\omega$ along with Lemmas~\ref{lem:1d_kapp_inf1a}, \ref{lem:1d_kapp_inf1b}, and \ref{lem:1d_kapp_inf1c} gives
\begin{align}
  &\Bigg| \sum_{j=1}^\infty S\left( \frac{\pi \, j}{2L} \right) \, 
  \cos\left(\frac{\pi \, j \, (x - x')}{2L} \right)
  \, \frac{1}{2L} \nonumber \\
  &- 
  \frac{1}{2L} \, \sum_{j=1}^\infty
     \left[ S\left( \frac{\pi \, 2j}{2L} \right) \, 
   - S\left( \frac{\pi \, (2j-1)}{2L} \right) \right] \, 
   \cos\left(\frac{\pi \, 2j \, (x + x')}{2L} \right) \nonumber \\
  &-
  \frac{1}{2L} \, \sum_{j=1}^\infty
  S\left( \frac{\pi \, (2j-1)}{2L} \right) \, 
   \Bigg[ \cos\left(\frac{\pi \, 2j \, (x + x')}{2L} \right) \nonumber \\
   &\qquad \qquad \qquad
   - \cos\left(\frac{\pi \, (2j-1) \, (x + x')}{2L} \right) \Bigg] \nonumber \\
   &- \frac{1}{\pi} \, \int_{0}^{\infty} S(\omega) \,
    \cos(\omega \, (x - x')) \dd\omega \Bigg| \nonumber \\
    &\le 
  \Bigg| \sum_{j=1}^\infty S\left( \frac{\pi \, j}{2L} \right) \, 
  \cos\left(\frac{\pi \, j \, (x - x')}{2L} \right)
  \, \frac{1}{2L} \nonumber \\
   &\qquad \qquad \qquad
   - \frac{1}{\pi} \, \int_{0}^{\infty} S(\omega) \,
    \cos(\omega \, (x - x')) \dd\omega \Bigg| \nonumber \\
  &+ \Bigg|
  \frac{1}{2L} \, \sum_{j=1}^\infty
     \left[ S\left( \frac{\pi \, 2j}{2L} \right) \, 
   - S\left( \frac{\pi \, (2j-1)}{2L} \right) \right] \, 
   \cos\left(\frac{\pi \, 2j \, (x + x')}{2L} \right) \Bigg| \nonumber \\
  &+
  \Bigg| \frac{1}{2L} \, \sum_{j=1}^\infty
  S\left( \frac{\pi \, (2j-1)}{2L} \right) \, 
   \Bigg[ \cos\left(\frac{\pi \, 2j \, (x + x')}{2L} \right) \nonumber \\
   &\qquad \qquad \qquad
   - \cos\left(\frac{\pi \, (2j-1) \, (x + x')}{2L} \right) \Bigg] \Bigg| \nonumber \\
   &\le
   \frac{D_2}{L} + \frac{D_3}{L} + \frac{D_4}{L} = \frac{D_1}{L},
\end{align}
where the explicit values for the costants can be found in the proofs of the lemmas.
\end{proof}

Let us now consider what happens when we replace the infinite sum approximation with a finite $m$ number of terms. We are now interested in
\begin{equation}
\begin{split}
 &\widetilde{k}_{\infty}(x,x') - \widetilde{k}_m(x,x') \\ &=
 \sum_{j=m+1}^{\infty} S\left(\frac{\pi \, j}{2L}\right)
   \, \frac{1}{L} \, \sin\left(\frac{\pi \, j \, (x + L)}{2L} \right)
   \, \sin\left(\frac{\pi \, j \, (x' + L)}{2L} \right).
\end{split}
\end{equation}
\begin{mylemma} \label{lem:1d_kapp_inf2}
  Assume that on $\omega \ge 0$, $S(\omega)$ is bounded and integrable, on $\omega > 0$ it has a bounded derivative, and that $\int_0^\infty |S'(\omega)| \, \dd \omega = C < \infty$.  Then there exists a constant $D_5$ such that for all $x,x' \in [-\widetilde{L},\widetilde{L}]$ we have
\begin{equation}
  \left| \widetilde{k}_{\infty}(x,x') - \widetilde{k}_m(x,x') \right|
 \leq \frac{D_5}{L}
  + \frac{2}{\pi} \, \int_{\frac{\pi \, m}{2L}}^{\infty} S(\omega) \dd\omega.
\end{equation}
\end{mylemma}

\begin{proof}
Because the sinusoidals are bounded by unity, we get
\begin{multline}
 \left| \sum_{j=m+1}^{\infty} S\left(\frac{\pi \, j}{2L}\right)
 \, \frac{1}{L} \, \sin\left(\frac{\pi \, j \, (x + L)}{2L} \right)
 \, \sin\left(\frac{\pi \, j \, (x' + L)}{2L} \right) \right| \\
 \leq
 \left| \sum_{j=m+1}^{\infty} S\left(\frac{\pi \, j}{2L}\right)
 \, \frac{1}{L}  \right|.
\end{multline}
For the right hand side we can now use Lemma~\ref{lem:riemann1} with $f(\omega) = \frac{2}{\pi} \, S(\omega)$ and $\Delta = \frac{\pi}{2L}$, which gives
\begin{equation}
  \Bigg| \sum_{j=m+1}^{\infty} S\left(\frac{\pi \, j}{2L}\right)
   \, \frac{1}{L} - \frac{2}{\pi} \int_{\frac{\pi \, m}{2L}}^{\infty}
   S(\omega) \dd\omega \Bigg| \leq C \, \frac{\pi}{2L} = \frac{D_5}{L}.
\end{equation}
Hence by the triangle inequality we get
\begin{align}
  &\Bigg| \sum_{j=m+1}^{\infty} S\left(\frac{\pi \, j}{2L}\right)
   \, \frac{1}{L} \Bigg| \nonumber\\
  &= 
   \Bigg| \sum_{j=m+1}^{\infty} S\left(\frac{\pi \, j}{2L}\right)
   \, \frac{1}{L} - \frac{2}{\pi} \int_{\frac{\pi \, m}{2L}}^{\infty}
   S(\omega) \dd\omega
   + \frac{2}{\pi} \int_{\frac{\pi \, m}{2L}}^{\infty}
   S(\omega) \dd\omega \Bigg| \nonumber \\
  &\leq 
   \Bigg| \sum_{j=m+1}^{\infty} S\left(\frac{\pi \, j}{2L}\right)
   \, \frac{1}{L} - \frac{2}{\pi} \int_{\frac{\pi \, m}{2L}}^{\infty}
   S(\omega) \dd\omega \Bigg|
   + \frac{2}{\pi} \int_{\frac{\pi \, m}{2L}}^{\infty}
   S(\omega) \dd\omega \nonumber \\
  &\leq 
   \frac{D_5}{L}
   + \frac{2}{\pi} \int_{\frac{\pi \, m}{2L}}^{\infty}
   S(\omega) \dd\omega
\end{align}
and thus the result follows.
\end{proof}

\begin{myremark} \label{rem:1d_kapp_inf2}
We can also obtain a bit more defined bound by not using an $m$-independent bound for forming $D_5$, which under the assumptions of Lemma~\ref{lem:1d_kapp_inf2} gives
\begin{equation}
\begin{split}
  &\left| \widetilde{k}_{\infty}(x,x') - \widetilde{k}_m(x,x') \right| \\
 &\leq \frac{\pi}{2L} \, \int_{\frac{\pi \, m}{2L}}^{\infty} |S'(\omega)| \dd\omega
  + \frac{2}{\pi} \, \int_{\frac{\pi \, m}{2L}}^{\infty} S(\omega) \dd\omega.
\end{split}
\end{equation}
\end{myremark}

The lemmas presented in this section can now be combined to a proof of the one-dimensional convergence theorem as follows:

\begin{proof}[Proof of Theorem \ref{the:1d_kapp_inf}]
  The first result follows by combining Lemmas~\ref{lem:1d_kapp_inf1} and \ref{lem:1d_kapp_inf2} via the triangle inequality. Because our assumptions imply that
\begin{equation}
  \lim_{x \to \infty} \int_{x}^{\infty} S(\omega) \dd\omega = 0,
\end{equation}
for any fixed $L$ we have
\begin{equation}
  \lim_{m \to \infty} \left[ \frac{E}{L} 
  + \frac{2}{\pi} \,
    \int_{\frac{\pi \, m}{2L}}^{\infty}
    S(\omega) \dd\omega \right] \to \frac{E}{L}.
\end{equation}
If we now take the limit $L \to \infty$, the second result in the theorem follows. 
\end{proof}

\subsection{Proof of Theorem~\ref{the:nd_kapp_inf}} \label{app:conv_proof2}
When $\vect{x} \in \R^d$, the Wiener--Khinchin identity and symmetry of the spectral density imply that
\begin{equation}
\begin{split}
  &k(\vect{x},\vect{x}') =
  \frac{1}{(2\pi)^d} \,
  \int_{\R^d}
  S(\vectb{\omega}) \, \exp(-\imag \, \vectb{\omega}\T (\vect{x} - \vect{x}')) \dd\vectb{\omega} \\
  &= \frac{1}{\pi^d} \,
  \int_{0}^{\infty} \cdots \int_{0}^{\infty}
  S(\vectb{\omega}) \, \prod_{k=1}^d \cos(\omega_k \, (x_k - x_k')) 
  \dd\omega_1 \cdots \dd\omega_d.
\end{split}
  \label{eq:nktrue}
\end{equation}
The $m = {\hat m}^d$ term approximation now has the form
\begin{multline}
  \widetilde{k}_m(\vect{x},\vect{x}') =
  \sum_{j_1,\ldots,j_d=1}^{\hat m}
    S\left( \frac{\pi \, j_1}{2L_1},\ldots,\frac{\pi \, j_d}{2L_d}  \right) \\
   \times \prod_{k=1}^d \frac{1}{L_k} \,
    \sin\left(\frac{\pi \, j_k \, (x_k + L_k)}{2L_k} \right) \,
    \sin\left(\frac{\pi \, j_k \, (x_k' + L_k)}{2L_k} \right).
\end{multline}

As in the one-dimensional problem we start by considering the case where $\hat{m} = \infty$.
\begin{mylemma} \label{lem:nd_kapp_inf1}
Let the assumptions of Lemma~\ref{lem:1d_kapp_inf1} be satisfied for each $\omega_j \mapsto S(\omega_1,\ldots,\omega_d)$ separately. Then there exists a constant $D_1$ such that for all $\vect{x},\vect{x}' \in [-\widetilde{L},\widetilde{L}]^d$ we have
\begin{align}
  &\Bigg| 
  \sum_{j_1,\ldots,j_d=1}^\infty 
  S\left( \frac{\pi \, j_1}{2L_1},\ldots,\frac{\pi \, j_d}{2L_d}  \right) \nonumber \\
  &\qquad \times 
  \prod_{k=1}^d \frac{1}{L_k} \,
    \sin\left(\frac{\pi \, j_k \, (x_k + L_k)}{2L_k} \right) \,
    \sin\left(\frac{\pi \, j_k \, (x_k' + L_k)}{2L_k} \right) \nonumber \\
  &\quad - \frac{1}{\pi^d} \,
    \int_{0}^{\infty} \cdots \int_{0}^{\infty}
    S(\vectb{\omega}) \, \prod_{k=1}^d \cos(\omega_k \, (x - x')) 
    \dd\omega_1 \cdots \dd\omega_d \Bigg| \nonumber \\
  &\leq D_1 \sum_{k=1}^d \frac{1}{L_k} \le \frac{D_1 \, d}{L},
  \label{eq:nd_kapp_inf_lemma}
\end{align}
where $L = \min_k L_k$. That is, for all $\vect{x},\vect{x}' \in [-\widetilde{L},\widetilde{L}]^d$
\begin{equation}
  \left| \widetilde{k}_{\infty}(\vect{x},\vect{x}') - k(\vect{x},\vect{x}') \right| 
  \leq D_1 \sum_{k=1}^d \frac{1}{L_k} \le \frac{D_1 \, d}{L}.
\end{equation}
\end{mylemma}

\begin{proof}
We can separate the summation over $j_1$ as follows:
\begin{equation}
\begin{split}
  &\sum_{j_2,\ldots,j_d=1}^\infty \Bigg[ \sum_{j_1=1}^{\infty}
  S\left( \frac{\pi \, j_1}{2L_1},\ldots,\frac{\pi \, j_d}{2L_d}  \right) \, \frac{1}{L_1} \\
   &\qquad \times \sin\left(\frac{\pi \, j_1 \, (x_1 + L_1)}{2L_1} \right) \,
    \sin\left(\frac{\pi \, j_1 \, (x_1' + L_1)}{2L_1} \right) 
  \Bigg] \\
  &\qquad \times 
  \prod_{k=2}^d \frac{1}{L_k} \,
    \sin\left(\frac{\pi \, j_k \, (x_k + L_k)}{2L_k} \right) \,
    \sin\left(\frac{\pi \, j_k \, (x_k' + L_k)}{2L_k} \right).
\end{split}
\end{equation}
By Lemma~\ref{lem:1d_kapp_inf1} there now exists a constant $D_{1,1}$ such that
\begin{align}
  &\Bigg| \sum_{j_1=1}^{\infty}
  S\left( \frac{\pi \, j_1}{2L_1},\ldots,\frac{\pi \, j_d}{2L_d} \right) \, \frac{1}{L_1} \nonumber \\
  &\quad \times \sin\left(\frac{\pi \, j_1 \, (x_1 + L_1)}{2L_1} \right) \,
    \sin\left(\frac{\pi \, j_1 \, (x_1' + L_1)}{2L_1} \right) \nonumber\\
  &\quad - \frac{1}{\pi} \, \int_{0}^{\infty} S\left(\omega_1,\frac{\pi \, j_2}{2L_2},\ldots,\frac{\pi \, j_d}{2L_d}\right) \, \cos(\omega_1 \, (x_1 - x_1')) \dd\omega_1 \Bigg| \nonumber \\
  &\leq \frac{D_{1,1}}{L_1}.
\end{align}
The triangle inequality then gives
\begin{align}
  &\Bigg| 
  \sum_{j_1,\ldots,j_d=1}^\infty 
  S\left( \frac{\pi \, j_1}{2L_1},\ldots,\frac{\pi \, j_d}{2L_d} \right) \nonumber\\
  & \quad\times 
  \prod_{k=1}^d \frac{1}{L_k} \,
    \sin\left(\frac{\pi \, j_k \, (x_k + L_k)}{2L_k} \right) \,
    \sin\left(\frac{\pi \, j_k \, (x_k' + L_k)}{2L_k} \right) \nonumber \\
  &- \frac{1}{\pi^d} \,
  \int_{0}^{\infty} \cdots \int_{0}^{\infty}
  S(\vectb{\omega}) \, \prod_{k=1}^d \cos(\omega_j \, (x_k - x_k'))  
  \dd\omega_1 \cdots \dd\omega_d
  \Bigg| \nonumber\\
  &\leq
  \frac{D_{1,1}}{L_1} 
  + \Bigg| \frac{1}{\pi} \, \sum_{j_2,\ldots,j_d=1}^\infty \int_{0}^{\infty}
  S\left(\omega_1,\frac{\pi \, j_2}{2L_2}, \ldots,
  \frac{\pi \, j_d}{2L_d}\right)  \nonumber \\
  & \quad \times \cos(\omega_1 \, (x_1 - x_1'))  
  \dd\omega_1 \nonumber \\
  & \quad\times 
  \prod_{k=2}^d \frac{1}{L_k} \,
    \sin\left(\frac{\pi \, j_k \, (x_k + L_k)}{2L_k} \right) \,
    \sin\left(\frac{\pi \, j_k \, (x_k' + L_k)}{2L_k} \right) \nonumber \\
  & - \frac{1}{\pi^d} \,
  \int_{0}^{\infty} \cdots \int_{0}^{\infty}
  S(\vectb{\omega}) \, \prod_{k=1}^d \cos(\omega_k \, (x_k - x_k')) 
  \dd\omega_1 \cdots \dd\omega_d
  \Bigg|.
\label{eq:nd_kapp_inf_lemma}
\end{align}
We can now similarly bound with respect to the summations over $j_2,\ldots,j_d$ which  leads to a bound of the form $\frac{D_{1,1}}{L_1} + \cdots + \frac{D_{1,d}}{L_d}$. Taking $D_1 = \max_k D_{1,k}$ leads to the desired result.
\end{proof}

Now we can consider what happens in the finite truncation of the series. That is, we analyze the following residual sum
\begin{align}
  &\widetilde{k}_{\infty}(\vect{x},\vect{x}') - \widetilde{k}_m(\vect{x},\vect{x}') \nonumber \\
  &\quad= \sum_{j_1,\ldots,j_d={\hat m}+1}^{\infty}
  S\left( \frac{\pi \, j_1}{2L_1},\ldots,\frac{\pi \, j_d}{2L_d}  \right) \nonumber \\
  &\qquad \times
  \prod_{k=1}^d \frac{1}{L_k} \,
    \sin\left(\frac{\pi \, j_k \, (x_k + L_k)}{2L_k} \right) \,
    \sin\left(\frac{\pi \, j_k \, (x_k' + L_k)}{2L_k} \right).
\end{align}
\begin{mylemma} \label{lem:nd_kapp_inf2}
  Let assumptions of Lemma~\ref{lem:1d_kapp_inf2} be satisfied for each $\omega_j \mapsto S(\omega_1, \ldots, \omega_d)$.
  There exists a constant $D_2$ such that for all $\vect{x},\vect{x}' \in [-\widetilde{L},\widetilde{L}]^d$ we have
\begin{equation}
\begin{split}
  &\left| \widetilde{k}_{\infty}(\vect{x},\vect{x}') - \widetilde{k}_m(\vect{x},\vect{x}') \right|
  \leq \frac{D_2 \, d}{L}
  + \frac{1}{\pi^d} 
  \int_{\norm{\vectb{\omega}} \geq \frac{\pi \, \hat{m}}{2L}} S(\vectb{\omega}) \dd\vectb{\omega},
\end{split}
\end{equation}  
where $L = \min_k L_k$.
\end{mylemma}

\begin{proof}
We can write the following bound
\begin{align}
  &\Bigg| \sum_{j_1,\ldots,j_d={\hat m}+1}^{\infty}
  S\left( \frac{\pi \, j_1}{2L_1},\ldots,\frac{\pi \, j_d}{2L_d}  \right) \nonumber \\
  &\qquad \times \prod_{k=1}^d \frac{1}{L_k} \,
    \sin\left(\frac{\pi \, j_k \, (x_k + L_k)}{2L_k} \right) \,
    \sin\left(\frac{\pi \, j_k \, (x_k' + L_k)}{2L_k} \right) \Bigg| \nonumber \\
  & 
  \leq \Bigg| \sum_{j_1,\ldots,j_d={\hat m}+1}^{\infty}
  S\left( \frac{\pi \, j_1}{2L_1},\ldots,\frac{\pi \, j_d}{2L_d} \right) \,
  \prod_{k=1}^d \frac{1}{L_k} \Bigg|.
\end{align}
We can now use Lemma~\ref{lem:riemann1} with $f(\omega_1) = \frac{2}{\pi} \, S\left( \omega_1, \frac{\pi \, j_2}{2L_2},\ldots,\frac{\pi \, j_d}{2L_d} \right)$ and $\Delta = \frac{\pi}{2L_1}$, which gives
\begin{align}
  &\Bigg| \sum_{j_1,\ldots,j_d={\hat m}+1}^{\infty}
  S\left( \frac{\pi \, j_1}{2L_1},\ldots,\frac{\pi \, j_d}{2L_d} \right) \,
  \prod_{k=1}^d \frac{1}{L_k} \nonumber\\
  &- \frac{2}{\pi} \sum_{j_2,\ldots,j_d={\hat m}+1}^{\infty}
  \int_{\frac{\pi \, {\hat m}}{2L_1}}^{\infty}
  S\left( \omega_1, \frac{\pi \, j_2}{2L_2},\ldots,\frac{\pi \, j_d}{2L_d} \right) 
  \dd\omega_1 \, \prod_{k=2}^d \frac{1}{L_k} \Bigg| \nonumber\\
  &\leq \frac{D_{2,1}}{L_1}.
\end{align}
Using a similar argument again, we get
\begin{align}
  &\Bigg|
  \frac{2}{\pi} \hspace*{-1em}\sum_{j_2,\ldots,j_d={\hat m}+1}^{\infty}
  \int_{\frac{\pi \, {\hat m}}{2L_1}}^{\infty}
  S\left( \omega_1, \frac{\pi \, j_2}{2L_2},\ldots,\frac{\pi \, j_d}{2L_d} \right) 
  \dd\omega_1 \, \prod_{k=2}^d \frac{1}{L_k} \nonumber\\
  &- 
  \frac{2^2}{\pi^2} \hspace*{-2em}\sum_{j_3,\ldots,j_d={\hat m}+1}^{\infty}
  \int_{\frac{\pi \, {\hat m}}{2L_1}}^{\infty} \int_{\frac{\pi \, {\hat m}}{2L_2}}^{\infty}
  S\left( \omega_1, \omega_2, \frac{\pi \, j_3}{2L_3},\ldots,\frac{\pi \, j_d}{2L_d} \right)
  \dd\omega_1 \dd\omega_2 \nonumber \\
  &\quad \prod_{k=3}^d \frac{1}{L_k} \Bigg| 
  \leq \frac{D_{2,2}}{L_2}.
\end{align}
After repeating this for all the indexes, by forming a telescoping sum of the terms and applying the triangle inequality then gives
\begin{multline}
  \Bigg| \sum_{j_1,\ldots,j_d={\hat m}+1}^{\infty}
  S\left( \frac{\pi \, j_1}{2L_1},\ldots,\frac{\pi \, j_d}{2L_d} \right) \,
  \prod_{k=1}^d \frac{1}{L_k} \\
  - \left( \frac{2}{\pi} \right)^d 
  \int_{\frac{\pi \, {\hat m}}{2L_1}}^{\infty} \cdots \int_{\frac{\pi \, {\hat m}}{2L_d}}^{\infty}
  S( \omega_1, \ldots, \omega_d) 
  \dd\omega_1 \cdots \dd\omega_d \Bigg|
  \leq \sum_{k=1}^d \frac{D_{2,k}}{L_k}.
\end{multline}
Applying the triangle inequality again gives
\begin{multline}
  \Bigg| \sum_{j_1,\ldots,j_d={\hat m}+1}^{\infty}
  S\left( \frac{\pi \, j_1}{2L_1},\ldots,\frac{\pi \, j_d}{2L_d} \right) \,
  \prod_{k=1}^d \frac{1}{L_k} \Bigg| \\
  \leq \sum_{k=1}^d \frac{D_{2,k}}{L_k}
  + \left( \frac{2}{\pi} \right)^d
  \int_{\frac{\pi \, {\hat m}}{2L_1}}^{\infty} \cdots \int_{\frac{\pi \, {\hat m}}{2L_d}}^{\infty}
  S( \omega_1, \ldots, \omega_d) 
  \dd\omega_1 \cdots \dd\omega_d.
\label{eq:tighter_bound}
\end{multline}
By interpreting the latter integral as being over the positive exterior of a rectangular hypercuboid and bounding it by a integral over exterior of a hypersphere which fits inside the cuboid, we can bound the expression by
\begin{equation}
  \sum_{k=1}^d \frac{D_{2,k}}{L_k}
  + \frac{1}{\pi^d} 
  \int_{\norm{\vectb{\omega}} \geq \frac{\pi \, {\hat m}}{2L}} S(\vectb{\omega}) \dd\vectb{\omega}.
\end{equation}
The first term can be further bounded by replacing $L_k$s with their minimum $L$ and by defining $D_2 = \max D_{2,k}$ which is $d$ times the maximum of $D_{2,k}$. This leads to the final form of the result.
\end{proof}

\begin{myremark}  \label{rem:nd_kapp_inf2}
Note that analogously to Remark~\ref{rem:1d_kapp_inf2} we could tighten the bound for $D_2$ by letting it depend on $\hat m$.
\end{myremark}

\begin{proof}[Proof of Theorem~\ref{the:nd_kapp_inf}]
  Analogous to the one-dimensional case. That is, we combine the results of the above lemmas using the triangle inequality.
\end{proof}

\bibliographystyle{spbasic}

\phantomsection
\addcontentsline{toc}{section}{References}
\bibliography{bibliography}%

\end{document}